\documentclass{article}

    \PassOptionsToPackage{numbers, compress}{natbib}

     \usepackage[final]{neurips_2019}

\usepackage{tcolorbox}
\usepackage{todonotes}
\usepackage[utf8]{inputenc} %
\usepackage[T1]{fontenc}    %
\usepackage[colorlinks,bookmarks=false,breaklinks=true,citecolor=olive,linkcolor=violet]{hyperref}       %
\usepackage{url}            %
\usepackage{booktabs}       %
\usepackage{amsfonts}       %
\usepackage{nicefrac}       %
\usepackage{microtype}      %
\usepackage{multicol}
\usepackage{subcaption}
\usepackage{amssymb, amsmath, amsthm}
\usepackage{thmtools}
\usepackage{mathtools}
\usepackage[normalem]{ulem}
\usepackage{refcount}
\usepackage[nounderscore]{syntax}
\usepackage{mathpartir}
\usepackage{stmaryrd}

\usepackage{graphicx}
\usepackage{algorithm,algpseudocode}
\makeatletter
\renewcommand{\ALG@name}{\netsor program}
\makeatother

\usepackage[inline]{enumitem}
\usepackage{breakurl}
\usepackage[capitalize,nameinlink]{cleveref}
\usepackage{mymacros}
\usepackage{bbm}
\crefname{thm}{\text{Theorem}}{\text{Theorems}}
\crefname{assm}{\text{Assumption}}{\text{Assumptions}}
\crefname{defn}{\text{Definition}}{\text{Definitions}}
\crefname{prop}{\text{Proposition}}{\text{Propositions}}
\crefname{cor}{\text{Corollary}}{\text{Corollaries}}
\crefname{lemma}{\text{Lemma}}{\text{Lemmas}}
\crefname{algorithm}{\text{Program}}{\text{Programs}}

\crefformat{footnote}{#2\footnotemark[#1]#3}

\renewcommand{\vec}{\boldsymbol}
\newcommand{\bigtheta}{\Theta}
\newcommand{\bigvtheta}{\Theta}

\newcommand{\onev}{\mathbf{1}}
\newcommand{\relu}{\mathrm{ReLU}}
\newcommand{\Vt}[1]{\mathrm{V}_{#1}}

\newcommand{\SoftMax}{\mathrm{SoftMax}}
\newcommand{\Attention}{\mathrm{Attention}}
\newcommand{\MaskedAttention}{\mathrm{MaskedAttention}}
\newcommand{\erf}{\mathrm{erf}}
\newcommand{\trsp}{\top}
\newcommand{\distto}{\xrar{\mathrm{d}}}
\newcommand{\asto}{\xrar{\mathrm{a.s.}}}

\newcommand{\disteq}{\overset{\mathrm{d}}{=}}

\newcommand{\KK}{K}

\newcommand{\cdc}{\mathcal{C}}
\newcommand{\Ss}{\mathcal{S}}
\newcommand{\Gvars}{\mathcal{G}}

\newcommand{\defeq}{\mathbin{\overset{\mathrm{def}}{=}}}

\newcommand{\Atype}{\mathsf{A}}
\newcommand{\Gtype}{\mathsf{G}}
\newcommand{\Htype}{\mathsf{H}}

\renewcommand{\cite}{\citep}

\newcommand{\netsor}{{$\textsc{Netsor}$} }
\newcommand{\netsorplus}{{$\textsc{Netsor}^+$} }
\newcommand{\netsormin}{{$\textsc{Netsor}^-$} }
\newcommand{\netsoro}{{$\textsc{Netsor}\circ$} }
\newcommand{\netsoroplus}{{$\textsc{Netsor}\circ^+$} }

\newcommand{\Layernorm}{\mathrm{Layernorm}}

\newcommand{\repo}{\url{github.com/thegregyang/GP4A}}

\makeatletter
\let\orgdescriptionlabel\descriptionlabel
\newcommand*{\@restrictlabeltext}[1]{#1\protected@edef\@currentlabel{#1}}
\newcommand*{\nolabel}[1]{#1}%
\renewcommand*{\descriptionlabel}[1]{%
  \let\orglabel\label
  \let\label\@gobble
  \let\orig@hfil\hfil
  \def\hfil{}%
  \let\nolabel\@gobble
  \let\restrictlabeltext\@firstofone
  \phantomsection
  \protected@edef\@currentlabel{#1}%
  \let\hfil\orig@hfil
  \let\label\orglabel
  \let\restrictlabeltext\@restrictlabeltext
  \orgdescriptionlabel{#1}%
}
\makeatother

\title{Tensor Programs I:\\
Wide Feedforward or Recurrent Neural Networks of Any Architecture are Gaussian Processes}

\author{%
  Greg Yang\thanks{Please see \url{https://arxiv.org/abs/1910.12478} for the full version of this paper.}\\
  Microsoft Research AI\\
  \texttt{gregyang@microsoft.com} \\
}

\begin{document}

\maketitle

\begin{abstract}
Wide neural networks with random weights and biases are Gaussian processes, as originally observed by Neal (1995) and more recently by Lee et al.~(2018) and Matthews et al.~(2018) for deep fully-connected networks, as well as by Novak et al.~(2019) and Garriga-Alonso et al.~(2019) for deep convolutional networks.
We show that this Neural Network-Gaussian Process correspondence surprisingly extends to all modern feedforward or recurrent neural networks composed of multilayer perceptron, RNNs (e.g. LSTMs, GRUs), ($n$D or graph) convolution, pooling, skip connection, attention, batch normalization, and/or layer normalization.
More generally, we introduce a language for expressing neural network computations, and our result encompasses all such expressible neural networks.
This work serves as a tutorial on the \emph{tensor programs} technique formulated in Yang (2019) and elucidates the Gaussian Process results obtained there.
We provide open-source implementations of the Gaussian Process kernels of simple RNN, GRU, transformer, and batchnorm+ReLU network at \repo.%
\end{abstract}

\section{Introduction}
\label{sec:Introduction}

Motivated to understand the Bayesian prior in neural networks (NNs), \citet{neal_bayesian_1995} theoretically showed that infinitely wide, shallow neural networks with random weights and biases are Gaussian processes (GPs).
He empirically explored this phenomenon over deep networks as well, but this was not proven rigorously until recently \cite{lee_deep_2018,matthews_gaussian_2018,novak_bayesian_2018,garriga-alonso_deep_2018}, with concrete progress made over the intervening years \cite{{williams_computing_1997,le_roux_continuous_2007,hazan_steps_2015,daniely_toward_2016}}.
This neural network-Gaussian process correspondence (NN-GP correspondence) has not only allowed one to transform the \emph{implicit prior} of NNs into \emph{explicit priors} that can be understood analytically \cite{poole_exponential_2016,schoenholz_deep_2016,yang_mean_2017,xiao_dynamical_2018,yang_mean_2019}, but has also created new state-of-the-art kernels by converting from deep neural networks \cite{lee_deep_2018,novak_bayesian_2018}.
Yet, so far the focus has dwelled entirely on multilayer perceptrons (MLPs) or simple convolutional neural networks (CNNs).
As new architectures are created with blistering speed, a question starts to emerge and reverberate:
\begin{equation*}
\text{\it 
Do all infinitely wide, randomly initialized neural networks correspond to Gaussian processes?}    
\end{equation*}
Even if the answer is yes, at the current rate where each new architecture warrants its own NN-GP correspondence paper, theory will never catch up to practice.
On a more basic level, what does this question even mean for recurrent neural networks?

\paragraph{Our Contributions}
In this paper, we formulate the notion of a Gaussian process with variable-dimensional output (see \cref{defn:GP}), and show that feedforward and recurrent neural networks of \emph{standard architectures} converge to Gaussian processes in this sense as their widths or number of channels go to infinity, when their weights and biases are randomized.
\emph{By \textbf{standard architecture} we mean any architecture that is some composition of multilayer perceptrons (MLPs), recurrent neural networks (RNNs) (e.g., Long-Short Term Memory (LSTM) \cite{hochreiter_long_1997} or Gated Recurrent Unit (GRU) \cite{cho_learning_2014}), skip connections \mbox{\cite{he_deep_2016,huang_densely_2016}},
convolutions
\cite{fukushima_cognitron:_1975,fukushima_neocognitron:_1982,rumelhart_learning_1985,lecun_gradient-based_1998,lecun_object_1999} or graph convolutions \cite{
brunaSpectralNetworksLocally2013arXiv.org,
henaffDeepConvolutionalNetworks2015arXiv.org,
duvenaudConvolutionalNetworksGraphs2015NeuralInformationProcessingSystems,
liGatedGraphSequence2015arXiv.org,
defferrardConvolutionalNeuralNetworks2016arXiv.org,
kipfSemiSupervisedClassificationGraph2016arXiv.org},
pooling \cite{lecun_gradient-based_1998,lecun_object_1999}, batch normalization (batchnorm) \cite{ioffe_batch_2015},
layer normalization \cite{ba_layer_2016} and/or attention \cite{bahdanau_neural_2014,vaswani_attention_2017}.}
Even more broadly, we design a new language, \netsor\!, for expressing neural network computations, and show the GP convergence for all such expressible networks.
By demonstrating that \netsor can implement any network of standard architectures, we obtain the aforementioned results as a corollary.
The results for RNNs, batchnorm, layernorm, attention, and their combination with other layers are new.
We open-source reference implementations\footnote{\label{fnt:repo}\repo} for the GP kernels of simple RNN, GRU, transformer, and feedforward batchnorm network; see \cref{fig:allkernels} for an illustration.

\newcommand{\repofnt}{$^{\text{\ref{fnt:repo}}}$}

\paragraph{Relation of This Paper with \cite{yangScalingLimitsWide2019arXiv.org}}
This paper serves several purposes.
1) Introduce the reader to the \emph{tensor programs} technique formulated in \cite{yangScalingLimitsWide2019arXiv.org}, using the Neural Network-Gaussian Process Correspondence as motivation.
2) Promote a redesigned set of notations for \emph{tensor programs} that hopefully makes the understanding and the application of this technique easier.
3) Prove a more general version of the Gaussian Process results first presented in \cite{yangScalingLimitsWide2019arXiv.org}.
4) Provide example calculations and reference implementations\repofnt{} of the GP kernels for several architectures like the vanilla RNN, GRU, batchnorm network, and transformers.

We assume the reader has not read \cite{yangScalingLimitsWide2019arXiv.org} and seek to explain all results in elementary terms.
However, we will provide commentary in footnotes throughout the paper on differences from \cite{yangScalingLimitsWide2019arXiv.org}.

Regarding 1), this paper will be the first in a series to explain the \emph{tensor programs} technique, each covering a more powerful type of tensor programs, and each motivated by specific theorems that can be proved or calculations made possible by these new tensor programs.
In particular, here we will only talk about tensor programs without matrix transposes.
Regarding 3), the results presented here will supersede all results in \cite{yangScalingLimitsWide2019arXiv.org} concerning Gaussian Processes, with one caveat that here we will not cover architectures using both a weight $W$ and its transpose $W^\trsp$ in its forward pass (but this result will come for free in a later paper in this series).

\section{Gaussian Process with Variable-Dimensional Output}
\label{sec:GPVarDim}

We first clarify the notion of a Gaussian process with variable dimension output.
\begin{defn}[Gaussian Process]\label{defn:GP}
We say a random function $f: X \to \R^m$ (with fixed dimensional output) is a Gaussian process if for any finite subset $\{x^1, \ldots, x^k\} \sbe X$, the random vector $(f(x^1), \ldots, f(x^k)) \in \R^{m \times k}$ is distributed as a $km$-dimensional Gaussian.
If $f$ has variable dimensional output (e.g. $f$ is an RNN), such as when $f(x) \in \R^{l(x)}$ for some length function $l: X \to \N$ \footnote{i.e. $f: \prod_{x \in X} \R^{l(x)}$ is a dependent function}, then we say $f$ is a Gaussian process if for any finite subset $\{x^1, \ldots, x^k\} \sbe X$, the random vector $(f(x^1), \ldots, f(x^k))$ is distributed as a $(\sum_i l(x^i))$-dimensional Gaussian.
\end{defn}

To illustrate a GP with variable-dimensional output, consider a simple RNN that runs on two input sequences given by the GloVe embeddings \cite{pennington_glove:_2014} \footnote{The embedding associates each word to a real vector of 100 dimensions such that semantically similar words are mapped to closer vectors} of the words of the two sentences
\begin{equation}
\begin{aligned}
    \text{sentence 1 (7 words):} &\quad \text{``The brown fox jumps over the dog.''}\\
    \text{sentence 2 (9 words):} &\quad \text{``The quick brown fox jumps over the lazy dog.''}
\end{aligned}
\tag{$\star$}\label{sentences}
\end{equation}
A pseudocode is given in \cref{tp:RNN} in \cref{sec:netsorprograms} (ignore the type annotations like $\Gtype(n), \Htype(n), \Atype(n)$ for now).
The RNN emits a single scalar after reading each token (in \cref{tp:RNN}, this is $v^\trsp s^{ia}/\sqrt n$, where $s^{ia}$ is the RNN state after reading the $i$th token of the $a$th sentence, and $v$ is the readout layer); this number takes into account all of the word embeddings read so far.
Thus, it will output a total of 7 scalars after reading sentence 1, and a total of 9 scalars after reading sentence 2.
To say that this RNN is a GP would imply that all $7+9=16$ scalars are jointly Gaussian-distributed (corresponding to a $16 \times 16$ kernel), over the randomness of the weights and biases imbued during initialization.
This is indeed the empirical phenomenon with a width-1000 RNN, and \cref{fig:randRNN}(E) visualizes the the joint distribution of the last scalars output by the RNN at the end of each sentence.
It clearly exhibits a Gaussian nature, and perfectly fits the theoretically predicted Gaussian distribution (dashed ovals), which we shall describe in \cref{cor:GPConv}.

\section{Recap: GP Behavior of a Multilayer Perceptron (MLP)}
Before explaining our main results, we first review the argument from prior works \citep{lee_deep_2018,matthews_gaussian_2018,novak_bayesian_2018} for the GP convergence of a wide MLP with randomly initialized weights and biases, and we also demonstrate why such an argument is inadequate for RNNs.
Consider an MLP with widths $\{n^l\}_l$, weight matrices $\{W^l \in \R^{n^l \times n^{l-1}}\}_l$, and biases $\{b^l \in \R^{n^l}\}_l$, where $l$ ranges among the layer numbers of the MLP.
Its computation is given recursively as
\begin{align*}
    h^1(x) = W^1 x + b^1 \hspace{2pc}\text{and}\hspace{2pc}  h^l(x) = W^l\phi( h^{l-1}(x)) + b^l \text{ for $l \ge 2$}.
    \numberthis\label{eqn:simpleMLP}
\end{align*}
At initialization time, suppose $W^l_{\alpha \beta} \sim \Gaus(0, \sigma_w^2/n^{l-1})$ for each $\alpha \in [n^l], \beta\in[n^{l-1}]$, and $b^l_\alpha \sim \Gaus(0, \sigma_b^2)$.
Consider two inputs $x, x'$.
Conditioned on $h^{l-1}(x)$ and $h^{l-1}(x')$, iid for each $\alpha$, $(h^l(x)_\alpha, h^l(x')_\alpha)$ is distributed as
\[\Gaus\left(0, \f{\sigma_w^2}{n^{l-1}} \begin{pmatrix}
            \|\phi(h^{l-1}(x))\|^2 &   \phi(h^{l-1}(x)) \cdot \phi(h^{l-1}(x'))\\
            \phi(h^{l-1}(x)) \cdot \phi(h^{l-1}(x'))  &   \|\phi(h^{l-1}(x'))\|^2
            \end{pmatrix}
            + \sigma_b^2.
            \right)\]
If $(h^{l-1}(x)_\alpha, h^{l-1}(x')_\alpha)$ is distributed as $\Gaus(0, \Sigma^{l-1})$, iid for each $\alpha$, then by a law of large number argument, the covariance matrix above converges to a deterministic limit 
\[\Sigma^l \defeq \sigma_w^2\EV_{(z,z') \sim \Gaus(0, \Sigma^{l-1})} \begin{pmatrix}
\phi(z)^2 &   \phi(z) \phi(z')\\
\phi(z)\phi(z') &   \phi(z')^2
\end{pmatrix}
+\sigma_b^2\]
as the width $n^{l-1} \to \infty$, making $(h^l(x)_\alpha, h^l(x')_\alpha)$ Gaussian distributed as $\Gaus(0, \Sigma^l)$.
Iteratively applying this argument for each $l$ yields the result for a deep MLP.
A similar logic works for feedforward CNNs.

Unfortunately, this argument breaks down if the weights $\{W^l\}_l$ are tied, i.e. all $W^l$ are equal to a common matrix $W$, as in the case of an RNN.
In this case, when we condition on the preactivations $h^{l-1}(x), h^{l-1}(x')$ of the previous layer, $W$ is no longer conditionally an iid random Gaussian matrix, and all subsequent reasoning breaks down.
We can repair this situation for RNNs in an ad hoc way via the Gaussian conditioning technique (\cref{lemma:condTrick}), but we prefer to set our sights wider, and deal with all standard architectures, and more, in one fell swoop.
To this end, we develop a framework based on our new \netsor language.

\section{\texorpdfstring{\netsor}{Netsor}: Language for Expressing Neural Network Computation}
\label{sec:netsorprograms}

To show that networks of all standard architectures converge to GPs, we first show that they can be expressed by the following very general \netsor language (see \cref{tp:MLP,tp:RNN} for examples)\footnote{\netsor is a specific kind of tensor program; for other variants, see \cref{sec:VersionTensorPrograms}.}, and then show that any computation expressed this way exhibits GP behavior when its dimensions are large.
\begin{defn}\label{defn:netsor}
\footnote{We keep the definition here informal in terms of programming language convention to be accessible to the general machine learning audience. For those with PL background, see \cref{sec:formalspec}.}
\textit{\netsor programs} are straight-line programs, where each variable follows one of three types, $\Gtype, \Htype$, or $\Atype$ (such variables are called \emph{G-vars}, \emph{H-vars}, and \emph{A-vars}), and after input variables, new variables can be introduced by one of the rules \ref{linetype:MatMul}, \ref{linetype:lincomb}, \ref{linetype:nonlin} to be discussed shortly.
$\Gtype$ and $\Htype$ are \emph{vector types} and $\Atype$ is a \emph{matrix type}; intuitively, G-vars should be thought of as vectors that are asymptotically Gaussian, H-vars are images of G-vars by coordinatewise nonlinearities, and A-vars are random matrices with iid Gaussian entries.
Each type is annotated by dimensionality information:
\begin{itemize}
    \item If $x$ is a (vector) variable of type $\Gtype$ (or $\Htype$) and has dimension $n$, we write $x: \Gtype(n)$ (or $x: \Htype(n)$).
    \item If $A$ is a (matrix) variable of type $\Atype$ and has size $n_1 \times n_2$, we write $A: \Atype(n_1, n_2)$.
\end{itemize}
$\Gtype$ is a \emph{subtype} of $\Htype$, so that $x: \Gtype(n)$ implies $x: \Htype(n)$.
A \netsor program consists of the following three parts.
\begin{description}
    \item[Input]
        A set of input G- or A-vars.
    \item[Body]
        New variables can be introduced and assigned via the following rules (with {\it intuition in italics})
    \begin{description}
        \item[\texttt{MatMul}\label{linetype:MatMul}] if $A: \Atype(n_1, n_2)$ and $x: \Htype(n_2)$, we can form a G-var via matrix-vector product:
        \[A x : \Gtype(n_1),\quad \text{\it ``random iid matrix times a vector is roughly a Gaussian vector.''}\footnote{Beware: in a later paper (and in \cite{yangScalingLimitsWide2019arXiv.org}, tensor program general case), we will introduce matrix transpose as a valid operation, and in that case, $Ax$ can be very far from a Gaussian, and this intuition is no longer correct.
        Thus, this intuition is more subtle than it might seem at face value.}
        \]
        \item[\texttt{LinComb}\label{linetype:lincomb}] Suppose $x^1, \ldots, x^k: \Gtype(n)$ are G-vars with the same dimension and $a_1, \ldots a_k \in \R$ are constants.
        Then we can form their linear combination as a G-var:
        \[\sum_{i=1}^n a_i x^i : \Gtype(n),\quad
        \text{\it ``linear combination of Gaussian vectors is Gaussian.''}\]
        \item[\texttt{Nonlin}\label{linetype:nonlin}] If $x^1, \ldots, x^k: \Gtype(n)$ are G-vars with the same dimension $n$ and $\phi: \R^k \to \R$, then
        \[\phi(x^1, \ldots, x^k): \Htype(n),\quad
        \text{\it ``image of Gaussian vector is not always Gaussian''}\]
        where $\phi$ acts coordinatewise.
    \end{description}
    \item[Output]
        For the purpose of this paper\footnote{In general, the output of a tensor program need not be defined, as most of the time we are concerned with how the H-vars produced over the course of the program interact with each other.}, the output of a \netsor program can be any tuple of scalars, $(v^1{}^\trsp y^1/\sqrt{n_1}, \ldots, v^k{}^\trsp y^k/\sqrt{n_k})$, where $v^1: \Gtype(n_1); \ldots; v^k: \Gtype(n_k)$ are some input G-vars not used elsewhere (and possibly with duplicates $v^i = v^j$), and $y^1: \Htype(n_1); \ldots; y^k: \Htype(n_k)$ are some H-vars (possibly with duplicates $y^i = y^j$).
\end{description}
\end{defn}

\begin{algorithm}[tb]
    \caption{MLP Computation on Network Input $x$}
    \label{tp:MLP}
    \begin{algorithmic}[1]
      \Require $W^1 x: \Gtype(n^1)$ \Comment{layer 1 embedding of input}
      \Require $b^1: \Gtype(n^1)$ \Comment{layer 1 bias}
      \Require $W^2: \Atype(n^2, n^1)$ \Comment{layer 2 weights}
      \Require $b^2: \Gtype(n^2)$ \Comment{layer 2 bias}
      \Require $v: \Gtype(n^2)$ \Comment{readout layer weights}
      \State $h^1 := W^1 x + b^1: \Gtype(n^1)$ \Comment{layer 1 preactivation; \ref{linetype:lincomb}}
      \State $x^1 := \phi(h^1): \Htype(n^1)$ \Comment{layer 1 activation; \ref{linetype:nonlin}}
      \State $\tilde h^2 := W^2 x^1: \Gtype(n^2)$ \Comment{\ref{linetype:MatMul}}
      \State $h^2 := \tilde h^2 + b^2: \Gtype(n^2)$ \Comment{layer 2 preactivation; \ref{linetype:lincomb}}
      \State $x^2 := \phi(h^2): \Htype(n^2)$ \Comment{layer 2 activation; \ref{linetype:nonlin}}
      \Ensure $v^\trsp x^2/\sqrt{n^2}$
    \end{algorithmic}
\end{algorithm}

\begin{algorithm}[t]
    \caption{Simple RNN Computation on Two Input Sequences}
    \label{tp:RNN}
    \begin{multicols}{2}
    \begin{algorithmic}
      \State {\it // Embeddings of two inputs sequences}
      \Require $U x^{11}, \ldots, U x^{t1}: \Gtype(n)$
      \Require $U x^{12}, \ldots, U x^{r2}: \Gtype(n)$
      \State {\it // Weight and bias}
      \Require $W: \Atype(n, n)$
      \Require $b: \Gtype(n)$
      \State {\it // Readout weights}
      \Require $v: \Gtype(n)$
      \State {\it // Computation on sequence 1}
      \State $h^{11} := U x^{11} + b: \Gtype(n)$
      \State $s^{11} := \phi(h^{11}): \Htype(n)$
      \State $\tilde h^{21} := W s^{11}: \Gtype(n)$
      \State $h^{21} := \tilde h^{21} + U x^{21} + b: \Gtype(n)$
      \State $s^{21} := \phi(h^{21}): \Htype(n)$
      \State $\vdots$
      \State $\tilde h^{t1} := W s^{t-1,1}: \Gtype(n)$
      \State $h^{t1} := \tilde h^{t1} + U x^{t1} + b: \Gtype(n)$
      \State $s^{t1} := \phi(h^{t1}): \Htype(n)$
      \State {\it // Computation on sequence 2}
      \State $h^{12} := U x^{12} + b: \Gtype(n)$
      \State $s^{12} := \phi(h^{12}): \Htype(n)$
      \State $\tilde h^{22} := W s^{12}: \Gtype(n)$
      \State $h^{22} := \tilde h^{22} + U x^{22} + b: \Gtype(n)$
      \State $s^{22} := \phi(h^{22}): \Htype(n)$
      \State $\vdots$
      \State $\tilde h^{r2} := W s^{r-1, 2}: \Gtype(n)$
      \State $h^{r2} := \tilde h^{r2} + U x^{r2} + b: \Gtype(n)$
      \State $s^{r2} := \phi(h^{r2}): \Htype(n)$
      \Ensure $(v^\trsp s^{11}/\sqrt{n}, \ldots, v^\trsp s^{t1}/\sqrt{n},$
      \State \qquad$v^\trsp s^{12}/\sqrt{n}, \ldots, v^\trsp s^{r2}/\sqrt{n})$
    \end{algorithmic}
    \end{multicols}
\end{algorithm}

\paragraph{Examples}
\cref{tp:MLP} gives an example of a \netsor program representing an MLP computation.
Note that \emph{we account for the input $x$ through its embedding $W^1 x$, not $x$ itself.}
This is because 1) our theorems concern the case where all input G-vars are random;
in the context of expressing neural network computation, $x$ is a deterministic input, while
$W^1x$ is a Gaussian vector when $W^1$ has iid Gaussian entries; 2) $x$ has a fixed dimension, while we intend all dimensions (like $n^1, n^2$) in the \netsor program to tend to infinity, as we'll describe shortly.
Similarly, \cref{tp:RNN} expresses in \netsor the computation of a simple RNN on two separate input sequences; computation on more input sequences follows the same pattern.
Note how weight-sharing is easily expressed in \netsor because we can re-use A-vars arbitrarily.
\cref{sec:MoreExamples} shows more examples of standard architectures in \netsor and \netsorplus.

More generally, we can allow the nonlinearities in \ref{linetype:nonlin} to depend on parameters; this will be necessary to express layernorm and attention (see \cref{sec:MoreExamples}).
We capture this idea in a new rule:
\begin{description}
\item[\texttt{Nonlin$^+$}\label{linetype:nonlin+}]
Suppose $x^1, \ldots, x^k: \Gtype(n)$ are G-vars with the same dimension $n$ and $\theta_1, \ldots, \theta_t \in \R$ possibly depend on G-vars already defined.
If $\phi(-; -): \R^k \times \R^t \to \R$, then
        \[\phi(x^1, \ldots, x^k; \theta_1, \ldots, \theta_t): \Htype(n),\]
        where $\phi$ acts coordinatewise.
\end{description}
\begin{defn}\label{defn:netsor+}
\netsorplus programs are \netsor programs allowing \ref{linetype:nonlin+} rules.
\end{defn}
\netsor and \netsorplus specify different kinds of \emph{tensor programs}; in \cref{sec:VersionTensorPrograms} we discuss other kinds that are semantically equivalent.
In a future paper, we shall study the effect of allowing matrix transposes as an operation on A-vars.

\begin{remk}
In this paper, in \ref{linetype:nonlin+}, we will only instantiate $\theta_j$ with continuous functions of ``empirical moments'' of the form $n^{-1} \sum_{i=1}^n \psi(y^1, \ldots, y^r)$ for some set of G-vars $\{y_i\}_i$.
A key consequence of our scaling limit result is that these ``empirical moments'' converge almost surely to a deterministic limit under very general conditions (\cref{thm:netsorMasterTheorem,thm:Netsor+MasterTheorem}), so that $\phi(-; \bigtheta)$ is, under suitable smoothness conditions (\cref{defn:parameterControlled}), approximately a fixed nonlinearity when $n$ is large.
Thus, we should intuitively treat \ref{linetype:nonlin+} as \ref{linetype:nonlin} but with the nonlinearity determined automatically by the \netsor program itself.
\end{remk}

\ref{linetype:nonlin+} expands the expressible computation quite broadly, but to keep the main text lean and focused on the key ideas behind tensor programs, we relegate a more thorough discussion of \ref{linetype:nonlin+} in the appendix (see \cref{sec:netsorplusMasterTheorem,sec:netsorplusKernelExample,sec:netsorplusMasterTheoremProof}).

\section{Computing the GP Kernel from a \texorpdfstring{\netsor\!}{Netsor} Encoding of a Neural Network}

\newcommand{\Sigmain}{\Sigma^{\mathrm{in}}}
\newcommand{\muin}{\mu^{\mathrm{in}}}
\newcommand{\tSigma}{\Sigma}
\newcommand{\tmu}{\mu}

For readers who wish to be convinced that \netsor (or \netsorplus) can express standard architectures, see \cref{sec:MoreExamples}.
In this section, we show that any architecture expressible in \netsor and satisfies some mild conditions will exhibit Gaussian Process behavior in the large width limit.

In this section, we make the following simplifying assumption on the dimensions of the program and the randomness of the variables.
\begin{assm}
\label{assm:equalDimNoNonlin+}
Fix a \netsor program.
For simplicity, assume all dimensions in the program are equal to $n$.
Suppose for each A-var $W: \Atype(n, n)$, we sample $W_{\alpha \beta} \sim \Gaus(0, \sigma_W^2/n)$ for some $\sigma_W^2 > 0$, and for each
$\alpha \in [n]$, we sample, i.i.d., $\{x_\alpha: x \text{ is input G-var}\} \sim \Gaus(\muin, \Sigmain)$ for some mean $\muin$ and (possibly singular) covariance $\Sigmain$ over input G-vars.
\end{assm}

The constraint on the dimensions can be removed easily; see \cref{sec:VariableDim}.
This sampling induces randomness in all variables created in the program, and we shall characterize this randomness shortly.
We first review some notation that will be used immediately.

\paragraph{Notation}
In this paper, a \emph{kernel $\Sigma$ on a set $X$} is a symmetric function $\Sigma: X \times X \to \R$ such that
\begin{align*}
\sum_{i=1}^m \sum_{j=1}^m c_i c_j \Sigma(x_i, x_j) \ge 0
\end{align*}
holds for any $m \in \N$, $x_1, \ldots, x_m \in X$, and $c_1, \ldots, c_m \in \R$.
Given a kernel $\Sigma$ on a set of G-vars, we will both treat it as matrix and as a function, depending on the context.
\begin{description*}
\item[Function Notation]
    As a function, $\Sigma(g, g')$ is the value of $\Sigma$ on the pair of G-vars $(g, g')$.
    If $G = \{g^1, \ldots, g^k\}$ is a set of G-vars, then we also denote by $\Sigma(g, G)$ the row vector $\{\Sigma(g, g^1), \ldots, \Sigma(g, g^k)\}$.
    Likewise $\Sigma(G, g)$ is the column vector with the same values.
    If $G' = \{g^1{}', \ldots, g^l{}'\}$ is another set of G-vars (possible with overlap with $G$), then
    $\Sigma(G, G')$ is the matrix $\{\Sigma(g^i, g^j{}'): i \in [k], j \in[l]\}$.
\item[Restriction Notation]
    We also use the ``restriction'' notation $\Sigma |_G$ to denote the square matrix $\Sigma(G, G)$ in a more concise way.
\item[Matrix Notation]
    When an association of indices to G-vars is clear from context, we will also write $\Sigma_{ij}$ for the corresponding value of $\Sigma$ on the pair of $i$th and $j$th G-vars.
    Juxtaposition implies matrix multiplication, e.g.\ $\Sigma \Omega$ means matrix product if $\Omega$ is a matrix of appropriate size.
\item[Indices Notation]
    We will both use superscripts and subscripts for indices.
    We will never multiply in subscript or superscript, so juxtaposition of indices like $W^{ib}_{\alpha \beta}$ is the same as $W^{i,b}_{\alpha,\beta}$.
\item[H-vars as Both Symbols and Vectors]
    An H-var will be considered both as a symbol (like in $\Sigma(g, g')$ above) as well as the corresponding length $n$ vector (like in \cref{thm:netsorMasterTheorem} below), depending on the context.
\end{description*}

\begin{defn}
In the setting of \cref{assm:equalDimNoNonlin+},
we extend $\muin$ and $\Sigmain$ to $\tmu$ and $\tSigma$ that resp.\ take a single and a pair of G-vars and both output to $\R$.
Intuitively, $\tmu$ specifies the mean coordinate of a G-var, and $\tSigma$ specifies the coordinatewise covariance of a pair of G-vars;
this is formalized in \cref{thm:netsorMasterTheorem} below.
Index all the G-vars in the program as $g^1, \ldots, g^M$ (including input G-vars), in the order of appearance in the program.
For any pair of G-vars $g, g'$ (among $g^1, \ldots, g^M$), we define recursively
\begin{align*}
    \tmu(g)
        &=
            \begin{cases}
            \muin(g)  &   \text{if $g$ is input}\\
            \sum_{i} a_i \tmu(y^i)    &   \text{if $g = \sum_{i} a_i y^i$, introduced by \ref{linetype:lincomb}}\\
            0   &   \text{otherwise}
            \end{cases},
            \\
    \tSigma(g, g')
        &=
            \begin{cases}
            \Sigmain(g, g')   &   \text{if $g, g'$ are inputs}\\
            \sum_{i} a_i \tSigma(y^i, g')    &   \text{if $g = \sum_{i} a_i y^i$, introduced by \ref{linetype:lincomb}}\\
            \sum_{i} a_i \tSigma(g, y^i)    &   \text{if $g' = \sum_{i} a_i y^i$, introduced by \ref{linetype:lincomb}}\\
            \sigma^2_W \EV_Z \phi(Z) \bar \phi(Z) &   \text{if $g = Wh, g'=Wh'$, introduced by \ref{linetype:MatMul} w/ same A-var $W$}\\
            0   &   \text{otherwise}
            \end{cases}
            \numberthis\label{eqn:extendedMuSigma}
\end{align*}
where
\begin{itemize}
    \item $y^i$ are G-vars for all $i$
    \item $(h: \Htype(n))$ was introduced by the \ref{linetype:nonlin} with $h := \phi(g^1, \ldots, g^M)$, $h'$ was introduced by \ref{linetype:nonlin} with $h' :=\bar\phi(g^1, \ldots, g^M)$ (where WLOG we have padded the input slots of     $\phi$ and $\bar\phi$ to account for all G-vars)
    \item $Z \sim \Gaus(\tmu, \tSigma)$ is a random Gaussian vector with 1 entry for each G-var in the program.
\end{itemize}
Note that since $\phi$ and $\bar\phi$ only depends on entries of $Z$ corresponding to previous G-vars, the expectation $\EV_Z \phi(Z) \bar\phi(Z)$ only depends on entries of $\tmu$ and $\tSigma$ already defined, so there is no circular logic in this recursive definition of $\tmu$ and $\tSigma$.
See \cref{sec:MLPsingleinput} for a simple, worked-out example of how to recursively compute $\tmu$ and $\tSigma$ for \cref{tp:MLP}.
\end{defn}

For our main theorems, we isolate a very general class of nonlinearities that we are concerned with.
\begin{defn}\label{defn:controlled}
We say a function $\phi: \R^k \to \R$ is \emph{controlled} if $|\phi(x)|$ is bounded by a function of the form $e^{C\|x\|^{2-\epsilon} + c}$ with $C, c, \epsilon > 0$
\end{defn}
Controlled functions can explode faster than exponential but are still $L^1$ and $L^2$-integrable against Gaussian measures.
Additionally, there is no constraint on the smoothness of $\phi$ here.
Thus this definition captures almost all functions we would care about in practice.

The metric structure of the final layer representations of inputs under a deep neural network often reveals semantical information about the inputs.
This structure is reflected in the inner products between pairs of such representations, e.g.\ $s^{t1}{}^\trsp s^{r2}/n$ for $s^{t1}$ and $s^{r2}$ in \cref{tp:RNN}.
The following Master Theorem allows one to compute such inner products, and much more, for a wide network at initialization time (take $\psi$ below to be $\psi(z^1, \ldots, z^M) \defeq z^{M-1} z^M$).

\begin{restatable}[\netsor Master Theorem]{thm}{momentConvergence}
\label{thm:netsorMasterTheorem}
\footnote{Difference with \cite[Thm 4.3]{yangScalingLimitsWide2019arXiv.org}: We have gotten rid of the ``rank convergence'' assumption by showing that it comes for free.
See \ref{IH:coreSet} and \cref{lemma:rankStability} in \cref{sec:proofMasterTheorems}.}
Fix any \netsor program satisfying \cref{assm:equalDimNoNonlin+} and with all nonlinearities controlled.
If $g^1, \ldots, g^M$ are all of the G-vars in the entire program, including all input G-vars, then for any controlled $\psi: \R^M \to \R$, as $n \to \infty$,
\begin{align*}
    \f 1 n \sum_{\alpha=1}^n \psi(g^1_\alpha, \ldots, g^M_\alpha) \asto 
    \EV_{Z \sim \Gaus(\tmu, \tSigma)}\psi(Z)
    =
    \EV_{Z \sim \Gaus(\tmu, \tSigma)}\psi(Z^{g^1}, \ldots, Z^{g^M}),
\end{align*}
where $\asto$ means almost sure convergence,
$Z = (Z^{g^1}, \ldots, Z^{g^M}) \in \R^M$, and $\tmu = \{\tmu(g^i)\}_{i=1}^M \in \R^M$ and $\tSigma = \{\tSigma(g^i, g^j)\}_{i,j=1}^M \in \R^{M \times M}$ are given in \cref{eqn:extendedMuSigma}.
See \cref{fig:mastertheoremIllustration} for an illustration.
\end{restatable}

\begin{figure}
    \centering
    \includegraphics[width=\textwidth]{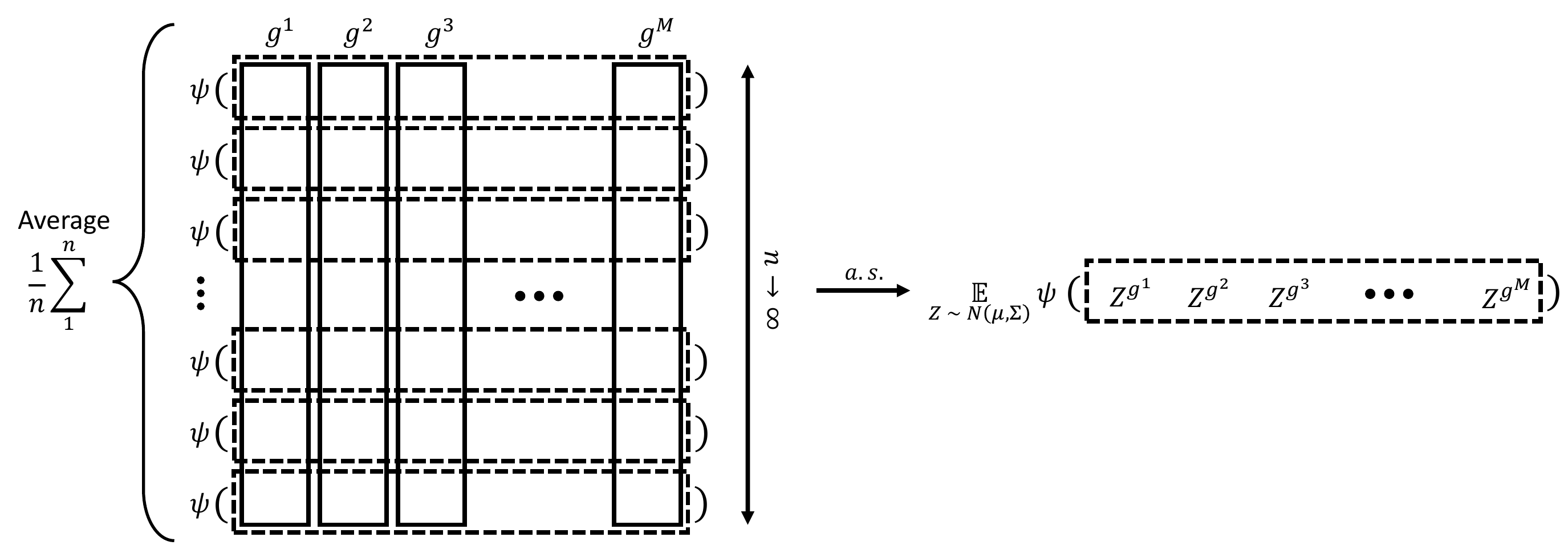}
    \caption{An illustration of the \netsor Master Theorem \cref{thm:netsorMasterTheorem}.}
    \label{fig:mastertheoremIllustration}
\end{figure}

Intuitively, \cref{thm:netsorMasterTheorem} says, for each $\alpha$, $(g^1_\alpha, \ldots, g^M_\alpha) \approx \Gaus(\tmu, \tSigma)$ in the large $n$ limit, and each $\alpha$-slice appears to be ``iid'' from the point of view of the empirical average by any controlled function $\psi$.
The proof of \cref{thm:netsorMasterTheorem} is given in \cref{sec:proofMasterTheorems}.

Combining \cref{thm:netsorMasterTheorem} with \cref{prop:gaussianDistConvFromMomentConv}, we can straightforwardly calculate the output distribution of a \netsor program.
\begin{cor}[Computing the GP Kernel]\label{cor:GPConv}
Adopt the same assumptions and notations as in \cref{thm:netsorMasterTheorem}.
If the program outputs $(v{}^\trsp x^1/\sqrt{n}, \ldots, v{}^\trsp x^k/\sqrt{n})$, where
\begin{itemize}
    \item $v: \Gtype(n), v_\alpha \sim \Gaus(0, \sigma_v^2),$ is an input G-var not used elsewhere in the program and is sampled independently from all other G-vars, and
    \item $x^i$ was introduced as $x^i := \phi^i(g^1, \ldots, g^M)$
\end{itemize}
then the output vector converges in distribution to $\Gaus(0, \KK)$ where
\begin{equation}
    \KK_{ij} = \sigma_v^2 \EV_{Z \sim \Gaus(\tmu, \tSigma)} \phi^i(Z) \phi^j(Z),
    \quad \text{with $\tmu, \tSigma$ defined in \cref{eqn:extendedMuSigma}.}
    \label{eqn:limitingCovarianceGP}
\end{equation}
\end{cor}
Intuitively, this corollary follows from the fact that, for any finite $n$, the output vector is some Gaussian $\Gaus(0, \tilde K)$ conditioned on $x^1, \ldots, x^k$, and the covariance $\tilde K$ converges to a deterministic covariance $\KK$, causing the output vector to converge in distribution to $\Gaus(0, \KK)$ as well.
The case when we have multiple distinct $v^i$ (allowed by \cref{defn:netsor}) can be obtained easily as well (see \cref{prop:gaussianDistConvFromMomentConv}).

Following \cref{cor:GPConv} and its extensions below, the convergence of standard architectures to Gaussian Processes becomes obvious:
Express the marginal of the distribution on every finite set of inputs as a \netsor (or \netsorplus) program, and then apply \cref{cor:GPConv}.
We summarize the result below.
\begin{cor}\label{cor:GPConvergence}
If its nonlinearities are controlled (\cref{defn:controlled}), then a (possibly recurrent) neural network of standard architecture converges to a Gaussian process in finite-dimensional distribution \footnote{Stronger convergence results, such as convergence in distribution with respect to some topology on functions on $\R^d$, would be available if one can show additionally the \emph{tightness} of the random neural networks under this topology.
However, here we are content with convergence of finite-dimensional marginals of the stochastic processes.
} in the sense of \cref{defn:GP} as its widths go to infinity and each of its weights $W$ and biases $b$ are randomized as $W_{\alpha \beta} \sim \Gaus(0, \sigma_W^2/n), b_\alpha \sim \Gaus(\mu_b, \sigma_b^2)$ for a collection of sampling hyperparameters $\{\sigma_W\}_W, \{\mu_b, \sigma_b\}_b$.
If its nonlinearities are more generally parametrized and are parameter-controlled (\cref{defn:parameterControlled}), such as in the case of attention models or where layernorm is involved, then the same result holds as long as \cref{assm:asRankStab} also holds.
\end{cor}

\begin{figure}
    \centering
    \includegraphics[width=\textwidth]{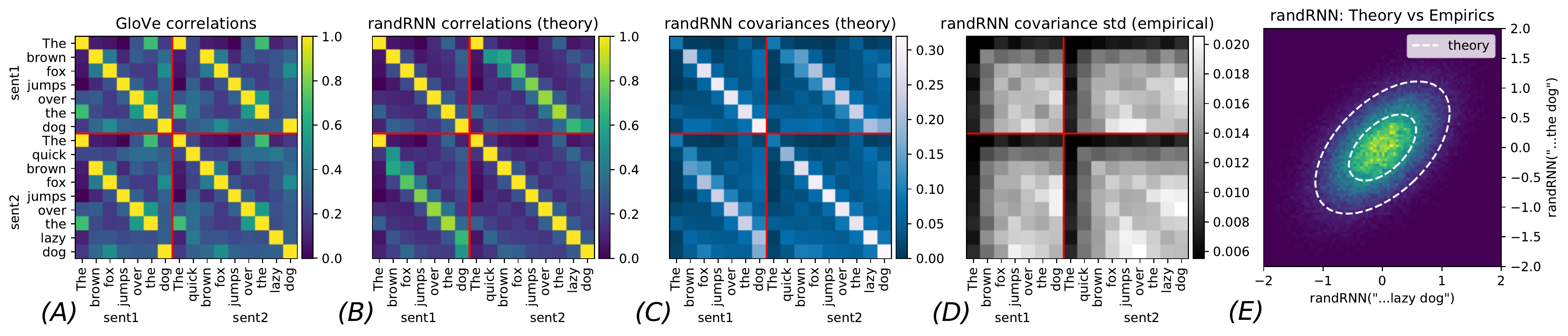}
    \caption{\emph{Infinite-width theory is highly predictive for simple RNN (\cref{tp:RNN}) with 1000 neurons and erf activation.}
    We pass two sentences (``The brown fox jumps over the dog'' and ``The quick brown fox jumps over the lazy dog'') by their word GloVe embeddings into randomly initialized simple RNNs.
    \textbf{(A)} Cosine distances between each pair of word GloVe embeddings.
    \textbf{(B)} Correlation matrix of the limiting Gaussian that \cref{tp:RNN} output vector converges to.
    Each row/column corresponds to an embedding of of the sentence up to that word.
    \textbf{(C)} \emph{Covariance} matrix of the same.
    See \cref{sec:ExactKernelRNN} for algorithm to compute this covariance.
    \textbf{(D)} Entrywise standard deviation of empirical covariance around (C), as measured from 100 random simple RNNs.
    Note the deviations are at least an order of magnitude smaller than the limiting values (C), for 1000 neurons.
    \textbf{(E)} Visualizing the joint distribution of the final outputs of the RNN at the end of each sentence, i.e. $(v^\trsp s^{t1}/\sqrt n, v^\trsp s^{r2}/\sqrt n)$ in \cref{tp:RNN}.
    We sampled 100,000 simple RNNs and plotted the 2d histogram as a heatmap.
    Simultaneously, we plot the limiting Gaussian level curves predicted by our theory, which fit the simulations perfectly.
    }
    \label{fig:randRNN}
\end{figure}

\paragraph{An Empirical Demonstration}
Despite being about infinite width, our theory is highly predictive for finite-width networks, as shown in \cref{fig:randRNN}.
As in \cref{sec:GPVarDim}, we randomly initialize a simple RNN (\cref{tp:RNN}) with 1000 neurons and erf activation (we choose erf instead of tanh because it simplifies kernel calculations; see \cref{sec:ExactKernelRNN} for the derivation of the algorithm to compute the kernel).
We pass the two sentences in (\ref{sentences}) to the random RNN by their GloVe embeddings.
After processing each token, the RNN outputs a scalar, as before, and over the two input sequences, the RNN outputs $7+9=16$ scalars in total.
Our result \cref{cor:GPConv} implies that, as the width of the RNN grows to infinity, these 16 scalars are distributed jointly as a Gaussian.
\cref{fig:randRNN}(E) illustrates this is indeed the case for the marginal on 2 scalars, as discussed in \cref{sec:GPVarDim}.
We also compare our theoretically derived, infinite-width covariance of the 16 scalars (\cref{fig:randRNN}(C)) with the empirical finite-width covariance obtained from multiple random initializations.
We find that the empirical covariance, as predicted, concentrates around the theoretical, and the entrywise standard deviation is typically at least an order of magnitude lower than the values themselves (\cref{fig:randRNN}(D)) (with width 1000 RNNs).
The random RNN is clearly doing nontrivial context embedding, as seen by comparing the \emph{correlation} matrix of the 16 scalars \cref{fig:randRNN}(B) (context-sensitive) with the matrix of cosine distances (i.e. correlations) between the GloVe embeddings \cref{fig:randRNN}(A) (context-insensitive).
A tell-tale sign is the entry corresponding to (``lazy'', ``dog''): even though as words, they are not semantically similar (so that the entry in \cref{fig:randRNN}(A) is small), the random RNN understands that the two sentences resp. up to ``lazy'' and ``dog'' have been very similar (so that the entry in \cref{fig:randRNN}(B) is large).
Given the precision of our theoretical predictions, we expect analyses of the equations derived here will lead to many nontrivial insights about recurrent (and other) neural network behavior in practice, which we leave for future work.

\paragraph{Examples and Extensions: A Brief Guide to the Appendix}
\cref{sec:GPKernelComputationExamples} contains a plethora of worked-out examples of the kernel computation for different architectures, starting from the known case of MLP to the new results of RNN (as shown in \cref{fig:randRNN}), GRU, batchnorm, and others.
At this point, we recommend the reader to follow along some of those examples to solidify the understanding of \cref{thm:netsorMasterTheorem}.

A Master Theorem for \netsorplus can be similarly proved.
This is stated in \cref{sec:netsorplusMasterTheorem} and can be proved easily given the proof of \cref{thm:netsorMasterTheorem}; see \cref{sec:netsorplusMasterTheoremProof}.
\cref{sec:netsorplusKernelExample} works out examples of kernel computations for layernorm and transformer, which can only be expressed through \netsorplus\!.
\cref{fig:allkernels} illustrates the kernels of simple RNN, GRU, transformer, and a batchnorm+ReLU network, and confirms that the finite width simulations tend to the infinite-width, theoretical kernels.

We also discuss different variants of \netsor and \netsorplus in \cref{sec:VersionTensorPrograms} which trade off syntactical simplicity with ease of use, but are semantically equivalent to \netsor or \netsorplus\!.
\cref{sec:VariableDim} discusses the case when the dimensions of a program need not be equal.
With the appropriate setup, a Master Theorem in this case can be proved similarly (\cref{thm:netsorMasterTheoremVarDim}).

\begin{figure}
\centering
\includegraphics[width=\textwidth]{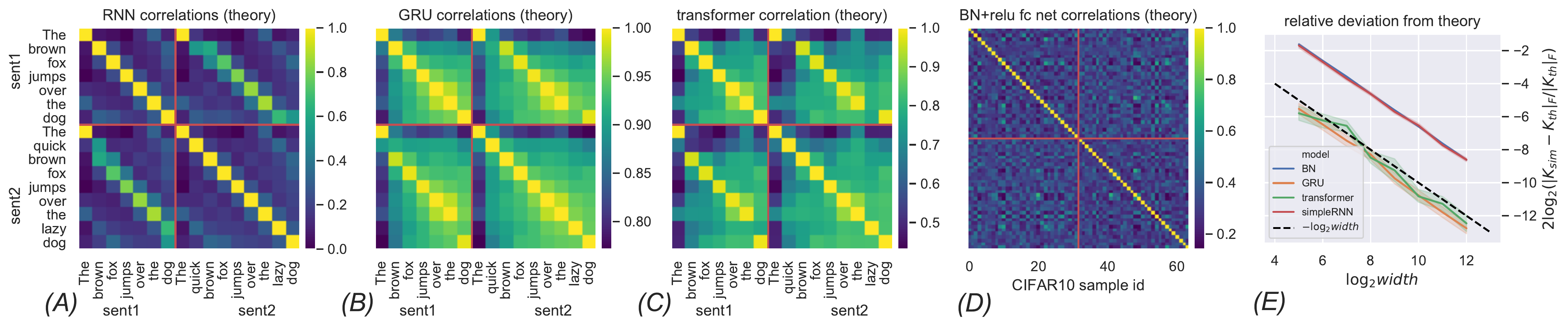}
\caption{\emph{Infinite-width GP kernels (more precisely, their correlation matrices) for which we provide reference implementations, and the deviation of finite-width simulations from the corresponding infinite-width limits.}
\textbf{(A) -- (C)}
The correlation matrices of the GP kernels for the simple RNN (same as in \cref{fig:randRNN}; see \cref{tp:RNN} for the architecture and \cref{sec:ExactKernelRNN} for derivation), GRU (\cref{tp:GRU}; \cref{sec:GRUkernel}), and transformer (\cref{tp:transformer}; \cref{sec:transformerKernel}), with input sequences given by the GloVe embeddings of (\ref{sentences}).
\textbf{(D)}
The correlation matrix of the GP kernel for a feedforward, fully-connected network with batchnorm+ReLU (batchnorm followed by ReLU) as nonlinearity (see \cref{sec:batchnormKernel} for derivation).
The inputs are the first 64 CIFAR10 images, split into two batches of 32 each.
The red lines indicate the batch split.
\textbf{(E)}
For each architecture above, we independently randomly initialize 100 networks for each width among $[2^5, 2^6, \ldots, 2^{13}]$.
We calculate the empirical kernel of the network outputs and plot its (relative) Frobenius distance to the infinite-width kernel.
This Frobenius distance drops like $1/\sqrt{width}$ as one would expect from a central limit intuition.
See our code\repofnt{} for \texttt{Python} implementations of these kernels and the code for generating this figure.
}
\label{fig:allkernels}
\end{figure}

\section{Related Works}

\paragraph{NN-GP Correspondence}
Many works have observed the neural network-Gaussian process correspondence (NN-GP correspondence) for subsets of standard architectures \citep{williams_computing_1997,le_roux_continuous_2007,hazan_steps_2015,daniely_toward_2016,lee_deep_2018,matthews_gaussian_2018,novak_bayesian_2018}.
    Others have exploited this NN-GP correspondence implicitly or explicitly to build new models \citep{cho_kernel_2009,lawrence_hierarchical_2007,damianou_deep_2013,wilson_stochastic_2016,wilson_deep_2016,bradshaw_adversarial_2017,van_der_wilk_convolutional_2017,kumar_deep_2018,blomqvist_deep_2018,borovykh_gaussian_2018,garriga-alonso_deep_2018,novak_bayesian_2018}.
    In particular,  by directly converting NN to GP using this correspondence,
    \citet{lee_deep_2018} constructed the state-of-the-art (SOTA) permutation-invariant GP on MNIST, and \citet{novak_bayesian_2018} was until recently SOTA on CIFAR10 for any GP with untrainable kernel.
    Additionally, the NN-GP correspondence has led to new understanding of neural network training and generalization \citep{novak_sensitivity_2018,valle-perezDeepLearningGeneralizes2018arXiv.org,yangFineGrainedSpectralPerspective2019arxiv.org}.

    In this paper, we generalized the NN-GP correspondence to \emph{standard architectures} and very general nonlinearities (controlled functions; see \cref{defn:controlled}).
    In contrast, \citet{matthews_gaussian_2018} requires $\phi$ to be linearly bounded in norm;
    \citet{daniely_toward_2016} requires $\phi$ be twice-differentiable with $|\phi|,|\phi'|, |\phi''|$ all bounded, or that $\phi=$ ReLU;
    and a sufficient condition given in \citet{novak_bayesian_2018} is that $\phi'$ exists and is bounded by $\exp(O(x^{2-\epsilon}))$, though it is unclear how the more general set of 3 conditions given there (in their section E.4) compares with ours.

    \paragraph{Signal Propagation in Neural Networks}
    A long line of work starting with \citet{glorot_understanding_2010} and \citet{he_delving_2015} studies the effect of initialization in deep neural networks \citep{poole_exponential_2016,schoenholz_deep_2017,yang_mean_2017,yang_deep_2018,hanin_how_2018,chen_dynamical_2018,yang_mean_2018,pennington_resurrecting_2017}, for example, what is the best initialization scheme to avoid gradient vanishing?
    These works apply the same calculations of covariances as we do for calculating $\tSigma$ here, though in a much more restricted set of architectures, and they are typically more concerned with the dynamics of such covariances with depth.

    \paragraph{Reservoir Computing}
    In reservoir computing \citep{jaegerEchoStateNetwork2007www.scholarpedia.org,maassRealTimeComputingStable2002DOI.orgCrossref,schrauwenOverviewReservoirComputing2007Zotero}, sequence processing is typically done by a randomly initialized recurrent neural network.
    A sequence of inputs is fed step by step into the network, and a final readout layer transforms the random RNN's state into an output.
    The only trainable parameters are the readout layer, but not the random RNN itself.
    Thus, in the infinite-width limit, reservoir computing corresponds exactly to GP inference with the RNN kernel computed in \cref{sec:ExactKernelRNN}.

\section{Conclusion}

We formulated the notion of Gaussian process with variable-dimensional outputs and showed that randomly initialized, wide feedforward and recurrent neural networks of standard architectures converge in distribution to Gaussian processes in such a sense.
This significantly generalizes prior work on the NN-GP correspondence.
We did so by introducing \netsor\!, a language for expressing computation common in deep learning, including neural networks of standard architecture, along with a theorem (\cref{thm:netsorMasterTheorem}) characterizing the behavior of a \netsor program as its tensors are randomized and their dimensions tend to infinity; many examples and extensions are exhibited in the appendix.
Finally, we empirically verified our theory for simple RNN, GRU, transformer, and batchnorm (\cref{fig:allkernels}) and open-sourced implementations of the corresponding infinite-width limit kernels at \repo.
In the next paper in this series, we will introduce a more powerful version of tensor program that allows matrix transposes, and use this tool to compute Neural Tangent Kernel \cite{jacot_neural_2018} for any architecture.

\newpage

\section*{Acknowlegements}

I'm very thankful for my buddy Hadi Salman who is always ready to help and who donated a lot of time helping me write the detailed examples for MLP and RNN kernels.
I'd also like to thank Mimee Xu who read the first versions of this paper and provided valuable feedback.
Finally, allow me to express my appreciation for myriads of friends and collaborators who have helped me improve this paper in one way or another:
Sam Schoenholz,
Yihe Dong,
Judy Shen,
Alessandro Sordoni,
Huishuai Zhang,
George Phillip,
Vinay Rao,
Sebastien Bubeck,
Zeyuan Allen-Zhu,
Kyle Aitkens,
Chunyuan Li,
Alex Polozov,
Ilya Razenshteyn,
Jianfeng Gao,
Pengchuan Zhang,
Jascha Sohl-Dickstein,
Jeffrey Pennington,
and others.

\bibliography{references,ref2}

\begin{thebibliography}{65}
\providecommand{\natexlab}[1]{#1}
\providecommand{\url}[1]{\texttt{#1}}
\expandafter\ifx\csname urlstyle\endcsname\relax
  \providecommand{\doi}[1]{doi: #1}\else
  \providecommand{\doi}{doi: \begingroup \urlstyle{rm}\Url}\fi

\bibitem[Ba et~al.(2016)Ba, Kiros, and Hinton]{ba_layer_2016}
Jimmy~Lei Ba, Jamie~Ryan Kiros, and Geoffrey~E. Hinton.
\newblock Layer {Normalization}.
\newblock \emph{arXiv:1607.06450 [cs, stat]}, July 2016.
\newblock URL \url{http://arxiv.org/abs/1607.06450}.
\newblock 00329 arXiv: 1607.06450.

\bibitem[Bahdanau et~al.(2014)Bahdanau, Cho, and Bengio]{bahdanau_neural_2014}
Dzmitry Bahdanau, Kyunghyun Cho, and Yoshua Bengio.
\newblock Neural {Machine} {Translation} by {Jointly} {Learning} to {Align} and
  {Translate}.
\newblock \emph{arXiv:1409.0473 [cs, stat]}, September 2014.
\newblock URL \url{http://arxiv.org/abs/1409.0473}.
\newblock arXiv: 1409.0473.

\bibitem[Bayati and Montanari(2011)]{bayati_dynamics_2011}
Mohsen Bayati and Andrea Montanari.
\newblock The dynamics of message passing on dense graphs, with applications to
  compressed sensing.
\newblock \emph{IEEE Transactions on Information Theory}, 57\penalty0
  (2):\penalty0 764--785, February 2011.
\newblock ISSN 0018-9448, 1557-9654.
\newblock \doi{10.1109/TIT.2010.2094817}.
\newblock URL \url{http://arxiv.org/abs/1001.3448}.
\newblock arXiv: 1001.3448.

\bibitem[Blomqvist et~al.(2018)Blomqvist, Kaski, and
  Heinonen]{blomqvist_deep_2018}
Kenneth Blomqvist, Samuel Kaski, and Markus Heinonen.
\newblock Deep convolutional {Gaussian} processes.
\newblock \emph{arXiv preprint arXiv:1810.03052}, 2018.

\bibitem[Bolthausen(2012)]{bolthausen_iterative_2012}
Erwin Bolthausen.
\newblock An iterative construction of solutions of the {TAP} equations for the
  {Sherrington}-{Kirkpatrick} model.
\newblock \emph{arXiv:1201.2891 [cond-mat, physics:math-ph]}, January 2012.
\newblock URL \url{http://arxiv.org/abs/1201.2891}.
\newblock arXiv: 1201.2891.

\bibitem[Borovykh(2018)]{borovykh_gaussian_2018}
Anastasia Borovykh.
\newblock A gaussian process perspective on convolutional neural networks.
\newblock \emph{arXiv preprint arXiv:1810.10798}, 2018.

\bibitem[Bradshaw et~al.(2017)Bradshaw, Matthews, and
  Ghahramani]{bradshaw_adversarial_2017}
John Bradshaw, Alexander G de~G Matthews, and Zoubin Ghahramani.
\newblock Adversarial examples, uncertainty, and transfer testing robustness in
  gaussian process hybrid deep networks.
\newblock \emph{arXiv preprint arXiv:1707.02476}, 2017.

\bibitem[Bruna et~al.(2013)Bruna, Zaremba, Szlam, and
  LeCun]{brunaSpectralNetworksLocally2013arXiv.org}
Joan Bruna, Wojciech Zaremba, Arthur Szlam, and Yann LeCun.
\newblock Spectral {{Networks}} and {{Locally Connected Networks}} on
  {{Graphs}}.
\newblock \emph{arXiv:1312.6203 [cs]}, December 2013.

\bibitem[Chen et~al.(2018)Chen, Pennington, and
  Schoenholz]{chen_dynamical_2018}
Minmin Chen, Jeffrey Pennington, and Samuel Schoenholz.
\newblock Dynamical {Isometry} and a {Mean} {Field} {Theory} of {RNNs}:
  {Gating} {Enables} {Signal} {Propagation} in {Recurrent} {Neural} {Networks}.
\newblock In \emph{Proceedings of the 35th {International} {Conference} on
  {Machine} {Learning}}, volume~80 of \emph{Proceedings of {Machine} {Learning}
  {Research}}, pages 873--882, Stockholmsmässan, Stockholm Sweden, July 2018.
  PMLR.
\newblock URL \url{http://proceedings.mlr.press/v80/chen18i.html}.

\bibitem[Cho et~al.(2014)Cho, van Merrienboer, Gulcehre, Bahdanau, Bougares,
  Schwenk, and Bengio]{cho_learning_2014}
Kyunghyun Cho, Bart van Merrienboer, Caglar Gulcehre, Dzmitry Bahdanau, Fethi
  Bougares, Holger Schwenk, and Yoshua Bengio.
\newblock Learning {Phrase} {Representations} using {RNN} {Encoder}-{Decoder}
  for {Statistical} {Machine} {Translation}.
\newblock \emph{arXiv:1406.1078 [cs, stat]}, June 2014.
\newblock URL \url{http://arxiv.org/abs/1406.1078}.
\newblock arXiv: 1406.1078.

\bibitem[Cho and Saul(2009)]{cho_kernel_2009}
Youngmin Cho and Lawrence~K. Saul.
\newblock Kernel methods for deep learning.
\newblock In \emph{Advances in neural information processing systems}, pages
  342--350, 2009.
\newblock URL
  \url{http://papers.nips.cc/paper/3628-kernel-methods-for-deep-learning}.

\bibitem[Damianou and Lawrence(2013)]{damianou_deep_2013}
Andreas Damianou and Neil Lawrence.
\newblock Deep gaussian processes.
\newblock In \emph{Artificial {Intelligence} and {Statistics}}, pages 207--215,
  2013.

\bibitem[Daniely et~al.(2016)Daniely, Frostig, and Singer]{daniely_toward_2016}
Amit Daniely, Roy Frostig, and Yoram Singer.
\newblock Toward {Deeper} {Understanding} of {Neural} {Networks}: {The} {Power}
  of {Initialization} and a {Dual} {View} on {Expressivity}.
\newblock In D.~D. Lee, M.~Sugiyama, U.~V. Luxburg, I.~Guyon, and R.~Garnett,
  editors, \emph{Advances in {Neural} {Information} {Processing} {Systems} 29},
  pages 2253--2261. Curran Associates, Inc., 2016.

\bibitem[Defferrard et~al.(2016)Defferrard, Bresson, and
  Vandergheynst]{defferrardConvolutionalNeuralNetworks2016arXiv.org}
Micha{\"e}l Defferrard, Xavier Bresson, and Pierre Vandergheynst.
\newblock Convolutional {{Neural Networks}} on {{Graphs}} with {{Fast Localized
  Spectral Filtering}}.
\newblock \emph{arXiv:1606.09375 [cs, stat]}, June 2016.

\bibitem[Duvenaud et~al.(2015)Duvenaud, Maclaurin, Iparraguirre, Bombarell,
  Hirzel, {Aspuru-Guzik}, and
  Adams]{duvenaudConvolutionalNetworksGraphs2015NeuralInformationProcessingSystems}
David~K Duvenaud, Dougal Maclaurin, Jorge Iparraguirre, Rafael Bombarell,
  Timothy Hirzel, Alan {Aspuru-Guzik}, and Ryan~P Adams.
\newblock Convolutional {{Networks}} on {{Graphs}} for {{Learning Molecular
  Fingerprints}}.
\newblock In C.~Cortes, N.~D. Lawrence, D.~D. Lee, M.~Sugiyama, and R.~Garnett,
  editors, \emph{Advances in {{Neural Information Processing Systems}} 28},
  pages 2224--2232. {Curran Associates, Inc.}, 2015.

\bibitem[Fukushima(1975)]{fukushima_cognitron:_1975}
Kunihiko Fukushima.
\newblock Cognitron: {A} self-organizing multilayered neural network.
\newblock \emph{Biological cybernetics}, 20\penalty0 (3-4):\penalty0 121--136,
  1975.

\bibitem[Fukushima and Miyake(1982)]{fukushima_neocognitron:_1982}
Kunihiko Fukushima and Sei Miyake.
\newblock Neocognitron: {A} self-organizing neural network model for a
  mechanism of visual pattern recognition.
\newblock In \emph{Competition and cooperation in neural nets}, pages 267--285.
  Springer, 1982.

\bibitem[Garriga-Alonso et~al.(2018)Garriga-Alonso, Aitchison, and
  Rasmussen]{garriga-alonso_deep_2018}
Adrià Garriga-Alonso, Laurence Aitchison, and Carl~Edward Rasmussen.
\newblock Deep {Convolutional} {Networks} as shallow {Gaussian} {Processes}.
\newblock \emph{arXiv:1808.05587 [cs, stat]}, August 2018.
\newblock URL \url{http://arxiv.org/abs/1808.05587}.
\newblock arXiv: 1808.05587.

\bibitem[Genz et~al.(2019)Genz, Bretz, Miwa, Mi, Leisch, Scheipl, and
  Hothorn]{mvtnorm}
Alan Genz, Frank Bretz, Tetsuhisa Miwa, Xuefei Mi, Friedrich Leisch, Fabian
  Scheipl, and Torsten Hothorn.
\newblock \emph{{mvtnorm}: Multivariate Normal and t Distributions}, 2019.
\newblock URL \url{https://CRAN.R-project.org/package=mvtnorm}.
\newblock R package version 1.0-11.

\bibitem[Glorot and Bengio(2010)]{glorot_understanding_2010}
Xavier Glorot and Yoshua Bengio.
\newblock Understanding the difficulty of training deep feedforward neural
  networks.
\newblock In Yee~Whye Teh and Mike Titterington, editors, \emph{Proceedings of
  the {Thirteenth} {International} {Conference} on {Artificial} {Intelligence}
  and {Statistics}}, volume~9 of \emph{Proceedings of {Machine} {Learning}
  {Research}}, pages 249--256, Chia Laguna Resort, Sardinia, Italy, May 2010.
  PMLR.
\newblock URL \url{http://proceedings.mlr.press/v9/glorot10a.html}.
\newblock 02641.

\bibitem[Hanin and Rolnick(2018)]{hanin_how_2018}
Boris Hanin and David Rolnick.
\newblock How to {Start} {Training}: {The} {Effect} of {Initialization} and
  {Architecture}.
\newblock \emph{arXiv:1803.01719 [cs, stat]}, March 2018.
\newblock URL \url{http://arxiv.org/abs/1803.01719}.
\newblock arXiv: 1803.01719.

\bibitem[Hazan and Jaakkola(2015)]{hazan_steps_2015}
Tamir Hazan and Tommi Jaakkola.
\newblock Steps {Toward} {Deep} {Kernel} {Methods} from {Infinite} {Neural}
  {Networks}.
\newblock \emph{arXiv:1508.05133 [cs]}, August 2015.
\newblock URL \url{http://arxiv.org/abs/1508.05133}.
\newblock arXiv: 1508.05133.

\bibitem[He et~al.(2015)He, Zhang, Ren, and Sun]{he_delving_2015}
Kaiming He, Xiangyu Zhang, Shaoqing Ren, and Jian Sun.
\newblock Delving deep into rectifiers: {Surpassing} human-level performance on
  imagenet classification.
\newblock In \emph{Proceedings of the {IEEE} international conference on
  computer vision}, pages 1026--1034, 2015.
\newblock URL
  \url{http://www.cv-foundation.org/openaccess/content_iccv_2015/html/He_Delving_Deep_into_ICCV_2015_paper.html}.

\bibitem[He et~al.(2016)He, Zhang, Ren, and Sun]{he_deep_2016}
Kaiming He, Xiangyu Zhang, Shaoqing Ren, and Jian Sun.
\newblock Deep {Residual} {Learning} for {Image} {Recognition}.
\newblock pages 770--778, 2016.
\newblock URL
  \url{https://www.cv-foundation.org/openaccess/content_cvpr_2016/html/He_Deep_Residual_Learning_CVPR_2016_paper.html}.

\bibitem[Henaff et~al.(2015)Henaff, Bruna, and
  LeCun]{henaffDeepConvolutionalNetworks2015arXiv.org}
Mikael Henaff, Joan Bruna, and Yann LeCun.
\newblock Deep {{Convolutional Networks}} on {{Graph}}-{{Structured Data}}.
\newblock \emph{arXiv:1506.05163 [cs]}, June 2015.

\bibitem[Hochreiter and Schmidhuber(1997)]{hochreiter_long_1997}
Sepp Hochreiter and Jürgen Schmidhuber.
\newblock Long {Short}-{Term} {Memory}.
\newblock \emph{Neural Comput.}, 9\penalty0 (8):\penalty0 1735--1780, November
  1997.
\newblock ISSN 0899-7667.
\newblock \doi{10.1162/neco.1997.9.8.1735}.
\newblock URL \url{http://dx.doi.org/10.1162/neco.1997.9.8.1735}.

\bibitem[Huang et~al.(2016)Huang, Liu, van~der Maaten, and
  Weinberger]{huang_densely_2016}
Gao Huang, Zhuang Liu, Laurens van~der Maaten, and Kilian~Q. Weinberger.
\newblock Densely {Connected} {Convolutional} {Networks}.
\newblock \emph{arXiv:1608.06993 [cs]}, August 2016.
\newblock URL \url{http://arxiv.org/abs/1608.06993}.
\newblock arXiv: 1608.06993.

\bibitem[Ioffe and Szegedy(2015)]{ioffe_batch_2015}
Sergey Ioffe and Christian Szegedy.
\newblock Batch {Normalization}: {Accelerating} {Deep} {Network} {Training} by
  {Reducing} {Internal} {Covariate} {Shift}.
\newblock In \emph{{PMLR}}, pages 448--456, June 2015.
\newblock URL \url{http://proceedings.mlr.press/v37/ioffe15.html}.

\bibitem[Jacot et~al.(2018)Jacot, Gabriel, and Hongler]{jacot_neural_2018}
Arthur Jacot, Franck Gabriel, and Clément Hongler.
\newblock Neural {Tangent} {Kernel}: {Convergence} and {Generalization} in
  {Neural} {Networks}.
\newblock \emph{arXiv:1806.07572 [cs, math, stat]}, June 2018.
\newblock URL \url{http://arxiv.org/abs/1806.07572}.
\newblock 00000 arXiv: 1806.07572.

\bibitem[Jaeger(2007)]{jaegerEchoStateNetwork2007www.scholarpedia.org}
Herbert Jaeger.
\newblock Echo state network.
\newblock \emph{Scholarpedia}, 2\penalty0 (9):\penalty0 2330, September 2007.
\newblock ISSN 1941-6016.
\newblock \doi{10.4249/scholarpedia.2330}.

\bibitem[Kipf and
  Welling(2016)]{kipfSemiSupervisedClassificationGraph2016arXiv.org}
Thomas~N. Kipf and Max Welling.
\newblock Semi-{{Supervised Classification}} with {{Graph Convolutional
  Networks}}.
\newblock \emph{arXiv:1609.02907 [cs, stat]}, September 2016.

\bibitem[Kumar et~al.(2018)Kumar, Singh, Srijith, and
  Damianou]{kumar_deep_2018}
Vinayak Kumar, Vaibhav Singh, PK~Srijith, and Andreas Damianou.
\newblock Deep {Gaussian} {Processes} with {Convolutional} {Kernels}.
\newblock \emph{arXiv preprint arXiv:1806.01655}, 2018.

\bibitem[Lawrence and Moore(2007)]{lawrence_hierarchical_2007}
Neil~D Lawrence and Andrew~J Moore.
\newblock Hierarchical {Gaussian} process latent variable models.
\newblock In \emph{Proceedings of the 24th international conference on
  {Machine} learning}, pages 481--488. ACM, 2007.

\bibitem[Le~Roux and Bengio(2007)]{le_roux_continuous_2007}
Nicolas Le~Roux and Yoshua Bengio.
\newblock Continuous neural networks.
\newblock In \emph{Artificial {Intelligence} and {Statistics}}, pages 404--411,
  2007.

\bibitem[LeCun et~al.(1998)LeCun, Bottou, Bengio, and
  Haffner]{lecun_gradient-based_1998}
Yann LeCun, Léon Bottou, Yoshua Bengio, and Patrick Haffner.
\newblock Gradient-based learning applied to document recognition.
\newblock \emph{Proceedings of the IEEE}, 86\penalty0 (11):\penalty0
  2278--2324, 1998.

\bibitem[LeCun et~al.(1999)LeCun, Haffner, Bottou, and
  Bengio]{lecun_object_1999}
Yann LeCun, Patrick Haffner, Léon Bottou, and Yoshua Bengio.
\newblock Object recognition with gradient-based learning.
\newblock In \emph{Shape, contour and grouping in computer vision}, pages
  319--345. Springer, 1999.

\bibitem[Lee et~al.(2018)Lee, Bahri, Novak, Schoenholz, Pennington, and
  Sohl-dickstein]{lee_deep_2018}
Jaehoon Lee, Yasaman Bahri, Roman Novak, Sam Schoenholz, Jeffrey Pennington,
  and Jascha Sohl-dickstein.
\newblock Deep {Neural} {Networks} as {Gaussian} {Processes}.
\newblock In \emph{International {Conference} on {Learning} {Representations}},
  2018.
\newblock URL \url{https://openreview.net/forum?id=B1EA-M-0Z}.

\bibitem[Li et~al.(2015)Li, Tarlow, Brockschmidt, and
  Zemel]{liGatedGraphSequence2015arXiv.org}
Yujia Li, Daniel Tarlow, Marc Brockschmidt, and Richard Zemel.
\newblock Gated {{Graph Sequence Neural Networks}}.
\newblock \emph{arXiv:1511.05493 [cs, stat]}, November 2015.

\bibitem[Maass et~al.(2002)Maass, Natschl{\"a}ger, and
  Markram]{maassRealTimeComputingStable2002DOI.orgCrossref}
Wolfgang Maass, Thomas Natschl{\"a}ger, and Henry Markram.
\newblock Real-{{Time Computing Without Stable States}}: {{A New Framework}}
  for {{Neural Computation Based}} on {{Perturbations}}.
\newblock \emph{Neural Computation}, 14\penalty0 (11):\penalty0 2531--2560,
  November 2002.
\newblock ISSN 0899-7667, 1530-888X.
\newblock \doi{10.1162/089976602760407955}.

\bibitem[Matthews et~al.(2018)Matthews, Hron, Rowland, Turner, and
  Ghahramani]{matthews_gaussian_2018}
Alexander G. de~G. Matthews, Jiri Hron, Mark Rowland, Richard~E. Turner, and
  Zoubin Ghahramani.
\newblock Gaussian {Process} {Behaviour} in {Wide} {Deep} {Neural} {Networks}.
\newblock In \emph{International {Conference} on {Learning} {Representations}},
  2018.
\newblock URL \url{https://openreview.net/forum?id=H1-nGgWC-}.

\bibitem[Neal(1995)]{neal_bayesian_1995}
Radford~M Neal.
\newblock \emph{{BAYESIAN} {LEARNING} {FOR} {NEURAL} {NETWORKS}}.
\newblock {PhD} {Thesis}, University of Toronto, 1995.

\bibitem[Novak et~al.(2018{\natexlab{a}})Novak, Bahri, Abolafia, Pennington,
  and Sohl-Dickstein]{novak_sensitivity_2018}
Roman Novak, Yasaman Bahri, Daniel~A. Abolafia, Jeffrey Pennington, and Jascha
  Sohl-Dickstein.
\newblock Sensitivity and {Generalization} in {Neural} {Networks}: an
  {Empirical} {Study}.
\newblock In \emph{International {Conference} on {Learning} {Representations}},
  2018{\natexlab{a}}.
\newblock URL \url{https://openreview.net/forum?id=HJC2SzZCW}.

\bibitem[Novak et~al.(2018{\natexlab{b}})Novak, Xiao, Lee, Bahri, Abolafia,
  Pennington, and Sohl-Dickstein]{novak_bayesian_2018}
Roman Novak, Lechao Xiao, Jaehoon Lee, Yasaman Bahri, Daniel~A Abolafia,
  Jeffrey Pennington, and Jascha Sohl-Dickstein.
\newblock Bayesian {Deep} {Convolutional} {Networks} with {Many} {Channels} are
  {Gaussian} {Processes}.
\newblock \emph{arXiv preprint arXiv:1810.05148}, 2018{\natexlab{b}}.

\bibitem[Pennington et~al.(2014)Pennington, Socher, and
  Manning]{pennington_glove:_2014}
Jeffrey Pennington, Richard Socher, and Christopher Manning.
\newblock Glove: {Global} {Vectors} for {Word} {Representation}.
\newblock pages 1532--1543. Association for Computational Linguistics, 2014.
\newblock \doi{10.3115/v1/D14-1162}.
\newblock URL \url{http://aclweb.org/anthology/D14-1162}.
\newblock 04219.

\bibitem[Pennington et~al.(2017)Pennington, Schoenholz, and
  Ganguli]{pennington_resurrecting_2017}
Jeffrey Pennington, Samuel Schoenholz, and Surya Ganguli.
\newblock Resurrecting the sigmoid in deep learning through dynamical isometry:
  theory and practice.
\newblock In I.~Guyon, U.~V. Luxburg, S.~Bengio, H.~Wallach, R.~Fergus,
  S.~Vishwanathan, and R.~Garnett, editors, \emph{Advances in {Neural}
  {Information} {Processing} {Systems} 30}, pages 4788--4798. Curran
  Associates, Inc., 2017.
\newblock 00004.

\bibitem[Poole et~al.(2016)Poole, Lahiri, Raghu, Sohl-Dickstein, and
  Ganguli]{poole_exponential_2016}
Ben Poole, Subhaneil Lahiri, Maithreyi Raghu, Jascha Sohl-Dickstein, and Surya
  Ganguli.
\newblock Exponential expressivity in deep neural networks through transient
  chaos.
\newblock In \emph{Advances {In} {Neural} {Information} {Processing}
  {Systems}}, pages 3360--3368, 2016.
\newblock 00047.

\bibitem[Rumelhart et~al.(1985)Rumelhart, Hinton, and
  Williams]{rumelhart_learning_1985}
David~E Rumelhart, Geoffrey~E Hinton, and Ronald~J Williams.
\newblock Learning internal representations by error propagation.
\newblock Technical report, California Univ San Diego La Jolla Inst for
  Cognitive Science, 1985.

\bibitem[Schlömer(2016--)]{quadpy}
Nico Schlömer.
\newblock {QuadPy}: Numerical integration (quadrature, cubature) in {Python},
  2016--.
\newblock URL \url{https://github.com/nschloe/quadpy}.
\newblock [Online; accessed <today>].

\bibitem[Schoenholz et~al.(2016)Schoenholz, Gilmer, Ganguli, and
  Sohl-Dickstein]{schoenholz_deep_2016}
Samuel~S. Schoenholz, Justin Gilmer, Surya Ganguli, and Jascha Sohl-Dickstein.
\newblock Deep {Information} {Propagation}.
\newblock \emph{arXiv:1611.01232 [cs, stat]}, November 2016.
\newblock URL \url{http://arxiv.org/abs/1611.01232}.
\newblock arXiv: 1611.01232.

\bibitem[Schoenholz et~al.(2017)Schoenholz, Gilmer, Ganguli, and
  Sohl-Dickstein]{schoenholz_deep_2017}
Samuel~S. Schoenholz, Justin Gilmer, Surya Ganguli, and Jascha Sohl-Dickstein.
\newblock Deep {Information} {Propagation}.
\newblock 2017.
\newblock URL \url{https://openreview.net/pdf?id=H1W1UN9gg}.

\bibitem[Schrauwen(2007)]{schrauwenOverviewReservoirComputing2007Zotero}
Benjamin Schrauwen.
\newblock An overview of reservoir computing: Theory, applications and
  implementations.
\newblock page~12, 2007.

\bibitem[Stein and Shakarchi(2011)]{stein_shakarchi_2011}
Elias~M. Stein and Rami Shakarchi.
\newblock \emph{Functional analysis: introduction to further topics in
  analysis}.
\newblock Princeton University Press, 2011.

\bibitem[{Valle-P{\'e}rez} et~al.(2018){Valle-P{\'e}rez}, Camargo, and
  Louis]{valle-perezDeepLearningGeneralizes2018arXiv.org}
Guillermo {Valle-P{\'e}rez}, Chico~Q. Camargo, and Ard~A. Louis.
\newblock Deep learning generalizes because the parameter-function map is
  biased towards simple functions.
\newblock \emph{arXiv:1805.08522 [cs, stat]}, May 2018.

\bibitem[van~der Wilk et~al.(2017)van~der Wilk, Rasmussen, and
  Hensman]{van_der_wilk_convolutional_2017}
Mark van~der Wilk, Carl~Edward Rasmussen, and James Hensman.
\newblock Convolutional {Gaussian} {Processes}.
\newblock In \emph{Advances in {Neural} {Information} {Processing} {Systems}
  30}, pages 2849--2858, 2017.

\bibitem[Vaswani et~al.(2017)Vaswani, Shazeer, Parmar, Uszkoreit, Jones, Gomez,
  Kaiser, and Polosukhin]{vaswani_attention_2017}
Ashish Vaswani, Noam Shazeer, Niki Parmar, Jakob Uszkoreit, Llion Jones,
  Aidan~N Gomez, {\textbackslash}Lukasz Kaiser, and Illia Polosukhin.
\newblock Attention is {All} {You} {Need}.
\newblock In \emph{Advances in {Neural} {Information} {Processing} {Systems}},
  pages 5998--6008, 2017.

\bibitem[Williams(1997)]{williams_computing_1997}
Christopher K~I Williams.
\newblock Computing with {Infinite} {Networks}.
\newblock In \emph{Advances in neural information processing systems}, page~7,
  1997.

\bibitem[Wilson et~al.(2016{\natexlab{a}})Wilson, Hu, Salakhutdinov, and
  Xing]{wilson_stochastic_2016}
Andrew~G Wilson, Zhiting Hu, Ruslan~R Salakhutdinov, and Eric~P Xing.
\newblock Stochastic {Variational} {Deep} {Kernel} {Learning}.
\newblock In \emph{Advances in {Neural} {Information} {Processing} {Systems}},
  pages 2586--2594, 2016{\natexlab{a}}.

\bibitem[Wilson et~al.(2016{\natexlab{b}})Wilson, Hu, Salakhutdinov, and
  Xing]{wilson_deep_2016}
Andrew~Gordon Wilson, Zhiting Hu, Ruslan Salakhutdinov, and Eric~P Xing.
\newblock Deep kernel learning.
\newblock In \emph{Artificial {Intelligence} and {Statistics}}, pages 370--378,
  2016{\natexlab{b}}.

\bibitem[Xiao et~al.(2018)Xiao, Bahri, Sohl-Dickstein, Schoenholz, and
  Pennington]{xiao_dynamical_2018}
Lechao Xiao, Yasaman Bahri, Jascha Sohl-Dickstein, Samuel Schoenholz, and
  Jeffrey Pennington.
\newblock Dynamical {Isometry} and a {Mean} {Field} {Theory} of {CNNs}: {How}
  to {Train} 10,000-{Layer} {Vanilla} {Convolutional} {Neural} {Networks}.
\newblock In \emph{Proceedings of the 35th {International} {Conference} on
  {Machine} {Learning}}, volume~80 of \emph{Proceedings of {Machine} {Learning}
  {Research}}, pages 5393--5402, Stockholmsmässan, Stockholm Sweden, July
  2018. PMLR.
\newblock URL \url{http://proceedings.mlr.press/v80/xiao18a.html}.

\bibitem[Yang(2019)]{yangScalingLimitsWide2019arXiv.org}
Greg Yang.
\newblock Scaling {{Limits}} of {{Wide Neural Networks}} with {{Weight
  Sharing}}: {{Gaussian Process Behavior}}, {{Gradient Independence}}, and
  {{Neural Tangent Kernel Derivation}}.
\newblock \emph{arXiv:1902.04760 [cond-mat, physics:math-ph, stat]}, February
  2019.

\bibitem[Yang and
  Salman(2019)]{yangFineGrainedSpectralPerspective2019arxiv.org}
Greg Yang and Hadi Salman.
\newblock A {{Fine}}-{{Grained Spectral Perspective}} on {{Neural Networks}}.
\newblock July 2019.

\bibitem[Yang and Schoenholz(2018)]{yang_deep_2018}
Greg Yang and Sam~S. Schoenholz.
\newblock Deep {Mean} {Field} {Theory}: {Layerwise} {Variance} and {Width}
  {Variation} as {Methods} to {Control} {Gradient} {Explosion}.
\newblock February 2018.
\newblock URL \url{https://openreview.net/forum?id=rJGY8GbR-}.

\bibitem[Yang and Schoenholz(2017)]{yang_mean_2017}
Greg Yang and Samuel~S. Schoenholz.
\newblock Mean {Field} {Residual} {Network}: {On} the {Edge} of {Chaos}.
\newblock In \emph{Advances in neural information processing systems}, 2017.

\bibitem[Yang et~al.(2018)Yang, Pennington, Rao, Sohl-Dickstein, and
  Schoenholz]{yang_mean_2018}
Greg Yang, Jeffrey Pennington, Vinay Rao, Jascha Sohl-Dickstein, and Samuel~S.
  Schoenholz.
\newblock A {Mean} {Field} {Theory} of {Batch} {Normalization}.
\newblock September 2018.
\newblock URL \url{https://openreview.net/forum?id=SyMDXnCcF7}.

\bibitem[Yang et~al.(2019)Yang, Pennington, Rao, Sohl-Dickstein, and
  Schoenholz]{yang_mean_2019}
Greg Yang, Jeffrey Pennington, Vinay Rao, Jascha Sohl-Dickstein, and Samuel~S.
  Schoenholz.
\newblock A {Mean} {Field} {Theory} of {Batch} {Normalization}.
\newblock \emph{arXiv:1902.08129 [cond-mat]}, February 2019.
\newblock URL \url{http://arxiv.org/abs/1902.08129}.
\newblock arXiv: 1902.08129.

\end{thebibliography}
\bibliographystyle{plainnat}
\newpage

\appendix

\section{Writing Standard Architectures in \texorpdfstring{\netsor{}}{Netsor}}
\label{sec:MoreExamples}

In this section, we showcase example programs for batchnorm, skip connection, convolution, pooling, GRU/LSTM, and (scaled) attention.
In most cases, we demonstrate the computation on a single input batch, image, or sequence.
Generalization to multiple inputs is obvious and follows the pattern of \cref{tp:RNN}.
It should also be apparent that any composition of these gadgets can be expressed in \netsor{}.

Additionally, observe that all nonlinearities used in these programs are controlled (\cref{defn:controlled}) or parameter-controlled (\cref{defn:parameterControlled}), so that \cref{thm:netsorMasterTheorem} or \cref{thm:Netsor+MasterTheorem} applies.
Also we remark that the GP convergence results for batchnorm, GRU/LSTM (and RNNs in general), and attention are new.

\paragraph{Batchnorm, Fully-Connected}
\label{subsec:exampleBatchnorm}
Let $\tilde\phi: \R^B \to \R^B,$
\begin{align*}
\tilde\phi(h) &\defeq \phi\lp\f{h - \nu(h)}{\sigma(h)}\rp,
    &
    \text{where}\quad
\nu(h) \defeq \f 1 B \sum_{i=1}^B h_i,
    \quad
\sigma(h) \defeq \sqrt{\f 1 B \sum_{i=1}^B (h_i - \nu(h))^2},
\numberthis\label{eqn:BN}
\end{align*}
be batchnorm followed by coordinatewise nonlinearity $\phi$, where $h \in \R^B$ should be interpreted as a single neuron across a batch, and $\nu$ and $\sigma$ are the \emph{batch} mean and standard deviations.
Here, $B$ should be thought of as fixed while $n \to \infty$.
Then, given a batch of G-vars $y^1, \ldots, y^B: \Gtype(n)$ (for example, they could be the preactivations after applying a linear layer), we can express the application of batchnorm via \ref{linetype:nonlin} as 
\begin{align*}
    x^1 &:= \tilde\phi_1(y^1, \ldots, y^B),\quad
    \ldots,\quad
    x^B := \tilde \phi_B(y^1, \ldots, y^B),
    \numberthis\label{eqn:BNnetsor}
\end{align*}
producing $B$ H-vars $x^1, \ldots, x^B: \Htype(n)$.
\cref{tp:batchnorm} expresses more generally the computation of a batchnorm network on multiple batches and over multiple layers.
Here, each ``for-loop'' is just shorthand for the corresponding unrolled program (since \cref{defn:netsor} doesn't allow for-loops).

\begin{algorithm}[tb]
    \caption{Batchnorm (with Fully Connected Layers) over Multiple Batches}
    \label{tp:batchnorm}
    \begin{multicols}{2}
    \begin{algorithmic}
      \State {\it // For $\phi: \R \to \R, h \in \R^B$,}
      \State {\it // let $\tilde\phi(h) = \phi((h - \nu(h)) / \sigma(h))$,}
      \State {\it // where $\nu(h) = \f 1 B \sum_{i=1}^B h_i$,}
      \State {\it // $\sigma(h) = \sqrt{\f 1 B \sum_{i=1}^B (h_i - \nu(h))^2}$.}
      \State {\it // Let $\tilde\phi_i, i \in [B],$ be the $i$th coordinate of $\tilde \phi$.}
      \State {\it // Embeddings of $k$ batches of inputs}
      \State {\it // with $a$th batch having size $B_a$}
      \Require $\{W^1 x^{ia}: \Gtype(n)\}_{a \in [k], i \in [B_a]}$
      \Require $W^2, \ldots, W^L: \Atype(n, n)$
      \State {\it // Readout layer}
      \Require $v: \Gtype(n)$
      \State
          \For{$a \in [k]$ and $i \in [B_a]$}
          \State $x^{1ia} := \tilde \phi_i(W^1 x^{1a},\ldots, W^1 x^{B_a a})$
          \State \quad\quad\quad\quad$: \Htype(n)$
          \EndFor
          \For{$l = 2, \ldots, L$}
            \For{$a \in [k]$ and $i \in [B_a]$}
              \State $h^{lia} := W^l x^{l-1,ia}: \Gtype(n)$
            \EndFor
            \For{$a \in [k]$ and $i \in [B_a]$}
              \State $x^{lia} := \tilde \phi_i(h^{l1a},\ldots, h^{lB_a a}): \Htype(n)$
            \EndFor
          \EndFor
      \Ensure $\{v^\trsp x^{Lia}/\sqrt n\}_{a \in [k], i \in [B_a]}$
    \end{algorithmic}
    \end{multicols}
\end{algorithm}

\begin{algorithm}[tb]
    \caption{2-Layer Convolutional Network with Global Average Pooling}
    \label{tp:conv}
    \begin{multicols}{2}
    \begin{algorithmic}
      \Require $\{W^1_j x_{i+j}: \Gtype(n^1)\}_{\substack{j \in ker^1, i \in pos^1\\ s.t.\ i+j \in pos^0}}$
      \Require $\{W^2_j: \Atype(n^2, n^1)\}_{j \in ker^2}$
      \State {\it // Readout weights}
      \Require $v: \Gtype(n^2)$
      \State {\it // Layer 1 Convolution}
      \For{$i \in pos^1$}
      \State {\it // Directly use input embeddings}
      \State {\it // \ref{linetype:lincomb}}
            \State {\it // Sum is over all $j \in ker^1$ such that}
            \State {\it // there is $i' \in pos^{0}$ with $i' = j + i$}
      \State $h^1_i := \sum_{j} W^1_j x_{i+j}: \Gtype(n^1)$
      \State $x^1_i := \phi(h^1_i): \Htype(n^1)$
      \EndFor
      \State {\it // Layer 2 Convolution}
      \For{$j \in ker^2, i \in pos^1\ s.t.\ i+j \in pos^2$}
      \State {\it // Convolution Weights Multiplication}
      \State {\it // \ref{linetype:MatMul}}
      \State $h^2_{i;j} := W^1_j x^1_{i+j}: \Gtype(n^2)$
      \EndFor
      \For{$i \in pos^2$}
            \State {\it // Sum is over all $j \in ker^2$ such that}
            \State {\it // there is $i' \in pos^1$ with $i' = j + i$}
      \State $h^2_i := \sum_{j} h^2_{i;j}: \Gtype(n^2)$
      \EndFor
      \State {\it // Nonlinearity \& Global Average Pooling}
      \State $\bar x^2 := \f 1 {|pos^2|} \sum_{i \in pos^2} \phi(h^2_i): \Htype(n^2)$
      \Ensure $v^\trsp \bar x^2 / \sqrt {n^2}$
    \end{algorithmic}
    \end{multicols}
\end{algorithm}

\paragraph{Skip Connection}

If $x: \Htype(n)$ is previous layer activation, and we have weights $W: \Atype(n, n)$ and bias $b: \Gtype(n)$, then we can express skip connection as
\begin{align*}
    \tilde h &:= W x    &   \text{\ref{linetype:MatMul}}\\
    h &:= \tilde h + b  &   \text{\ref{linetype:lincomb}}\\
    \bar x &:= x + \phi(h)&   \text{\ref{linetype:nonlin}}
\end{align*}

\paragraph{(Graph) Convolution}
Consider a convolution network with $L$ layers, with layer $l$ having $n^l$ channels, with kernel positions $ker^l$ (for example, for a $3\times 5$ kernel, $ker^l = [3] \times [5]$), and with feature map pixel positions $pos^l$ (for example, for a $128\times 256$ feature map, $pos^l = [128] \times [256]$).
Then we can describe the weights of a stride-1 convolution as $\{W^l_{i\alpha\beta}\}_{l \in [L], i \in ker^l, \alpha \in [n^l], \beta \in [n^{l-1}]}$, so that for each $i \in ker^l$, $W_i^l \in \R^{n^l \times n^{l-1}}$ is a dense matrix.
Further, suppose $x$ is an image input to the network with pixel coordinates $pos^0$ and $n^0$ channels, $\{x_{i \alpha}\}_{i \in pos^0, \alpha \in [n^0]}$, so that $x_i$ is a vector of dimension $n^0$.
Then the application of the convolution $W^1$ on $x$ is given by
\begin{align*}
    h^1_i :=
    (W^1*x)_i
        =
            \sum_{j} W^1_j x_{i+j}
            \in \R^{n^1}
            ,\quad
    h^1_{i\alpha} =
    (W^1*x)_{i\alpha}
        =
            \sum_{j\beta} W^1_{j \alpha \beta} x_{i+j, \beta}
            \in \R,
\end{align*}
where the sums are over $j\in ker^1, i+j \in pos^0$, and $\beta \in [n^0]$.
For higher layers we have similar formulas, with the only difference that we treat $W^1_j x_{i+j}$ as input G-vars but treat $W^{l+1}_j x^l_{i+j}, l \ge 2,$ as a \ref{linetype:MatMul} operation.
Thus, our framework can express CNN computations by expressing the individual matrix multiplications $W_j x_{i+j}$ and reusing the matrices $W_j$.
\cref{tp:conv} shows a simple example program, and \cref{tp:conv2} shows a full-fledged example over multiple inputs.
Again, each ``for-loop'' is just shorthand for the corresponding unrolled program.
Higher stride and general graph convolution can be expressed likewise.

\paragraph{Pooling}
We continue the notational scheme of the exposition on convolutions above.
Given the feature maps of layer $l$, $\{x_{i\alpha}\}_{i \in pos^l, \alpha \in [n^l]}$, global average pooling produces a single vector $\bar x = \{\bar x_\alpha\}_{\alpha \in [n^l]}$ given by 
\begin{align*}
    \bar x := \f 1 {|pos^l|} \sum_{i \in pos^l} x_i, \qquad \text{using \ref{linetype:lincomb}.}
\end{align*}
See \cref{tp:conv} for an example in combination with convolutional layers.
We can similarly express max pooling.
Suppose $pos^l = [2k] \times [2k]$.
Then max pooling with, for example, $2\times 2$ kernel and stride of 2 would produce $\{\hat x_{j \alpha}\}_{j\in [k] \times [k], \alpha \in [n^l]}$, with
\begin{align*}
    \hat x_{j\alpha} := \max(\{x_{2j+i,\alpha}\}_{i \in \{0,1\}\times\{0,1\}, 2j+i \in pos^l}), \qquad \text{using \ref{linetype:nonlin}.}
\end{align*}

\paragraph{GRU and LSTM}

We present a program expressing GRU computation in \cref{tp:GRU}; the program for LSTM is similar and is omitted.
The overall pattern is similar to the program for simple RNN (\cref{tp:RNN}), but with a crucial subtlety regarding the gating mechanism.
In \cref{tp:GRU}, $\tilde z^s, \tilde r^s, \tilde h^s$ are G-vars, but the gates $\sigma(\tilde z^s), \sigma(\tilde r^s)$ (where $\sigma$ is sigmoid) and the candidate update $\phi(\tilde h^s)$ are not G-vars.
As we can only apply \ref{linetype:nonlin} to G-vars, this requires us to unroll the definition of $h^s$ to be a function of only G-vars.
However, \cref{sec:VersionTensorPrograms} presents expanded, but semantically equivalent versions of \netsor which allow a more succinct representation of the same computation; see \cref{tp:GRUNetsoro}.
Finally, \cref{tp:GRU2} presents a full-fledged program with multiple input sequences.

\newcommand{\zz}{\mathsf{z}}
\newcommand{\rr}{\mathsf{r}}
\newcommand{\hh}{\mathsf{h}}
\begin{algorithm}[tb]
    \caption{GRU, with Gating Function $\sigma$ and Activation Function $\phi$}
    \label{tp:GRU}
    \begin{algorithmic}
      \State {\it // Embeddings of input sequence}
      \Require $U_\zz x^1, \ldots, U_\zz x^T: \Gtype(n)$
      \Require $U_\rr x^1, \ldots, U_\rr x^T: \Gtype(n)$
      \Require $U_\hh x^1, \ldots, U_\hh x^T: \Gtype(n)$
      \State {\it // Parameters}
      \Require $W_\zz, W_\rr, W_\hh: \Atype(n, n)$
      \Require $b_\zz, b_\rr, b_\hh: \Gtype(n)$
      \State {\it // Initial GRU state}
      \Require $h^0: \Gtype(n)$
      \State {\it // Readout layer}
      \Require $v: \Gtype(n)$
      \State {\it // Time step 1}
      \State $h_\zz^1 := W_\zz h^0: \Gtype(n)$
      \State $\tilde z^1 := h_\zz^1 + U_\zz x^1 + b_\zz: \Gtype(n)$
      \State $h_\rr^{1} := W_\rr h^0: \Gtype(n)$
      \State $\tilde r^1 := h_\rr^{1} + U_\rr x^1 + b_\rr: \Gtype(n)$
      \State {\it // $\sigma$ is gating function, typically sigmoid; applying \ref{linetype:nonlin}}
      \State $\hat h^0 := h^0 \odot \sigma(\tilde r^1): \Htype(n)$
      \State $h_\hh^{1} := W_\hh \hat h^0: \Gtype(n)$
      \State $\tilde h^1 := h_\hh^{1} + U_\hh x^1 + b_\hh: \Gtype(n)$
      \State {\it // Apply \ref{linetype:nonlin}}
      \State {\it // $\phi$ is activation function, typically $\tanh$}
      \State $h^1 := (1 - \sigma(\tilde z^1)) \odot h^0 + \sigma(\tilde z^1) \odot \phi(\tilde h^1): \Htype(n)$
      \State {\it // Time step 2}
      \State $h_\zz^{2} := W_\zz h^1: \Gtype(n)$
      \State $\tilde z^2 := h_\zz^{2} + U_\zz x^2 + b_\zz: \Gtype(n)$
      \State $h_\rr^{2} := W_\rr h^1: \Gtype(n)$
      \State $\tilde r^2 := h_\rr^{2} + U_\rr x^2 + b_\rr: \Gtype(n)$
      \State {\it // Morally, $\hat h^1 = \sigma(\tilde r^1) \odot h^1$, but we need to unroll $h^1$ to apply \ref{linetype:nonlin}}
      \State {\it // This can be expressed with more brevity using \netsoro; see \cref{remk:netsoro}}
      \State $\hat h^1 := \sigma(\tilde r^1) \odot ((1 - \sigma(\tilde z^1)) \odot h^0 + \sigma(\tilde z^1) \odot \phi(\tilde h^1)): \Htype(n)$
      \State $h_\hh^{2} := W_\hh \hat h^1: \Gtype(n)$
      \State $\tilde h^2 := h_\hh^{2} + U_\hh x^2 + b_\hh: \Gtype(n)$
      \State {\it // Unrolling $h^2$ to a coordinatewise function of G-vars}
      \State  $h^2 := (1 - \sigma(\tilde z^2)) \odot (1 - \sigma(\tilde z^1)) \odot h^0 + (1 - \sigma(\tilde z^2)) \odot \sigma(\tilde z^1) \odot \phi(\tilde h^1) + \sigma(\tilde z^2) \odot \phi(\tilde h^2): \Htype(n)$
      \State {\it // Time step 3}
      \State $\vdots$
      \State {\it // Time step $T$}
      \State {\it // Define $\tilde z^T, \tilde r^T, \tilde h^T$ just like above}
      \State $\vdots$
      \State
      {\it // Unrolling $h^T$ to a coordinatewise function of G-vars}
      \State  $h^T := h^0 \odot \bigodot_{i=1}^T (1 - \sigma(\tilde z^i))
      + \sum_{j=1}^T \phi(\tilde h^j) \odot \sigma(\tilde z^j) \odot \bigodot_{l=j+1}^T (1- \sigma(\tilde z^l)): \Htype(n)$
      
      \Ensure $(v^\trsp h^1/\sqrt{n}, \ldots, v^\trsp h^T/\sqrt{n})$
    \end{algorithmic}
\end{algorithm}

\paragraph{Layernorm}
Layernorm requires the extended rule \ref{linetype:nonlin+} to express.
Recall for $x \in \R^n$ (thought of as the vector of activations with one entry per neuron in a layer; contrast with batchnorm), $\Layernorm(x) = \f{x - \nu(x)}{\sigma(x)}$ where $\nu(x) = \f 1 n \sum_\alpha x_\alpha$ and $\sigma(x) = \sqrt{\f 1 n \sum_{\alpha=1}^n(x_\alpha - \nu(x))^2}$.
As we will see, $\nu(x)$ and $\sigma(x)$ will both converge a.s.\ to a deterministic limit.
Thus $\Layernorm(x)$ is just a linear combination of $x$ with the constant-1s vector (considered as a input G-var), with (roughly deterministic) coefficients $\mu(x)$ and $\sigma(x)$.
This is expressible using the \ref{linetype:nonlin+} rule:
\begin{align*}
\Layernorm(x) := \psi(x; \nu(x), \sigma(x)),\quad
\text{where}\quad
\psi(x; \theta_1, \theta_2) \defeq \f{x - \theta_1}{\theta_2}.
\end{align*}
Similarly, if layernorm is followed up by a nonlinearity $\phi$, then we can express
\begin{align*}
\phi(\Layernorm(x)) := \psi(x; \nu(x), \sigma(x)),\quad
\text{where}\quad
\psi(x; \theta_1, \theta_2) \defeq \phi\lp\f{x - \theta_1}{\theta_2}\rp.
\end{align*}
If layernorm is preceded by a nonlinearity, then likewise we can write
\begin{align*}
\Layernorm(\phi(x)) := \psi(x; \nu(\phi(x)), \sigma(\phi(x))),\quad
\text{where}\quad
\psi(x; \theta_1, \theta_2) \defeq \f{\phi(x) - \theta_1}{\theta_2}.
\end{align*}

\paragraph{Scaled Attention}
Scaled attention requires the extended rule \ref{linetype:nonlin+} to express.
Consider the following version of scaled attention:
Given a query vector $q \in \R^{n}$, keys $k^i \in \R^{n}$ for each $i \in [r]$, and corresponding values $v^i \in \R^{m}, i \in [r]$, we can define the following scaled attention
\[\Attention(q, \{k^i\}_i, \{v^i\}_i)
\defeq \sum_{i=1}^r a_i v^i, \qquad
a_i \defeq \SoftMax(q^\trsp k^1/n, \ldots, q^\trsp k^r/n)_i.\]

If $q, k^i$ are given as H-vars in a \netsor program, then \cref{thm:netsorMasterTheorem} will show that $q^\trsp k^i / n$ converges almost surely to a deterministic limit, so that each $a_i$ converges likewise.
If each $v^i = \psi(g^i)$ for some G-var $g^i$ and fixed nonlinearity $\psi$, then attention can be expressed as follows using \ref{linetype:nonlin+}:
\begin{align*}
\Attention(q, \{k^i\}_i, \{v^i\}_i) = \phi(g^1, \ldots, g^r; a_1, \ldots, a_r)
\end{align*}
where
\begin{align*}
\phi(x^1, \ldots, x^r; \theta_1, \ldots, \theta_r) \defeq \theta_1 \psi(x^1) + \cdots + \theta_r \psi(x^r).
\end{align*}

More complicated variants, such as allowing $\psi$ to take multiple G-vars as inputs, is likewise expressible.
When used as part of a decoder, a mask needs to be placed on the pre-softmax values so that no attention is paid to tokens \emph{in the future}.
This is aptly called \emph{masked attention} and is given by the following formula:
for $j \in [r]$,
\begin{align*}
&\MaskedAttention_j(q, \{k^i\}_{i=1}^r, \{v^i\}_{i=1}^r)
= \sum_{i=1}^r a_i^j v^i,\\
&\quad \text{where}\quad
a_i^j = \SoftMax(q^\trsp k^1/n, \ldots, q^\trsp k^j/n, -\infty, \ldots, -\infty)_i
.
\end{align*}
It can obviously be expressed in \netsor in the same fashion as (non-masked) attention.

Note that \citet{vaswani_attention_2017} scales the pre-softmax values by $1/\sqrt {n}$ instead of $1/n$:
\[a_i = \SoftMax(q^\trsp k^1/\sqrt{n}, \ldots, q^\trsp k^r/\sqrt{n})_i.\]
This is useful if $q_\alpha k_\alpha^i$ has mean zero (averaged over $\alpha$) for each $i$, so that $q^\trsp k^i/\sqrt{n}$ becomes roughly Gaussian, whereas $q^\trsp k^i/n$ converges to 0.
However, when the zero-mean condition doesn't hold, $q^\trsp k^i/\sqrt{n}$ would only blow up to infinity.

\section{Example GP Kernel Computation with \texorpdfstring{\netsor}{Netsor}}
\label{sec:GPKernelComputationExamples}

In this section, we show how to compute the GP kernel of different architecture, following the recursive construction of \cref{thm:netsorMasterTheorem}.

First, we review the \emph{V-transform} of a nonlinearity.
\begin{defn}\label{defn:Vtransform}
Given a multivariate nonlinearity $\Phi: \R^B \to \R^B$, its \emph{V-transform} $\Vt\Phi$ is a function taking $B \times B$ positive semidefinite matrices to $B \times B$ positive semidefinite matrices, and is given by the following formula
\begin{equation*}
\Vt\Phi(K) \defeq \EV_{z \sim \Gaus(0, K)} \Phi(z) \Phi(z)^\trsp.
\end{equation*}
When $\phi: \R \to \R$, we take $\Vt\phi$ to be V-transform of the $\R^B \to \R^B$ function that applies $\phi$ to each coordinate.
\end{defn}

We collect below some of the common V-transforms.
Here we describe the V-transforms using the function notation of kernels, but we shall freely switch between the function notation and the matrix notation in what follows.

\begin{fact}[\cite{cho_kernel_2009}]
\label{fact:Vrelu}
For any kernel $K$,
\begin{align*}
\Vt\relu(K)(x, x')
    &=
        \f 1 {2\pi} (\sqrt{1-c^2} + (\pi - \arccos c) c) \sqrt{K(x, x) K(x', x')}
        \\
\Vt{\relu'}(K)(x, x')
    &=
        \f 1 {2\pi} (\pi - \arccos c)
        \\
\end{align*}
where $c = K(x, x')/\sqrt{K(x, x)K(x', x')}$.
\end{fact}

\begin{fact}[\cite{neal_bayesian_1995}]
\label{fact:Verf}
For any kernel $K$,
\begin{align*}
\Vt\erf(K)(x, x')
    &=
        \f 2 {\pi} \arcsin
        \f{K(x, x')}{\sqrt{(K(x, x)+ 0.5)(K(x', x')+0.5)}}
        \\
\Vt{\erf'}(K)(x, x')
    &=
        \f 4 { 
        \pi
        \sqrt{(1 + 2 K(x, x))(1 + 2 K(x', x')) - 4 K(x, x')^2}}
        .
\end{align*}
\end{fact}

\begin{fact}
\label{fact:Vexp}
Let $\phi(x) = \exp(x/\sigma)$ for some $\sigma > 0$.
For any kernel $K$,
\begin{align*}
\Vt\phi(K)(x, x')
    &=
        \exp\lp \f{K(x, x) + 2K(x, x') + K(x', x')}{2\sigma^2}\rp.
\end{align*}
\end{fact}

In \cref{sec:batchnormKernel}, we will also discuss the V-transform of batchnorm (which has been derived in \cite{yang_mean_2019}).

\subsection{MLP}

The kernel computation of a multilayer perceptron is by now well-known \cite{poole_exponential_2016,schoenholz_deep_2017,lee_deep_2018}.
In this section, we work out an example of how to recover the usual kernel computation via tensor programs.

\subsubsection{MLP with Single Input \texorpdfstring{(\cref{tp:MLP})}{}}
\label{sec:MLPsingleinput}
The aim here is to illustrate step-by-step applications of \cref{eqn:extendedMuSigma} for \cref{tp:MLP}, but note that in practice, it is often more convenient to find some bulk recursive or compositional structure of $\tSigma$, and to leverage that structure for computing $\tSigma$ (see \cref{sec:ExactKernelRNN} for an example).

For simplicity, assume the nonlinearities $\phi$ are ReLUs and the widths $n^1 = n^2 = n$ to satisfy \cref{assm:equalDimNoNonlin+}.
By \cref{cor:GPConv}, we know that the output of the MLP, $v^Tx^2/\sqrt{n}$, is distributed as a Gaussian with mean $0$ and variance of $\sigma^2  \EV_z\phi^2(z)$, where $z \sim \Gaus(\tmu(h^2), \tSigma(h^2, h^2))$, and $h^2$ is the layer 2 preactivation (also a G-var) in \cref{tp:MLP}. 
Therefore, all we need is to compute $\tmu(h^2)$ and  $\tSigma(h^2, h^2)$ using \eqref{eqn:extendedMuSigma}, which requires the calculation of $\tmu$ and $\tSigma$ for possibly all the other G-vars in \cref{tp:MLP} due to the recursive nature of \eqref{eqn:extendedMuSigma}.
In this example, we shall compute all entries of $\tmu$ and $\tSigma$ explicitly as a demonstration.
We do so for each G-var in the order of appearance in the \netsor program.
Explicitly, the ordering is $\Gvars= \left( W^1x, b^1, b^2, v, h^1, \tilde h^2, h^2 \right)$.
The input G-vars are $W^1 x, b^1, b^2, $ and $v$, and the sole input A-var is $W^2$.

\paragraph{Setup}
In the fashion of typical Glorot initializations, we shall sample the parameters $W^1, W^2, b^1, b^2$ of the network as
\begin{equation*}
    W^1_{\alpha \beta} \sim \Gaus(0, 1 / \dim x),
    W^2_{\alpha \beta} \sim \Gaus(0, 1 / n),
    v_\alpha \sim \Gaus(0, 1),
    b^1_\alpha \sim \Gaus(0, 1),
    b^2_\alpha \sim \Gaus(0, 1)
    .
\end{equation*}
This corresponds, in the context of \cref{assm:equalDimNoNonlin+}, to setting $\sigma_{W^2} = 1$, $\muin = 0$ (in particular, $\tmu(W^1x) = \muin(W^1x) = 0$ due to the sampling of $W^1$), and $\Sigmain$ as follows:

\begin{align*}
&\Sigmain(W^1x, W^1x) = 1 \\
&\Sigmain(b^1, b^1) = 1 \\
&\Sigmain(b^2, b^2) = 1 \\
&\Sigmain(b^2, b^1) = 0 \\
&\Sigmain(v, v) = 1 \\
&\Sigmain(b^i, W^1x) = 0~~ \text{for}~~ i \in \{1,2\}\\
&\Sigmain(b^i, v) = 0 ~~ \text{for}~~ i \in \{1,2\}\\
&\Sigmain(v, W^1x) = 0
\end{align*}

\paragraph{Calculating $\tmu$ and $\tSigma$}
Now we will show a detailed calculation of $\tmu$ and $\tSigma$ for all of the G-vars $\Gvars$ appearing here.
By the input G-var case of \cref{eqn:extendedMuSigma},
\begin{align*}
&\tSigma(W^1x, W^1x) = \Sigmain(W^1x, W^1x) = 1 \\
&\tSigma(b^1, b^1) =  \Sigmain(b^1, b^1) = 1 \\
&\tSigma(b^2, b^2) =  \Sigmain(b^2, b^2) = 1 \\
&\tSigma(b^1, b^2) =  \Sigmain(b^2, b^1) = 0 \\
&\tSigma(v, v) =  \Sigmain(v, v) = 1 \\
&\tSigma(W^1x, b^i) = \tSigma(b^i, W^1x) = \Sigmain(b^i, W^1x) = 0~~ \text{for}~~ i \in \{1,2\}\\
&\tSigma(v, b^i) = \tSigma(b^i, v) = \Sigmain(b^i, v) = 0 ~~ \text{for}~~ i \in \{1,2\}\\
&\tSigma(W^1x, v) = \tSigma(v, W^1x) = \Sigmain(v, W^1x) = 0
\end{align*}

Next, we extend $\tmu$ and $\tSigma$ to $h^1$, introduced via \ref{linetype:lincomb} by $h^1 := W^1x + b^1$.
Note that $h^1$ is a G-var of the form $g = \sum_i a_i y^i$ where $a_1=a_2= 1$, $y^1 =W^1x $, and $y^2 =b^1 $.
Therefore, by the \ref{linetype:lincomb} case of \cref{eqn:extendedMuSigma},
\begin{align*}
&\tmu(h^1) = \tmu(W^1x) + \tmu(b^1) = 0\\
&\tSigma(h^1,W^1x) = \tSigma(W^1x, W^1x) + \tSigma(b^1, W^1x) = 1 + 0 = 1\\
&\tSigma(h^1,b^1) = \tSigma(W^1x, b^1) + \tSigma(b^1, b^1) = 0 + 1 = 1\\
&\tSigma(h^1,b^2) = \tSigma(W^1x, b^2) + \tSigma(b^1, b^2) = 0 + 0 = 0\\
&\tSigma(h^1,v) = \tSigma(W^1x, v) + \tSigma(b^1, v) = 0 + 0 = 0\\
&\tSigma(h^1,h^1) = \tSigma(h^1, W^1x) + \tSigma(h^1,b^1) = 1 + 1 = 2
.
\end{align*}
So $h^1$ is correlated with $W^1 x$ and $b^1$ in obvious ways, and is independent from $b^2$ and $v$.

Next, we extend $\tmu$ and $\tSigma$ to, introduced via \ref{linetype:MatMul} by $\tilde h^2 := W^2 x^1$.
Note that $\tilde h^2$ is a G-var of the form $g = Wh$ where $W = W^2$ is an A-var, and $h=x^1$ is an H-var introduced by $x^1 := \relu(h^1)$.
Therefore, by the ``otherwise'' case of \cref{eqn:extendedMuSigma},
\begin{align*}
&\tmu(\tilde h^2) =  0\\
&\tSigma(\tilde h^2,W^1x) = 0 \\
&\tSigma(\tilde h^2,b^1) = 0\\
&\tSigma(\tilde h^2,b^2) = 0\\
&\tSigma(\tilde h^2,v) = 0\\
&\tSigma(\tilde h^2,h^1) = 0
\end{align*}
and by the \ref{linetype:MatMul} case of \cref{eqn:extendedMuSigma} (setting $\phi$ and $\bar\phi$ there to both be ReLU),
\begin{align*}
&\tSigma(\tilde h^2,\tilde h^2) = \sigma^2_{W^2}\EV_z \phi(z)\bar\phi(z) = \EV_z \relu(z)^2 = 1\\\nonumber
&~~~~~\text{ where }
z\sim\Gaus(\tmu(h^1),\tSigma(h^1, h^1))  = \Gaus(0, 2)
.
\end{align*}
Thus, $\tilde h^2$ can be thought of as ``independent'' from all other G-vars.

Finally, we extend $\tmu$ and $\tSigma$ to $h^2$, introduced via \ref{linetype:lincomb} by $h^2 := \tilde h^2 + b^2$.
Note that $h^2$ is a G-var of the form $g = \sum_i a_i y^i$ where $a_1=a_2= 1$, $y^1 =\tilde h^2 $, and $y^2 =b^2$.
Then by the \ref{linetype:lincomb} case of \cref{eqn:extendedMuSigma},
\begin{align*}
&\tmu(h^2) = \tmu(\tilde h^2) + \tmu(b^2) = 0\\
&\tSigma(h^2, W^1x) = \tSigma(\tilde h^2, W^1x) + \tSigma(b^2, W^1x) = 0 + 0 = 0\\
&\tSigma(h^2, b^1) = \tSigma(\tilde h^2, b^1) + \tSigma(b^2, b^1) = 0 + 0 = 0\\
&\tSigma(h^2, b^2) = \tSigma(\tilde h^2, b^2) + \tSigma(b^2, b^2) = 0 + 1 = 1\\
&\tSigma(h^2,v) = \tSigma(\tilde h^2, v) + \tSigma(b^2, v) = 0 + 0 = 0\\
&\tSigma(h^2,h^1) = \tSigma(\tilde h^2, h^1) +  \tSigma(b^2, h^1) = 0\\
&\tSigma(h^2,\tilde h^2) = \tSigma(\tilde h^2, \tilde h^2) + \tSigma(b^2, \tilde h^2) = 1 + 0 = 1\\
&\tSigma(h^2,h^2) = \tSigma(\tilde h^2, h^2) + \tSigma(b^2, h^2) = \tSigma(h^2, \tilde h^2) + \tSigma(h^2, b^2) = 1 + 1 = 2
.
\end{align*}
Note that $h^2$ turns out to be ``independent'' from $h^1$, i.e.\ $\tSigma(h^2, h^1) = 0$, just as one might expect from the mean field or the NNGP literature.

\paragraph{Distribution of the Program Output}
We are now done with calculating $\tmu(g)$ and $\tSigma(g, g')$ for all $g, g' \in \Gvars$.
Recall the output of the program is $v^\trsp x^2 / \sqrt n$, where $x^2$ was introduced via \ref{linetype:nonlin} by $x^2 := \relu(h^2)$.
According to \cref{cor:GPConv}, the output variance is then given by
\begin{align*}
    &\sigma_v^2\EV_z\phi^2(z) = \EV_z \relu(z)^2 = 1,\\
    &\text{where }
    z\sim\Gaus(\tmu(h^2),\tSigma(h^2, h^{2}))
        = \Gaus(0, 2).
\end{align*}

\subsubsection{MLP with Multiple Inputs \texorpdfstring{(\cref{tp:MLP2})}{}}

\label{sec:MLPmulti}
\begin{algorithm}[tb]
    \caption{MLP Computation on a Set of Inputs}
    \label{tp:MLP2}

    \begin{multicols}{2}
    \begin{algorithmic}
      \State {\it // Embeddings of inputs}
      \Require $W^1 x^{1}, \ldots, W^1 x^{B}: \Gtype(n)$
      \State {\it // Biases across $L$ layers}
      \Require $b^1, \ldots, b^L: \Gtype(n)$
      \State {\it // Weights from layer 2 on}
      \Require $W^2, \ldots, W^L: \Atype(n, n)$
      \State {\it // Readout weights}
      \Require $v: \Gtype(n)$
      \State
      \State
      \For{$i = 1, \ldots, B$}
        \State $h^{1i} := W^1 x^{i} + b^1: \Gtype(n)$
        \State $x^{1i} := \phi(h^{1i}): \Htype(n)$
        \For{$l = 2, \ldots, L$}
          \State $\tilde h^{li} := W^l x^{l-1, i}: \Gtype(n)$
          \State $h^{li} := \tilde h^{li} + b^l: \Gtype(n)$
          \State $x^{li} := \phi(h^{li}): \Htype(n)$
        \EndFor
      \EndFor
      \Ensure $(v^\trsp x^{L1}/\sqrt n, \ldots, v^\trsp x^{LB}/\sqrt n)$
      
    \end{algorithmic}
    \end{multicols}
\end{algorithm}

Now suppose we have an $L$-layer MLP with $B$ inputs $x^1, \ldots, x^B$.
\cref{tp:MLP2} expresses its computation (again, the ``for'' loops are shorthands for the unrolled series of assignments).

Here we will avoid computing out all values of $\tSigma$ but only those that affect the infinite-width GP.
By \cref{cor:GPConv}, the output of \cref{tp:MLP2} is distributed as
\begin{align*}
\KK_{ij}
    \defeq
        \Cov\lp \f{v^\trsp x^{Li}}{\sqrt n}, \f{v^\trsp x^{Lj}}{\sqrt n} \rp
    =
        \sigma_v^2 \EV_Z \phi(Z^{h^{Li}})\phi(Z^{h^{Lj}})
        \numberthis\label{eqn:MLPGPkernel}
\end{align*}
where $\sigma_v^2$ is the coordinatewise variance of $v$, $Z \sim \Gaus(\tmu, \tSigma)$, and $Z^{h^{Li}}$ is the component of $Z$ corresponding to $h^{Li}$ (likewise for $Z^{h^{Lj}}$).
Therefore, we need to compute $\tmu$ and $\tSigma$ for the G-vars $\{h^{L1}, \ldots, h^{LB}\}$.

\paragraph{Setup}
Suppose the inputs $x^i$ have dimension $m$.
If we sample the neural network parameters in the usual Glorot fashion,
\begin{equation}
\begin{aligned}
W^1_{\alpha\beta} &\sim \Gaus(0, \sigma_w^2/m),&
W^l_{\alpha\beta} &\sim \Gaus(0, \sigma_w^2/n),\; \forall 2 \le l \le L,\\
v_\alpha &\sim \Gaus(0, \sigma_v^2),&
b^l_{\alpha} &\sim \Gaus(0, \sigma_b^2),\; \forall l \in [L],
\end{aligned}
\label{eqn:MLPparamsampling}
\end{equation}
then we have $\Sigmain$ defined by
\begin{align*}
\Sigmain(W^1 x^i, W^1 x^j) = \sigma_w^2 x^i{}^\trsp x^j/m,\quad
\Sigmain(b^l, b^l) = \sigma_b^2,\quad
\Sigmain(v, v) = \sigma_v^2
\end{align*}
for all $i,j \in [B]$ and $l \in [L]$,
and $\Sigmain(g, g') = 0$ for any other pairs of input G-vars $g, g'$.
On the other hand, $\muin(g) = 0$ for all input G-vars $g$.

\paragraph{Computing $\tmu$}
From \cref{eqn:extendedMuSigma}, it is clear that $\muin = 0$ implies $\tmu = 0$.

\paragraph{Computing $\tSigma$}
Again, our goal is to compute $\tSigma$ restricted to the G-vars $\{h^{L1}, \ldots, h^{LB}\}$.

\begin{lemma}\label{lemma:MLPrec}
For any $l = 2, \ldots, L$ and any $i,j \in [B]$,
\begin{align*}
\tSigma(h^{li}, h^{lj})
    &=
        \sigma_w^2 \EV_{z_1,z_2} \phi(z_1) \phi(z_2) + \sigma_b^2
        \numberthis\label{eqn:MLPKernelRecursion}
\end{align*}
where $(z_1, z_2) \sim \Gaus\lp 0, \tSigma|_{h^{l-1,i}, h^{l-1,j}}\rp$, and (unrolling the restriction notation)
\begin{align*}
\tSigma|_{h^{l-1,i}, h^{l-1,j}} = \begin{pmatrix}
\tSigma(h^{l-1,i}, h^{l-1,i}) & \tSigma(h^{l-1,i}, h^{l-1,j})\\
\tSigma(h^{l-1,j}, h^{l-1,i}) & \tSigma(h^{l-1,j}, h^{l-1,j})
\end{pmatrix}.
\end{align*}

\end{lemma}
\begin{proof}
From \cref{tp:MLP2}, $h^{li}$ is introduced via \ref{linetype:lincomb} by
\begin{align*}
h^{li} := \tilde h^{li} + b^l.
\end{align*}
Thus by the \ref{linetype:lincomb} cases of \cref{eqn:extendedMuSigma}, we have
\begin{align*}
\tSigma(h^{li}, h^{lj})
    &=
        \tSigma(\tilde h^{li}, h^{lj})
        + \tSigma(b^l, h^{lj})
        \\
    &=
        \tSigma(\tilde h^{li}, \tilde h^{lj})
        + \tSigma(b^l, \tilde h^{lj})
        + \tSigma(\tilde h^{li}, b^l)
        + \tSigma(b^l, b^l)
        .
\end{align*}
Now, we cannot pattern match $\tSigma(b^l, \tilde h^{lj})$ with any of the cases of \cref{eqn:extendedMuSigma} other than the ``otherwise'' case, which means $\tSigma(b^l, \tilde h^{lj}) = 0$.
Likewise, $\tSigma(\tilde h^{li}, b^l) = 0$.
Therefore,
\begin{align*}
\tSigma(h^{li}, h^{lj})
    &=
        \tSigma(\tilde h^{li}, \tilde h^{lj})
        + \tSigma(b^l, b^l)
        \\
    &=
        \tSigma(\tilde h^{li}, \tilde h^{lj})
        + \sigma_b^2
        .
\end{align*}
Now let's analyze the $\tSigma(\tilde h^{li}, \tilde h^{lj})$ term.
The G-var $\tilde h^{li}$ is introduced via \ref{linetype:MatMul} by
\begin{align*}
\tilde h^{li} := W^l x^{l-1, i},\quad
\text{where}\quad
x^{l-1,i} := \phi(h^{l-1,i})
\end{align*}
and likewise for $h^{lj}$.
By the \ref{linetype:MatMul} case of \cref{eqn:extendedMuSigma}, we have
\begin{align*}
\tSigma(\tilde h^{li}, \tilde h^{lj})
    &=
        \sigma_w^2 \EV_Z \phi(Z^{h^{l-1,i}}) \phi(Z^{h^{l-1,j}})
\end{align*}
where $Z \sim \Gaus(\tmu, \tSigma).$
Since the integrand only depends two components of $Z$, we can simplify the expression as
\begin{align*}
\EV_Z \phi(Z^{h^{l-1,i}}) \phi(Z^{h^{l-1,j}})
    &=
        \EV_{z_1,z_2} \phi(z_1) \phi(z_2),\quad
        \text{where}\quad
        (z_1, z_2) \sim \Gaus\lp 0, \tSigma|_{h^{l-1,i}, h^{l-1,j}}\rp
        .
\end{align*}
Putting it all together, we recover the expression in the claim, as desired.
\end{proof}

\paragraph{Computing the GP Kernel $\KK$}
\cref{eqn:MLPKernelRecursion} along with \cref{eqn:MLPGPkernel} gives all we need to compute $\KK$ by recursion.
If $\phi$ has a nice V-transform $\Vt\phi$, then we can vectorize this equation and obtain the following algorithm (which, again, is by now well-known \cite{poole_exponential_2016,schoenholz_deep_2017,lee_deep_2018})
\begin{tcolorbox}[title=Computing MLP kernel on $B$ inputs]
Consider an $L$-layer MLP with nonlinearity $\phi$ at each layer.
Suppose we have $B$ network inputs $x^1, \ldots, x^B \in \R^m$, as in \cref{tp:MLP2}, and we sample the MLP parameters as in \cref{eqn:MLPparamsampling}.
Then the MLP converges in distribution to a GP on those $B$ inputs, with kernel $\KK$ computed as follows
\begin{enumerate}
    \item Initialize $\KK \in \R^{B \times B}$ by $\KK_{ij} \gets \sigma_w^2 x^i{}^\trsp x^j /m$
    \item For $l = 1, \ldots, L-1$, do $\KK \gets \sigma_w^2 \Vt\phi(\KK) + \sigma_b^2$
    \item Return $\KK \gets \sigma_v^2 \Vt\phi(\KK)$
\end{enumerate}
\end{tcolorbox}

\subsection{Simple RNN \texorpdfstring{(\cref{tp:RNN})}{}}

\label{sec:ExactKernelRNN}

By \cref{cor:GPConv}, we know that the output of \cref{tp:RNN},
\[\lp \f {v^\trsp s^{11}}{\sqrt{n}}, \ldots, \f{v^\trsp s^{t1}}{\sqrt{n}}, \f{v^\trsp s^{12}}{\sqrt{n}}, \ldots, \f{v^\trsp s^{r2}}{\sqrt{n}} \rp\]
is, in the large $n$ limit, distributed as a Gaussian with mean $0$ and the covariance $\KK$ where, for any $a, b \in \{1, 2\}$ (denoting sequence number),
\[
\KK_{ia,jb} \defeq
\lim_{n \to \infty} \Cov\lp \f{v^\trsp s^i{}^a}{\sqrt n}, \f{v^\trsp s^j{}^b}{\sqrt n}\rp = \sigma_v^2 \EV_Z \phi(Z^{h^{ia}})\phi(Z^{h^{jb}})
\]
where
$Z \sim \Gaus(\tmu, \tSigma)$, and $Z^{h^{ia}}$ is the component of $Z$ corresponding to $h^{ia}$ and likewise for $Z^{h^{jb}}.$
Therefore, we need to compute $\tmu$ and $\tSigma$ for the G-vars $\{h^{11}, \ldots, h^{s1}, h^{12}, \ldots, h^{t2}\}$ in \cref{tp:RNN}.

\paragraph{Setup}
Suppose the input tokens $x^{ia}$ to the RNN have dimension $m$.
We will obtain the $\tmu$ and $\tSigma$ for \cref{tp:RNN} with 
\begin{equation}
U_{\alpha \beta} \sim \Gaus(0, \sigma_U^2/m), W_{\alpha\beta} \sim \Gaus(0, \sigma_W^2/n), b_\alpha \sim \Gaus(0, \sigma_b^2), v_\alpha \sim \Gaus(0, \sigma_v^2).
\label{eqn:RNNinit}
\end{equation}
The randomization of $U$ induces the following covariance structure in the input token embeddings
\begin{align*}
    \Sigmain(Ux, Uy) = \sigma_U^2 x^\trsp y/m
    \numberthis\label{eqn:RNNinputcov}
\end{align*}
for any $x, y \in \{x^{i1}\}_{i=1}^t \cup \{x^{j2}\}_{j=1}^r$.
For any other pair $g, g'$ of input G-vars, the sampling implies $\Sigmain(g, g') = 0$.
Additionally, $\muin(g) = 0$ for all input G-var $g$.

\paragraph{Computing $\tmu$}
In fact, one can quickly see that $\tmu(g) = 0$ for all G-vars $g$, not just the input G-vars.

\paragraph{Computing $\tSigma$}
By \cref{eqn:extendedMuSigma}, all $\{h^{i1}, \tilde h^{i1}\}_{i=1}^t \cup \{h^{j2}, \tilde h^{j2}\}_{j=1}^r$ are possibly correlated with each other.
They satisfy the following recurrence
\begin{lemma}\label{lemma:tildehRec}
For any $a, b \in \{1, 2\}$, we have
\begin{align}
    \tSigma(h^{ia}, h^{jb})
        &=
            \tSigma(\tilde h^{ia}, \tilde h^{jb})
            +
            \Sigmain(Ux^{ia}, Ux^{jb})
            +
            \sigma_b^2,
        &\forall i, j \ge 1
            \label{eqn:hfromtildeh}
            \\
    \tSigma(\tilde h^{ia}, \tilde h^{jb})
        &=
            \sigma_W^2 \EV \phi(z_1)\phi(z_2),
        &\forall i, j \ge 2,
            \label{eqn:tildehRecurrence}
\end{align}
where
\[(z_1, z_2)
        \sim
            \Gaus(0, \tSigma|_{h^{i-1, a}, h^{j-1, b}}),\]
and the base case $\tSigma(\tilde h^{ia}, \tilde h^{jb}) = 0$ if $i \le 1$ or $j \le 1$.
Here, $\tSigma|_{set}$ means the submatrix of $\tSigma$ with rows and columns in $set$.
\end{lemma}
Note that the notation $b$ appears both as a sequence index and as the bias of the RNN.
Since the former will only appear as a superscript and the latter will not, there should be no confusion.

\begin{proof}

Note that for each $i \ge 2$, $h^{ia}$ is introduced via \ref{linetype:lincomb} by $h^{ia} := \tilde h^{ia} + Ux^{ia} + b$, where $\tilde h^{ia}$, $Ux^{ia}$, and $b$ are G-vars.
Then by the \ref{linetype:lincomb} case of \cref{eqn:extendedMuSigma},
\begin{align*}
    \tSigma(h^{ia}, h^{jb})
        &=
            \tSigma(\tilde h^{ia}, h^{jb})
            + \tSigma(Ux^{ia}, h^{jb})
            + \tSigma(b, h^{jb})
            \\
        &=
            \tSigma(\tilde h^{ia}, \tilde h^{jb})
            + \tSigma(Ux^{ia}, \tilde h^{jb})
            + \tSigma(b, \tilde h^{jb})\\
        &\phantomeq
            + \tSigma(\tilde h^{ia}, Ux^{jb})
            + \tSigma(Ux^{ia}, Ux^{jb})
            + \tSigma(b, Ux^{jb})\\
        &\phantomeq
            + \tSigma(\tilde h^{ia}, b)
            + \tSigma(Ux^{ia}, b)
            + \tSigma(b, b)
            .
            \\
\end{align*}
By the ``otherwise'' case of \cref{eqn:extendedMuSigma},
\begin{align*}
    \tSigma(Ux^{ia}, \tilde h^{jb})
    = \tSigma(b, \tilde h^{jb})
    = \tSigma(\tilde h^{ia}, Ux^{jb})
    = \tSigma(\tilde h^{ia}, b)
    = 0,
\end{align*}
and by the ``input G-var'' case of \cref{eqn:extendedMuSigma},
\begin{align*}
    \tSigma(b, Ux^{jb})
    = \tSigma(\tilde h^{ia}, b)
    = \Sigmain(b, Ux^{jb})
    = \Sigmain(\tilde h^{ia}, b)
    = 0
    .
\end{align*}
We thus have
\begin{align*}
    \tSigma(h^{ia}, h^{jb})
        &=
            \tSigma(\tilde h^{ia}, \tilde h^{jb})
            + \tSigma(Ux^{ia}, Ux^{jb})
            + \tSigma(b, b)
            \\
        &=
            \tSigma(\tilde h^{ia}, \tilde h^{jb})
            + \tSigma(Ux^{ia}, Ux^{jb})
            + \sigma_b^2
\end{align*}
which is \cref{eqn:hfromtildeh}.

Next, note that $\tilde h^{ia}$ is introduced via \ref{linetype:MatMul} by $\tilde h^{ia} := W s^{ia}$, and $s^{ia}$ is an H-var introduced by $s^{ia} := \phi(h^{ia})$.
Thus, by the \ref{linetype:MatMul} rule of \cref{eqn:extendedMuSigma},
\begin{align*}
    \tSigma(\tilde h^{ia}, \tilde h^{jb})
        &=
            \sigma_W^2 \EV_Z \phi(Z^{h^{ia}}) \phi(Z^{h^{jb}})
\end{align*}
where $Z \sim \Gaus(\tmu, \tSigma)$ and $Z^{h^{ia}}$ (resp.\ $Z^{h^{jb}}$) is its component corresponding to $h^{ia}$ (resp.\ $h^{jb}$).
Since the integrand only depends on the two components $Z^{h^{ia}}$ and $Z^{h^{jb}}$, we can rewrite
\begin{align*}
    \tSigma(\tilde h^{ia}, \tilde h^{jb})
        &=
            \sigma_W^2 \EV_{z_1, z_2} \phi(z_1) \phi(z_2)
\end{align*}
where $(z_1, z_2) \sim \Gaus(0, \tSigma|_{h^{i-1, a}, h^{j-1, b}})$.
This is \cref{eqn:tildehRecurrence}.
\end{proof}

Let us formulate the results above in a way more suggestive of the algorithm required to compute the kernel.
For any $2 \le p \le t, 2 \le q \le r$, define $\tilde H^{pq} \defeq \{\tilde h^{i1}\}_{i=2}^{p} \cup \{\tilde h^{j2}\}_{j=2}^{q}$ and $X^{pq} \defeq \{Ux^{i1}\}_{i=1}^p \cup \{Ux^{j2}\}_{j=1}^q.$
Denote by $\tSigma|_{\tilde H^{pq}}$ the restriction of $\tSigma$ to $\tilde H^{pq}$, and likewise $\Sigmain|_{X^{pq}}$ the restriction of $\Sigmain$ to $X^{pq}$.
We can visualize
\begin{align*}
\tSigma|_{\tilde H^{pq}}
  &=
    \begin{pmatrix}
    A^{pp}  & B^{pq}\\
    B^{pq}{}^\trsp &  C^{qq}
    \end{pmatrix}
    \in \R^{(p+q-2) \times (p+q-2)}
    &
\Sigmain|_{X^{pq}}
  &=
    \begin{pmatrix}
    P^{pp} & Q^{pq}\\
    Q^{pq}{}^\trsp & R^{qq}
    \end{pmatrix}
    \in \R^{(p+q) \times (p+q)}
    \numberthis\label{eqn:Sigmapq}
\end{align*}
where
\begin{align*}
A^{pp} &\defeq \{\tSigma(\tilde h^{i1}, \tilde h^{j1})\}_{i,j=2}^{p,p}&
C^{qq} &\defeq \{\tSigma(\tilde h^{i2}, \tilde h^{j2})\}_{i,j=2}^{q,q}&
B^{pq} &\defeq \{\tSigma(\tilde h^{i1}, \tilde h^{j2})\}_{i,j=2}^{p,q}\\
P^{pp} &\defeq \{\Sigmain(Ux^{i1}, Ux^{j1})\}_{i,j=1}^{p,p}&
R^{qq} &\defeq \{\Sigmain(Ux^{i2}, Ux^{j2})\}_{i,j=1}^{q,q}&
Q^{pq} &\defeq \{\Sigmain(Ux^{i1}, Ux^{j2})\}_{i,j=1}^{p,q}
.
\end{align*}
Let $\tSigma |^0_{\tilde H^{pq}}$ also denote $\tSigma|_{\tilde H^{pq}}$ padded with an additional column and an additional row of 0s on the left and top of each block $A^{pp}, B^{pq}, B^{pq}{}^\trsp, C^{qq}$:
\begin{align}
\tSigma|^0_{\tilde H^{pq}}
  &=
    \begin{pmatrix}
    0 & 0 & 0 & 0\\
    0 & A^{pp} & 0 & B^{pq}\\
    0 & 0 & 0 & 0\\
    0 & B^{pq}{}^\trsp & 0 & C^{qq}
    \end{pmatrix}
    \in \R^{(p+q) \times (p+q)}
    .
    \label{eqn:Sigma0pq}
\end{align}
Then \cref{eqn:hfromtildeh,eqn:tildehRecurrence} can be combined and vectorized as
\begin{align}
    \tSigma|_{\tilde H^{p+1,q+1}}
        =
            \Vt\phi \lp
                \tSigma |^0_{\tilde H^{pq}}
                + \Sigmain |_{X^{pq}}
                + \sigma_b^2
            \rp
        .
        \label{eqn:RNNkernelVectorized}
\end{align}

\cref{eqn:RNNkernelVectorized} quickly yields to a iterative, vectorized algorithm for computing $\tSigma|_{\tilde H^{tr}}$ (recall $t$ and $r$ are the lengths of the two input sequences), assuming that $\Vt\phi$ can be efficiently computed (such as those under \cref{defn:Vtransform}).
Then, with $H^{pq} \defeq \{h^{i1}\}_{i=1}^p \cup \{h^{j2}\}_{j=1}^q$, we have
\begin{align*}
    \tSigma|_{H^{tr}}
        =
            \tSigma|_{\tilde H^{tr}}^0
            + \Sigmain|_{X^{tr}}
            + \sigma_b^2
        .
\end{align*}

\paragraph{Computing the GP Kernel}
Finally, given $\tSigma|_{H^{tr}}$, by \cref{cor:GPConv}, the covariance of the output of \cref{tp:RNN} in the large $n$ limit is
\begin{align*}
    \KK = \sigma_v^2 \Vt\phi\lp \tSigma|_{H^{tr}} \rp
                    .
\end{align*}

\begin{tcolorbox}[title=Computing RNN kernel]
Consider a simple RNN with nonlinearity $\phi$, as in \cref{tp:RNN}.
Suppose we have 2 input sequences $x^{11},\ldots, x^{t1}$ and $x^{12}, \ldots, x^{r2} \in \R^m$.
Assume we sample the RNN parameters as in \cref{eqn:RNNinit}.
Then the outputs of the RNN converge jointly in distribution to a Gaussian with covariance computed as follows.
\begin{enumerate}
    \item Initialize $\Sigmain$ according to \cref{eqn:RNNinputcov}.
    \item Starting with $p = q = 0$, do
    \begin{enumerate}
        \item $\tSigma|_{\tilde H^{p+1,q+1}}
                    \gets
                        \sigma_w^2 \Vt\phi \lp
                            \tSigma |^0_{\tilde H^{pq}}
                            + \Sigmain |_{X^{pq}}
                            + \sigma_b^2
                        \rp$
                (see \cref{eqn:Sigmapq,eqn:Sigma0pq} for notations)
        \item Set $p \gets \min(p+1, t), q \gets \min(q+1, r)$.
        \item If $p$ and $q$ did not change, break.
    \end{enumerate}
    \item Compute
            $\tSigma|_{H^{tr}}
                \gets
                    \tSigma|_{\tilde H^{tr}}^0
                    + \Sigmain|_{X^{tr}}
                    + \sigma_b^2
                .
            $
    \item The output kernel is given by
        \begin{align*}
            \KK \gets \sigma_v^2 \Vt\phi\lp \tSigma|_{H^{tr}} \rp
            .
        \end{align*}
\end{enumerate}
\end{tcolorbox}

See our repository \repo{} for a reference implementation of this kernel.

\subsection{Batchnorm \texorpdfstring{(\cref{tp:batchnorm})}{}}
\label{sec:batchnormKernel}

As shown in \cref{sec:MoreExamples}, batchnorm (followed by a coordinatewise nonlinearity) can be just thought of a multivariate nonlinearity, and the computation of $\tSigma$ can largely follow the same pattern as for any other feedforward neural net (see \cref{tp:batchnorm}).
However, doing so efficiently is not so obvious, especially when the batch size is large.
In this section, we describe how to overcome this apparent difficulty.

\subsubsection{Batchnorm with Single Batch}
Let us compute the GP kernel for \cref{tp:batchnorm}, assuming there is only one batch.

\paragraph{Setup}
The network in \cref{tp:batchnorm} has parameters $W^1 \in \R^{n \times m}$, where $m$ is the input dimension, and $W^l \in \R^{n \times n}$ for $2 \le l \le L$.
Since batchnorm is scale-invariant, we will just assume that
\begin{align*}
W^1_{\alpha\beta} \sim \Gaus(0, 1/m),\quad
W^l_{\alpha\beta} \sim \Gaus(0, 1/n),\; \forall l \ge 2,
\end{align*}
and $v_\alpha \sim \Gaus(0, \sigma_v^2)$.
This means that the initial \netsor sampling data have values
\begin{align*}
\Sigmain(v, v) &= \sigma_v^2\\
\Sigmain(W^1 x^{ia}, W^1 x^{i'a'}) &= x^{ia}{}^\trsp x^{i'a'}/m,
\end{align*}
and $\muin = 0$ identically.

\paragraph{Computing $\tmu$}
As before, from \cref{eqn:extendedMuSigma}, it is easy to see that $\muin = 0$ implies $\tmu = 0$.

\paragraph{Computing $\tSigma$}
Applying \cref{eqn:extendedMuSigma} in the fashion of \cref{lemma:MLPrec}, we get
\begin{lemma}\label{lemma:batchnorm1batchrec}
For any $a \in [k]$ and $2 \le l \le L$, let $H^{la}$ denote the set $\{h^{lia}\}_{i \in [B_a]}$.
Also write $H^{1a} \defeq \{W^1 x^{ia}\}_{i \in [B_a]}$.
Recall that $\tSigma|_{set}$ denotes the square submatrix of $\tSigma$ with rows and columns in $set$.
Then for any $a \in [k]$ and $l = 1, \ldots, L-1$,
\begin{align*}
\tSigma|_{H^{l+1,a}} = \EV_{\zeta \sim \Gaus(0, \tSigma|_{H^{la}})} \tilde\phi(\zeta) \tilde \phi(\zeta)^\trsp
    = \Vt{\tilde \phi}\lp\tSigma|_{H^{la}}\rp
    \in \R^{B_a \times B_a},
    \numberthis\label{eqn:batchnorm1batchrec}
\end{align*}
where $\tilde \phi$ is batchnorm followed by coordinatewise nonlinearity $\phi$, as in \cref{sec:MoreExamples} and \cref{tp:batchnorm}.
\end{lemma}

\textit{A priori}, evaluating this expectation requires a $B_a$-dimensional Gaussian integral, which becomes intractible when $B$ is large.
However, if $\phi = \relu$, then surprisingly one can reduce the $B$-dimensional integral this Gaussian expectation seems to require to a 1-dimensional integral:
By \cite{yang_mean_2019} we can express \cref{eqn:batchnorm1batchrec} as
\begin{equation}
\tSigma|_{H^{l+1,a}} = B_a \int_0^\infty
    \frac{\Vt\relu(\Sigma^G (I + 2s\Sigma^G)^{-1})}
    {\sqrt{\det(I + 2 s \Sigma^G)}} \dd s
    \label{eqn:batchnorm1D}
\end{equation}
where $\Vt\relu$ is as in \cref{fact:Vrelu}, and
\begin{align*}
    \Sigma^G &\defeq G \tSigma|_{H^{la}} G\\
    G &\defeq I_B - \frac 1 B \mathbf 1 \mathbf 1^\trsp.
\end{align*}

\subsubsection{Batchnorm with Multiple Batches}
Now let's consider the covariance of preactivations between different batches.
We maintain the same setup as above, and as usual $\tmu = 0$ identically.

\paragraph{Computing $\tSigma$}
By another straightforward application of \cref{eqn:extendedMuSigma}, we get
\begin{lemma}\label{lemma:batchnorm2batchrec}
With the same notation as in \cref{lemma:batchnorm1batchrec}, for two different batches $a \ne b$,
\begin{align*}
\tSigma(H^{l+1,a}, H^{l+1,b})
    &=
        \EV \tilde \phi(\zeta_1) \tilde \phi(\zeta_2)^\trsp
\end{align*}
where the expectation is taken over
\[
\begin{pmatrix}
\zeta_1\\
\zeta_2
\end{pmatrix}
\sim
\Gaus\lp
0,
\begin{pmatrix}
\tSigma(H^{la}, H^{la}) & \tSigma(H^{la}, H^{lb})\\
\tSigma(H^{lb}, H^{la}) & \tSigma(H^{lb}, H^{lb})
\end{pmatrix}
\rp.
\]
\end{lemma}

Again, this expectaion seems to require an integral in $(B_a + B_b)$ dimensions.
However, via some integral tricks, this expectation can be transformed to the following 2D integral:

\begin{equation}
\tSigma(H^{l+1,a}, H^{l+1,b}) = 
\sqrt{B_a B_b} \pi^{-1} \int_0^\infty \dd s \int_0^\infty \dd t\ 
(st)^{-1/2} \det(I_{B_a+B_b} + 2 \Omega)^{-1/2}
\Vt\relu(\Pi)_{12}
\label{eqn:batchnorm2D}
\end{equation}
where
\begin{align*}
    \Omega &=
        D^{1/2}\begin{pmatrix} G_a\tSigma(Y, Y) G_a & G_a \tSigma(Y, Y') G_b\\ G_b \tSigma(Y', Y) G_a & G_b \tSigma(Y', Y') G_b \end{pmatrix} D^{1/2}
            \\
    \Pi &=
        D^{-1/2} \Omega (I + 2 \Omega)^{-1} D^{-1/2}
            \\
    D &=
        s I_{B_a} \oplus t I_{B_b}
        = \begin{pmatrix} s I_{B_a} & 0 \\ 0 & t I_{B_b} \end{pmatrix}
            \\
    G_a &= I_{B_a} - B_a^{-1} \mathbf{1}\mathbf{1}^\top
        \\
    G_b &= I_{B_b} - B_b^{-1} \mathbf{1}\mathbf{1}^\top
\end{align*}
and $\Vt\relu(\Pi)_{12}$ is the block of $\Vt\relu(\Pi)$ on the first row, second column, of size $B_a \times B_b$.

\paragraph{Computing the GP Kernel $\KK$}
The computation is similar to \cref{lemma:batchnorm1batchrec,lemma:batchnorm2batchrec}, so let us directly summarize the entire algorithm below.

\begin{tcolorbox}[title=Computing Batchnorm+ReLU kernel]
We compute the kernel of a $L$-layer fully-connected network where each layer has a batchnorm followed by ReLU, as shown in \cref{tp:batchnorm}.
Let the input dimension be $m$ and the common width of hidden layers be $n$.
Sample the first layer weights $W^1$ as $W^1_{\alpha\beta} \sim \Gaus(0, \sigma_W^2/m)$ and each higher layer's weight matrix $W^l$ as $W^l_{\alpha \beta} \sim \Gaus(0, \sigma_W^2/n)$ with $\sigma_W = 1$ (since batchnorm is scale-invariant), and the readout layer as $v_\alpha \sim \Gaus(0, \sigma_v^2)$.
We omit biases since batchnorm is shift invariant.
Suppose we have $k$ batches of inputs $\{x^{ib}: i \in [B_b], b \in [k]\}$, with batch $b$ containing $B_b$ inputs.
Then the outputs of the network converge jointly in distribution to a Gaussian $\Gaus(0, \KK)$ with $\KK$ computed as follows.
\begin{enumerate}
    \item Initialize $\{\KK^0_{ab} \in \R^{B_a \times B_b}\}_{a,b=1}^k$ by $(\KK^0_{ab})_{ij} \gets x^{ia}{}^\trsp x^{jb}/m.$
    \item For $l = 1, \ldots, L$, do
    \begin{enumerate}
        \item For $a = 1, \ldots, k$, do
        \begin{enumerate}
            \item $\KK^l_{aa} \gets \Vt{\tilde \phi}(\KK^{l-1}_{aa})$ by evaluating a 1D integral according to \cref{eqn:batchnorm1D}.
        \end{enumerate}
        \item For $a, b \in [k]$, $a\ne b$, do
        \begin{enumerate}
            \item Compute $\KK^l_{ab}$ by using $\KK^{l-1}_{ab}, \KK^{l-1}_{aa}, \KK^{l-1}_{bb}$ and evaluating a 2D integral according to \cref{eqn:batchnorm2D}.
        \end{enumerate}
    \end{enumerate}
    \item Return
    \[
        \KK \gets
        \sigma_v^2
        \begin{pmatrix}
        \KK^L_{11} & \cdots & \KK^L_{1k}\\
        \vdots & \ddots & \vdots\\
        \KK^L_{k1} & \cdots & \KK^L_{kk}
        .
        \end{pmatrix}
        \in \R^{\sum_{a} B_a \times \sum_{a} B_a}
    \]
\end{enumerate}
\end{tcolorbox}

\paragraph{Vectorized Implementation}

In our repo \repo, we show how to implement single- and multi-batch BN using the \texttt{quadpy} \cite{quadpy} package for vectorized quadrature integration, and by using eigendecomposition to simplify the computation of the integrand in the integrals above.

\subsection{Convolution and Pooling \texorpdfstring{(\cref{tp:conv})}{}}

Convolution and pooling, in the context of neural network-Gaussian process correspondence, have already been treated in \cite{novak_bayesian_2018,garriga-alonso_deep_2018}.
In this section we revisit the same derivations from the perspective of tensor programs.

\subsubsection{CNN with Single Input}

Let us compute the GP kernel for \cref{tp:conv}, for the following setup:

\paragraph{Setup}
The CNN in \cref{tp:conv} has parameters $\{W^1_j\}_{j \in ker^1}, \{W^2_j\}_{j \in ker^2}$.
It has widths $n^1$ and $n^2$, but for simplicity, let's assume $n^1 = n^2 = n$.
Each input image is given as $\{x_i \in \R^m\}_{i \in pos^0}$ where $pos^0$ denotes ``pixel locations'' and $m$ denotes number of channels
(for example, $pos^0 = [32] \times [32]$ and $m = 3$ for the CIFAR10 dataset).
Suppose we sample the parameters as
\begin{align*}
(W^1_j)_{\alpha \beta} \sim \Gaus(0, \sigma_w^2/m),\; \forall j \in ker^1,\quad
(W^2_j)_{\alpha\beta} \sim \Gaus(0, \sigma_w^2/n^1),\; \forall j \in ker^2,\quad
v_\alpha \sim \Gaus(0, \sigma_v^2/n)
.
\end{align*}
This induces $\Sigmain$ as follows:
\begin{align*}
\Sigmain(v, v) &= \sigma_v^2\\
\Sigmain(W^1_j x_{i+j}, W^1_{j'} x_{i'+j'}) &= \sigma_w^2 x_{i+j}^\trsp x_{i'+j'}/ m
\end{align*}
for any $i,i' \in pos^1 ,j,j' \in ker^1$ such that $i + j, i'+j' \in pos^0$; and $\Sigmain(g, g') = 0$ for any other pairs of input G-vars.
In addition, $\muin = 0$ identically.

\paragraph{Computing $\tmu$}
As before, from \cref{eqn:extendedMuSigma}, it is easy to see that $\muin = 0$ implies $\tmu = 0$.

\paragraph{Computing $\tSigma$}
By straightforward applications of \cref{eqn:extendedMuSigma}, we obtain the following
\begin{lemma}
For any $i , i' \in pos^1$,
\begin{align*}
    \tSigma(h^1_i, h^1_{i'})
        &=
            \sum_{j, j' \in ker^1} \Sigmain(W^1_j x_{i+j}, W^1_{j'} x_{i'+j'})
            .
\end{align*}
In addition, for any $j, j' \in ker^2$ such that $i+j, i'+j' \in pos^1$,
\begin{align*}
    \tSigma(h^2_{i;j}, h^2_{i';j'})
        &=
            \sigma_w^2 \EV_{z_1,z_2}\phi(z_1)\phi(z_2)
            \ind(j=j')
            ,
\end{align*}
where $(z_1, z_2) \sim \Gaus(0, \tSigma|_{h^1_{i}, h^1_{i'}})$.
Finally, for any $i, i' \in pos^2$,
\begin{align*}
    \tSigma(h^2_i, h^2_{i'})
        &=
            \sum_{j, j' \in ker^2}
                \tSigma(h^2_{i;j}, h^2_{i';j'})
\end{align*}
where the sum is over all $j,j' \in ker^2$ such that $i+j, i'+j' \in pos^{1}$.
\end{lemma}

\paragraph{Computing the GP kernel $\KK$}
By \cref{cor:GPConv}, the output of the CNN converges in distribution to $\Gaus(0, \KK)$ where $\KK$ is a scalar given by
\begin{align*}
\KK \defeq \lim_{n \to \infty} \Var\lp \f{v^\trsp \bar x^2} n \rp 
    = \sigma_v^2 \EV_Z \lp \f 1 {|pos^2|} \sum_i \phi(Z^{h^2_i}) \rp^2
    = \sigma_v^2 \EV_{i,i' \in pos^2} \EV_Z \phi(Z^{h^2_i}) \phi(Z^{h^2_{i'}})
\end{align*}
where $Z \sim \Gaus(\tmu, \tSigma)$ and where in the last expression, $i,i'$ are sampled independently and uniformly from $pos^2$.
If we let $H^2 \defeq \{h^2_i\}_{i \in pos^2}$, then $\KK$ can be computed as
\[\KK = \sigma_v^2 \EV_{i,i'\in pos^2} \Lambda_{ii'}
\quad\text{where}\quad
\Lambda \defeq \Vt\phi\lp \tSigma|_{H^2}\rp.\]

\subsubsection{CNN with Multiple Inputs}

\begin{algorithm}[tb]
    \caption{$L$-layer Convolutional Network with Global Average Pooling}
    \label{tp:conv2}
    \begin{algorithmic}
      \Require $\{W^1_j x_{i+j}^a: \Gtype(n^1)\}_{\substack{a\in[B],\\j \in ker^1, i \in pos^1\\ s.t.\ i+j \in pos^0}}$
      \Require $\{W^l_j: \Atype(n^l, n^{l-1})\}_{2\le l \le L, j \in ker^l}$
      \State {\it // Readout weights}
      \Require $v: \Gtype(n^L)$
      \For{$a \in [B]$}
          \State {\it // Layer 1 convolution}
          \For{$i \in pos^1$}
          \State {\it // Directly use input embeddings}
          \State {\it // \ref{linetype:lincomb}}
          \State {\it // Sum is over all $j \in ker^1$ such that}
          \State {\it // there is $i' \in pos^{0}$ with $i' = j + i$}
          \State $h^{1a}_i := \sum_{j} W^1_j x^a_{i+j}: \Gtype(n^1)$
          \EndFor
          \State {\it // Higher layer convolutions}
          \For{$l = 2, \ldots, L$}
              \For{$i \in pos^{l-1}$}
                \State $x^{l-1,a}_i := \phi(h^{l-1,a}_i): \Htype(n^{l-1})$
              \EndFor
              \For{$j \in ker^l, i \in pos^{l}\ s.t.\ i+j \in pos^{l-1}$}
                  \State {\it // \ref{linetype:MatMul}}
                  \State $h^{la}_{i;j} := W^{l}_j x^{l-1,a}_{i+j}: \Gtype(n^l)$
              \EndFor
              \For{$i \in pos^l$}
                \State {\it // Sum is over all $j \in ker^l$ such that}
                \State {\it // there is $i' \in pos^{l-1}$ with $i' = j + i$}
                \State $h^{la}_i := \sum_{j} h^{la}_{i;j}: \Gtype(n^l)$
              \EndFor
          \EndFor
          \State {\it // Nonlinearity \& Global Average Pooling}
          \State $\bar x^{La} := \f 1 {|pos^L|} \sum_{i \in ker^L} \phi(h^{La}_i): \Htype(n^L)$
      \EndFor
      \Ensure $v^\trsp \bar x^{L1} / \sqrt {n^L}, \ldots, v^\trsp \bar x^{LB} / \sqrt {n^L}$
    \end{algorithmic}
\end{algorithm}

Now consider the general case when we have multiple inputs $x^1, \ldots, x^B$ and have $L$ layers (but, for simplicity, still no bias), as in \cref{tp:conv2}.
The derivation is very similar to the single input case, but we will err on the side of completeness.

\paragraph{Setup}
The CNN in \cref{tp:conv2} has parameters $\{W^l\}_{l \in [L], j \in ker^l}$.
It has widths $n^1, \ldots, n^L$, but as before, we shall assume they are all equal to $n$ for simplicity.
Each input image $x^a$ is given as $\{x_i^a \in \R^m\}_{i \in pos^0}$ where $pos^0$ denotes ``pixel locations'' and $m$ denotes number of channels.
Suppose we sample the parameters as
\begin{align*}
(W^1_j)_{\alpha \beta} \sim \Gaus(0, \sigma_w^2/m),\; \forall j \in ker^1,\quad
(W^l_j)_{\alpha\beta} \sim \Gaus(0, \sigma_w^2/n^{l-1}),\; \forall j \in ker^l
,
\end{align*}
for all $l =2, \ldots, L$, and $v_\alpha \sim \Gaus(0, \sigma_v^2/n)$.
This induces $\Sigmain$ as follows:
\begin{align*}
\Sigmain(v, v) &= \sigma_v^2\\
\Sigmain(W^1_j x^a_{i+j}, W^1_{j'} x^{a'}_{i'+j'}) &= \sigma_w^2 x^a_{i+j}{}^\trsp x^{a'}_{i'+j'}/ m
\end{align*}
for any $a,a' \in [B]$ and any $i,i' \in pos^1 ,j,j' \in ker^1$ such that $i + j, i'+j' \in pos^0$; and $\Sigmain(g, g') = 0$ for any other pairs of input G-vars.
In addition, $\muin = 0$ identically.

\paragraph{Computing $\tmu$}
As before, from \cref{eqn:extendedMuSigma}, it is easy to see that $\muin = 0$ implies $\tmu = 0$.

\paragraph{Computing $\tSigma$}
By straightforward applications of \cref{eqn:extendedMuSigma}, we obtain the following
\begin{lemma}\label{eqn:CNNRec}
For any $a, a' \in [B]$ and any $i , i' \in pos^1$,
\begin{align*}
    \tSigma(h^{1a}_i, h^{1a'}_{i'})
        &=
            \sum_{j, j' \in ker^1} \Sigmain(W^1_j x^a_{i+j}, W^1_{j'} x^{a'}_{i'+j'})
            .
            \numberthis\label{eqn:CNNlayer1}
\end{align*}
In addition, for any $2 \le l \le L$ and any $j, j' \in ker^2$ such that $i+j, i'+j' \in pos^1$,
\begin{align*}
    \tSigma(h^{la}_{i;j}, h^{la'}_{i';j'})
        &=
            \sigma_w^2 \EV_{z_1,z_2}\phi(z_1)\phi(z_2)
            \ind(j=j')
            ,
            \numberthis\label{eqn:CNNprelayerl}
\end{align*}
where $(z_1, z_2) \sim \Gaus\lp 0, \tSigma|_{h^{l-1,a}_{i}, h^{l-1,a'}_{i'}}\rp$.
Finally, for any $i, i' \in pos^l$,
\begin{align*}
    \tSigma(h^{la}_i, h^{la'}_{i'})
        &=
            \sum_{j, j'}
                \tSigma(h^{la}_{i;j}, h^{la'}_{i';j'})
            \numberthis\label{eqn:CNNlayerl}
\end{align*}
where the sum is over all $j,j' \in ker^l$ such that $i+j, i' + j' \in pos^{l-1}$.
\end{lemma}
These equations are all we need to compute the GP kernel $\KK$.

\paragraph{Computing the GP kernel $\KK$}
By \cref{cor:GPConv}, the output of the CNN converges in distribution to $\Gaus(0, \KK)$ where $\KK \in \R^{B \times B}$ is given by
\begin{align*}
\KK_{aa'} \defeq \lim_{n \to \infty} \Cov\lp \f{v^\trsp \bar x^{La}}{\sqrt n}, \f{v^\trsp \bar x^{La'}}{\sqrt{n}} \rp 
    &= \sigma_v^2 \EV_Z \lp \f 1 {|pos^L|} \sum_i \phi(Z^{h^{La}_i}) \rp
                        \lp \f 1 {|pos^L|} \sum_i \phi(Z^{h^{La'}_i}) \rp
                        \\
    &= \sigma_v^2 \EV_{i,i' \in pos^L} \EV_Z \phi(Z^{h^{La}_i}) \phi(Z^{h^{La'}_{i'}})
    \numberthis\label{eqn:CNNGPkernel}
\end{align*}
where $Z \sim \Gaus(\tmu, \tSigma)$ and where in the last expression, $i,i'$ are sampled independently and uniformly from $pos^L$.
Since $\tSigma(h^{La}_i, h^{La'}_{i'})$ can be obtained recursively via \cref{eqn:CNNRec}, one can compute $\KK$ easily via recursion.
But we can do better by vectorizing the whole computation.

\paragraph{Vectorized Implementation}

\newcommand{\vecSigma}{\mathbf{\Sigma}}

Let us define the 4-tensor
\[\vecSigma^l = \{\vecSigma^l_{aa'ii'}: a, a' \in [B],
i,i' \in pos^l\}\]
by
\begin{align*}
\vecSigma^l_{aa'ii'} \defeq \tSigma(h^{la}_i, h^{la'}_{i'})
.
\end{align*}
Then \cref{eqn:CNNlayer1} corresponds to
\begin{align*}
\vecSigma^1_{aa'ii'} = \sigma_w^2 \sum_{j,j'} x^a_{i+j}{}^\trsp x^{a'}_{i'+j'} / m
\end{align*}
where the sum is over all $j,j' \in ker^1$ such that $i+j, i'+j' \in pos^0$.
For $l = 2, \ldots, L-1$, \cref{eqn:CNNprelayerl,eqn:CNNlayerl} can be vectorized as
\begin{align*}
    \vecSigma^{l}_{aa'} &= \kappa^l * \hat \vecSigma^{l}_{aa'},\quad
    \text{where}\quad
    \hat \vecSigma^{l} = \sigma_w^2 \Vt\phi(\vecSigma^{l-1}),
\end{align*}
treating $\vecSigma^{l-1}$ as a $(B \cdot pos^{l-1}) \times (B \cdot pos^{l-1})$ matrix, and
$\kappa^l *$ is the ``convolution''
\begin{align*}
    (\kappa^l * \hat \vecSigma^{l}_{aa'})_{ii'}
        &=
            \sum_{
            \substack{j,j'\in ker^l\\i+j \in pos^{l-1}\\i'+j'\in pos^{l-1}
            }}
            (\hat \vecSigma^l_{aa'})_{i+j, i'+j'}
            .
            \numberthis\label{eqn:kernelconv}
\end{align*}
This $\kappa^l$ convolution can indeed be implemented as a (CUDA) convolutional operation, vectorized over all $a,a'$.

Finally, to obtain the infinite-width GP kernel $\KK$ of the CN output, we can vectorize \cref{eqn:CNNGPkernel} as
\begin{align*}
\KK_{aa'} = \sigma_v^2 E * \hat \vecSigma^L,\quad\text{where}\quad
\hat \vecSigma^L = \Vt\phi(\vecSigma^{L-1}),
\end{align*}
and $E*$ denotes the spacial averaging
\begin{align*}
(E * \hat \vecSigma^L)_{aa'} \defeq
    \EV_{i,i'\in pos^L} \hat \vecSigma^L_{aa'ii'}
.
            \numberthis\label{eqn:kernelGAP}
\end{align*}
Again, $E*$ can be implemented as a convolution operator vectorized over all $a,a'$.

In summary,
\begin{tcolorbox}[title=Computing CNN Kernel]
Suppose we have an $L$-layer convolutional neural network with coordinatewise nonlinearity $\phi$ but no bias, as in \cref{tp:conv2}, that takes images with $m$ channels and of size $pos^0 \times pos^0$.
Suppose we have a set of inputs $x^1, \ldots, x^B$ where each input $x^a$ is given as $x^a = \{x^a_i \in \R^m\}_{i \in pos^0}$.
Then the CNN outputs converge in distribution to a Gaussian $\Gaus(0, \KK)$ where $\KK \in \R^{B \times B}$ can be calculated as follows.
\begin{enumerate}
    \item Initialize $\vecSigma^1 \in \R^{B \times B \times pos^l \times pos^l}$ by 
        \begin{align*}
        \vecSigma^1_{aa'ii'} = \sigma_w^2 \sum_{j,j'} x^a_{i+j}{}^\trsp x^{a'}_{i'+j'} / m
        \end{align*}
        where the sum is over all $j,j' \in ker^1$ such that $i+j, i'+j' \in pos^0$.
    \item For $l = 2, \ldots, L-1$, do 
        \begin{enumerate}
        \item $\vecSigma^l \gets \sigma_w^2 \kappa^l * \Vt\phi (\vecSigma^{l-1})$ (see \cref{eqn:kernelconv} for $\kappa^l$'s definition)
        \end{enumerate}
    \item return $\KK \gets \sigma_v^2 E * \Vt\phi\lp\vecSigma^{L-1}\rp$
        (see \cref{eqn:kernelGAP} for $E$'s definition)
\end{enumerate}
\end{tcolorbox}

\subsection{GRU \texorpdfstring{(\cref{tp:GRU})}{}}
\label{sec:GRUkernel}

We demonstrate how to compute the GP kernel for GRU as encoded in \netsor by \cref{tp:GRU}.
A key distinguishing feature of this conversion is that we will need to compute high dimensional Gaussian expectations, where the dimension is as large as the number of timesteps unrolled, in contrast to the simple RNN case \cref{tp:RNN}.
These Gaussian expectations correspond to the expected values of multiplication of the gate values across time.

We first proceed with a single input sequence, as in \cref{tp:GRU}.
We then comment on the generalization to multiple sequences at the end.

\paragraph{Setup}
We will obtain the $\tmu$ and $\tSigma$ for \cref{tp:GRU} with 
\begin{equation}
\begin{aligned}
(U_\zz)_{\alpha \beta}, (U_\rr)_{\alpha \beta}, (U_\hh)_{\alpha \beta} &\sim \Gaus(0, \sigma_U^2/n),&
(W_\zz)_{\alpha\beta}, (W_\rr)_{\alpha\beta}, (W_\hh)_{\alpha\beta} &\sim \Gaus(0, \sigma_W^2/n),\\
(b_\zz)_\alpha, (b_\rr)_\alpha, (b_\hh)_\alpha &\sim \Gaus(0, \sigma_b^2),&
v_\alpha \sim \Gaus(0, \sigma_v^2),&\text{ and }h^0 = 0.
\end{aligned}
\label{eqn:GRUinit}
\end{equation}

Suppose the input tokens $x^i$ to the GRU have dimension $m$.
The randomization of $U$ induces the following covariance structure in the input token embeddings
\begin{align*}
    \Sigmain(Ux, Uy) = \sigma_U^2 x^\trsp y /m
\end{align*}
for any $x, y \in \{x^i\}_{i=1}^T$.
For any other pair $g, g'$ of input G-vars, $\Sigmain(g, g') = 0$.
Additionally, $\muin(g) = 0$ for all input G-vars $g$.

\paragraph{Computing $\tmu$}
In fact, one can quickly see that $\tmu(g) = 0$ for \emph{all} G-vars $g$.

\paragraph{Computing $\tSigma$}
Applying \cref{eqn:extendedMuSigma} to \cref{tp:GRU} and some simplification in the manner of \cref{lemma:tildehRec}'s proof yields, for any two times $t, s$,
\begin{align*}
\tSigma(\tilde z^t, \tilde z^s)
    &=
        \tSigma(h_\zz^{t}, h_\zz^{s}) + \sigma_{U}^2 x^t{}^\trsp x^s/m + \sigma_b^2
        \numberthis\label{eqn:GRUrecz}
        \\
\tSigma(\tilde r^t, \tilde r^s)
    &=
        \tSigma(h_\rr^{t}, h_\rr^{s}) + \sigma_U^2 x^t{}^\trsp x^s/m + \sigma_b^2
        \numberthis\label{eqn:GRUrecr}
        \\
\tSigma(\tilde h^t, \tilde h^s)
    &=
        \tSigma(h_\hh^{t}, h_\hh^{s}) + \sigma_U^2 x^t{}^\trsp x^s/m + \sigma_b^2
        \numberthis\label{eqn:GRUrech}
        \\
\tSigma(h_\zz^{t}, h_\zz^{s})
    &= \tSigma(h_\rr^{t}, h_\rr^{s})\\
    &=
        \sigma_W^2 \sum_{i=1}^t \sum_{j=1}^s \bigg\{
        \EV \phi(Z^{\tilde h^i})\phi(Z^{\tilde h^j})
        \\
    &\phantom{{}={}}\quad
        \times 
            \EV 
            \left[\sigma(Z^{\tilde z^i}) \prod_{p = i+1}^t (1 - \sigma(Z^{\tilde z^{p}}))\right]
            \times \left[
            \sigma(Z^{\tilde z^j}) \prod_{q = j+1}^s (1 - \sigma(Z^{\tilde z^{q}}))\right]
            \bigg\}
          \numberthis\label{eqn:GRUrechz}
          \\
\tSigma(h_\hh^{t}, h_\hh^{s})
    &=
        \sigma_W^2 \tSigma(h_\zz^{t}, h_\zz^{s}) \EV \sigma(Z^{\tilde r^t})\sigma(Z^{\tilde r^s})
        \numberthis\label{eqn:GRUrechh}
\end{align*}
where expectations are taken over $Z = \{Z^g\}_{g\text{ is G-var}} \sim \Gaus(\tmu, \tSigma)$, which has one component for each G-var in the program.
Then, applying \cref{cor:GPConv}, we see that the output of the GRU
\begin{equation*}
(v^\trsp h^1/\sqrt{n}, \ldots, v^\trsp h^T/\sqrt{n})
\end{equation*}
converges in distribution to a zero mean Gaussian distribution with covariance matrix $\KK = \{\KK_{ts}\}_{t, s=1}^T$,
\begin{align*}
\KK_{ts}
    &=
    \sigma_v^2 \sum_{i=1}^t \sum_{j=1}^s \bigg\{
        \EV \phi(Z^{\tilde h^i})\phi(Z^{\tilde h^j})
        \\
    &\phantom{{}={}}\quad
        \times 
            \EV 
            \left[\sigma(Z^{\tilde z^i}) \prod_{p = i+1}^t (1 - \sigma(Z^{\tilde z^{p}}))\right]
            \times \left[
            \sigma(Z^{\tilde z^j}) \prod_{q = j+1}^s (1 - \sigma(Z^{\tilde z^{q}}))\right]
            \bigg\}
            .
        \numberthis\label{eqn:GRUoutputcov}
\end{align*}

\cref{eqn:GRUrecr,eqn:GRUrecz,eqn:GRUrech,eqn:GRUrechz,eqn:GRUrechh,eqn:GRUoutputcov} yield the complete set of equations to compute the output covariance $\KK$, but to do so efficiently rests entirely on evaluating the the possibly $T$-dimensional integral behind
\begin{equation}
\EV 
    \left[\sigma(Z^{\tilde z^i}) \prod_{p = i+1}^t (1 - \sigma(Z^{\tilde z^{p}}))\right]
    \times \left[
    \sigma(Z^{\tilde z^j}) \prod_{q = j+1}^s (1 - \sigma(Z^{\tilde z^{q}}))\right]
    .
    \label{eqn:highdimSigmoidExpectation}
\end{equation}

For general $\sigma$ and $\phi$, this is hopeless.
However, when $\phi = \erf$ and $\sigma = (1 + \erf)/2$ --- which approximate $\phi = \tanh$ and $\sigma = $ sigmoid --- \cref{eqn:highdimSigmoidExpectation} can in fact be evaluated efficiently by reducing it to a Gaussian orthant probability, which can be evaluated efficiently \cite{mvtnorm}:

\begin{lemma}\label{lemma:erfTrick}
Let $\phi = \erf$ and $\sigma = (1 + \erf)/2$.
Then for any $\mu \in \R^T$ and any PSD $\Sigma \in \R^{T \times T}$,
\begin{align*}
    \EV_{x \sim \Gaus(\mu, \Sigma)} \prod_{i=1}^T \phi(x_i)
    &= \EV_{x \sim \Gaus(\mu, \Sigma + \f 1 2 I)} \prod_{i=1}^T \sgn(x_i)
        \\
    \EV_{x \sim \Gaus(\mu, \Sigma)} \prod_{i=1}^T \sigma(x_i)
    &= \EV_{x \sim \Gaus(\mu, \Sigma + \f 1 2 I)} \prod_{i=1}^T \ind(x_i \ge 0)\\
    &= \Pr_{x \sim \Gaus(\mu, \Sigma + \f 1 2 I)} [x \ge 0]
    .
\end{align*}
\end{lemma}

\begin{remk}
Observe that if $T = 2$, then \cref{lemma:erfTrick} recovers the arccosine kernel of erf via \cref{fact:Verf}:
\begin{align*}
    \EV\left[\erf(x_1) \erf(x_2): x \sim \Gaus(0, \Sigma)\right]
    &=
        \f 2 \pi \arcsin \f{\Sigma_{12}}{\sqrt{(\Sigma_{11} + \f 1 2)(\Sigma_{22} + \f 1 2)}}.
\end{align*}
\end{remk}

\begin{algorithm}[t]
    \caption{GRU, Multiple Input Sequences}
    \label{tp:GRU2}
    \begin{algorithmic}
      \State {\it // Embeddings of $B$ input sequences}
      \State {\it // with $a$th sequence having length $T_a$}
      \Require $\{U_\zz x^{ta}: \Gtype(n)\}_{a \in [B], t \in [T_a]}$
      \Require $\{U_\rr x^{ta}: \Gtype(n)\}_{a \in [B], t \in [T_a]}$
      \Require $\{U_\hh x^{ta}: \Gtype(n)\}_{a \in [B], t \in [T_a]}$
      \State {\it // Parameters}
      \Require $W_\zz, W_\rr, W_\hh: \Atype(n, n)$
      \Require $b_\zz, b_\rr, b_\hh: \Gtype(n)$
      \State {\it // Initial GRU state}
      \Require $h^0: \Gtype(n)$
      \State {\it // Readout layer}
      \Require $v: \Gtype(n)$
      \For{$a \in [B]$}
        \For{$t \in [T_a]$}
          \State $h_\zz^{ta} := W_\zz h^{t-1,a}: \Gtype(n)$
          \State $\tilde z^{ta} := h_\zz^{ta} + U_\zz x^{ta} + b_\zz: \Gtype(n)$
          \State $h_\rr^{ta} := W_\rr h^{t-1,a}: \Gtype(n)$
          \State $\tilde r^{ta} := h_\rr^{ta} + U_\rr x^{ta} + b_\rr: \Gtype(n)$
          \State {\it // Morally, $\hat h^{t-1,a} = \sigma(\tilde r^{t-1,a}) \odot h^{t-1,a}$, but we need to unroll $h^{t-1,a}$ to apply \ref{linetype:nonlin}}
          \State $\hat h^{t-1, a} := \sigma(\tilde r^{t-1, a}) \odot$
          \State $\quad
            \bigg(
              h^0 \odot \bigodot_{i=1}^{t-1} (1 - \sigma(\tilde z^{ia}))
              + \sum_{j=1}^{t-1} \phi(\tilde h^{ja}) \odot \sigma(\tilde z^{ja}) \odot \bigodot_{l=j+1}^{t-1} (1- \sigma(\tilde z^{la}))
            \bigg): \Htype(n)$
          \State $h_\hh^{ta} := W_\hh \hat h^{t-1,a}: \Gtype(n)$
          \State $\tilde h^{ta} := h_\hh^{ta} + U_\hh x^{ta} + b_\hh: \Gtype(n)$
          \State
          {\it // Unrolling $h^t$ to a coordinatewise function of G-vars}
          \State  $h^t := h^0 \odot \bigodot_{i=1}^t (1 - \sigma(\tilde z^{ia}))
          + \sum_{j=1}^t \phi(\tilde h^{ja}) \odot \sigma(\tilde z^{ja}) \odot \bigodot_{l=j+1}^t (1- \sigma(\tilde z^{la})): \Htype(n)$
        \EndFor
      \EndFor  
      \Ensure $\{v^\trsp h^{ta}/\sqrt{n}\}_{a \in [B], t \in [T_a]}$
    \end{algorithmic}
\end{algorithm}

To apply \cref{lemma:erfTrick}, we can express \cref{eqn:highdimSigmoidExpectation} as
\begin{align*}
    &\EV 
    \left[\sigma(Z^{\tilde z^i}) \prod_{p = i+1}^t \sigma(-Z^{\tilde z^{p}})\right]
    \times \left[
    \sigma(Z^{\tilde z^j}) \prod_{q = j+1}^s \sigma(-Z^{\tilde z^{q}})\right]
        \\
    &=
        \EV 
        \left[\sigma(\hat Y^i ) \prod_{p = i+1}^s \sigma(\hat Y^{p})\right]
        \times \left[
        \sigma(\check Y^j) \prod_{q = j+1}^s \sigma(\check Y^{q})\right]
        \\
    &=
        \EV 
        \left[\prod_{p = i}^t \sigma(\hat Y^{p})\right]
        \times \left[
        \prod_{q = j}^s \sigma(\check Y^{q})\right]
\end{align*}
where $(\hat Y^i, \ldots, \hat Y^t, \check Y^j, \ldots, \check Y^s) \sim \Gaus(\nu, \Omega)$ with $b$ and $\Omega$ given as follows
\begin{align*}
    \nu(\hat Y^i) &= \tmu(\tilde z^i)&
    \nu(\check Y^j) &= \tmu(\tilde z^j)\\
    \nu(\hat Y^{p}) &= -\tmu(\tilde z^{p}), \forall p \ge i + 1&
    \nu(\check Y^{q}) &= -\tmu(\tilde z^{q}), \forall q \ge j + 1\\
\end{align*}
\begin{align*}
    \Omega(\hat Y^{p}, \check Y^{q}) &=
    \begin{cases}
        \tSigma(\tilde z^{p}, \tilde z^{q})
            &   \text{if $p \ge i + 1, q \ge j + 1$, or
                        $p = i, q = j$}\\
        -\tSigma(\tilde z^{p}, \tilde z^{q})
            &   \text{otherwise.}
    \end{cases}\\
    \Omega(\hat Y^{p}, \hat Y^{p'}) &=
    \begin{cases}
        \tSigma(\tilde z^{p}, \tilde z^{p'})
            &   \text{if $p, p' \ge i + 1$, or
                        $p = p' = i$}\\
        -\tSigma(\tilde z^{p}, \tilde z^{p'})
            &   \text{otherwise.}
    \end{cases}\\
    \Omega(\check Y^{q}, \check Y^{q'}) &=
    \begin{cases}
        \tSigma(\tilde z^{q}, \tilde z^{q'})
            &   \text{if $q, q' \ge j + 1$, or
                        $q = q' = j$}\\
        -\tSigma(\tilde z^{q}, \tilde z^{q'})
            &   \text{otherwise.}
    \end{cases}
\end{align*}

Using \cref{lemma:erfTrick}, one then has
\begin{align*}
    &\EV 
    \left[\sigma(Z^{\tilde z^i}) \prod_{p = i+1}^t \sigma(-Z^{\tilde z^{p}})\right]
    \times \left[
    \sigma(Z^{\tilde z^j}) \prod_{q = j+1}^s \sigma(-Z^{\tilde z^{q}})\right]
        \\
    &=
        \Pr\left[ X \ge 0 : X \sim \Gaus\lp \nu, \f 1 2 I + \Omega\rp\right]
        .
        \numberthis \label{eqn:GRUintegralSimplify}
\end{align*}

If we two input sequences, the equations for recursively computing $\tSigma$ and of $\KK$ are similar to the above, and we summarize them below
\begin{tcolorbox}[title=Computing the GRU kernel]
Consider a GRU processing $B$ sequences in the fashion of \cref{tp:GRU2}, with $\phi = \erf$ and $\sigma = (1 + \erf)/2$.
Sample the GRU's parameters as in \cref{eqn:GRUinit}.
Then for sequence numbers $a,b  \in [B]$ and time steps $2 \le t \le T_a, 2 \le s \le T_b,$
we have the following recurrence relations
\begin{align*}
\tSigma(\tilde z^{ta}, \tilde z^{sb})
    &=
        \tSigma(h_\zz^{ta}, h_\zz^{sb}) + \sigma_{U}^2 x^{ta}{}^\trsp x^{sb} /m + \sigma_b^2
        \\
\tSigma(\tilde r^{ta}, \tilde r^{sb})
    &=
        \tSigma(h_\rr^{ta}, h_\rr^{sb}) + \sigma_U^2 x^{ta}{}^\trsp x^{sb} /m + \sigma_b^2
        \\
\tSigma(\tilde h^{ta}, \tilde h^{sb})
    &=
        \tSigma(h_\hh^{ta}, h_\hh^{sb}) + \sigma_U^2 x^{ta}{}^\trsp x^{sb} /m + \sigma_b^2
        \\
\tSigma(h_\zz^{ta}, h_\zz^{sb})
    &= \tSigma(h_\rr^{ta}, h_\rr^{sb})\\
    &=
        \sigma_W^2 \sum_{i=1}^t \sum_{j=1}^s 
        \zeta_{i:t,j:s}^{ab}
        \EV \phi(Z^{\tilde h^{ia}})\phi(Z^{\tilde h^{jb}})
          \\
\tSigma(h_\hh^{ta}, h_\hh^{sb})
    &=
        \sigma_W^2 \tSigma(h_\zz^{ta}, h_\zz^{sb}) \EV \sigma(Z^{\tilde r^{ta}})\sigma(Z^{\tilde r^{sb}})
\end{align*}
with initial conditions
\begin{align*}
\tSigma(\tilde z^{ta}, \tilde z^{sb})
=\tSigma(\tilde r^{ta}, \tilde r^{sb})
=\tSigma(\tilde h^{ta}, \tilde h^{sb})
    &= 0
        &\text{if $t=1$ or $s=1$}
        ,
\end{align*}
and the output covariance $\KK$ of the GRU outputs in the large $n$ limit can be computed as
\begin{align*}
\KK_{ta, sb}
    &=
        \lim_{n \to \infty}
        \Cov\lp \f{v^\trsp x^{ta}}{\sqrt n},
                \f{v^\trsp x^{sb}}{\sqrt n}
            \rp
        \\
    &=
        \sigma_v^2 \sum_{i=1}^t \sum_{j=1}^s
\zeta_{i:t,j:s}^{ab}
        \EV \phi(Z^{\tilde h^{ia}})\phi(Z^{\tilde h^{jb}})
\end{align*}
where $Z \sim \Gaus(0, \tSigma)$ and
\begin{align*}
\zeta_{i:t,j:s}^{ab}
    &\defeq
        \EV_Z
        \left[\sigma(Z^{\tilde z^{ia}}) \prod_{p = i+1}^t (1 - \sigma(Z^{\tilde z^{p a}}))\right]
        \times \left[
        \sigma(Z^{\tilde z^{jb}}) \prod_{q = j+1}^s (1 - \sigma(Z^{\tilde z^{qb}}))\right]
        ,
\end{align*}
which can be reduced to a computation of orthant probability in the fashion of \cref{eqn:GRUintegralSimplify}.
\end{tcolorbox}

The above equations can be turned into a
(relatively) efficient algorithm for computing the GP kernel of a GRU.
Our repo \repo{} shows a reference implementation of it (allowing slightly more general initialization hyperparameters).
It leverages the \texttt{R} package \texttt{mvtnorm} \citep{mvtnorm} to evaluate the Gaussian orthant probability involved in \cref{eqn:GRUintegralSimplify}.

In the rest of the section, we prove \cref{lemma:erfTrick}.

\newcommand{\Fourier}{\mathcal{F}}
\newcommand{\pv}{\mathrm{p.v.}}

\paragraph{Review of (Tempered) Distributions}
Before we begin the proof of \cref{lemma:erfTrick}, we briefly recall the notion of a tempered distribution, which is a ``pseudo-function'' that is formally defined as an element of the dual of Schwartz space (intuitively, the space of functions with rapidly decreasing derivatives of all orders) \citep{stein_shakarchi_2011}.
Given a Schwartz function $f$ and a tempered distribution $\tau$, the value of $\tau$ on $f$ will be denoted here by
\begin{align*}
    \la \tau, f \ra.
\end{align*}
For example, if $\tau$ is a locally-integrable function, then $\tau$ is also a tempered distribution and $\la \tau, f\ra$ can be defined by
\begin{align*}
    \la \tau, f\ra = \int \tau(x) f(x) \dd x
    .
\end{align*}
As all Schwartz functions have Fourier transforms \citep{stein_shakarchi_2011}, any tempered distribution has Fourier transform defined by
\begin{align*}
    \la \Fourier\{\tau\}, f \ra \defeq \la \tau, \Fourier\{f\}\ra.
\end{align*}
In what follows, notationally, Fourier transform will convert functions or distributions in variable $t$ to functions to distributions in variable $x$, or vice versa.
See \cite{stein_shakarchi_2011} for more background on distributions.

\begin{proof}[Proof of \cref{lemma:erfTrick}]

As a tempered distribution, $\phi = \erf$ can be expressed as
\begin{align*}
\phi(x) 
    &= \Fourier\left\{\f{-i\sqrt 2}{\sqrt\pi} \pv \f{e^{-t^2/4}}t\right\}(x)
        \\
    &= \f 1 {\sqrt{2\pi}} \pv \int e^{ixt}\f{-i\sqrt 2}{\sqrt\pi} \f{e^{-t^2/4}}t \dd t
    ,
\end{align*}
where $\pv$ denotes principal value integration
\begin{align*}
    \pv \int e^{ixt}\f{-i\sqrt 2}{\sqrt\pi} \f{e^{-t^2/4}}t \dd t
    &\defeq
    \lim_{\varepsilon \to 0}
        \lp\int_{-\infty}^{-\varepsilon} + \int_{\varepsilon}^\infty\rp e^{ixt}\f{-i\sqrt 2}{\sqrt\pi} \f{e^{-t^2/4}}t \dd t
    .
\end{align*}

Over multiple variables, because Fourier transform over $\R^T$ is equivalent to applying 1D Fourier transform over each coordinate, we have
\begin{align*}
\prod_{i=1}^T \phi(x_i)
    = \Fourier\left\{\left(\f{-i\sqrt 2}{\sqrt\pi}\right)^T \pv \f{e^{-\sum_{i=1}^T t_i^2/4}}{\prod_{i=1}^T t_i}\right\}(x)
    = \f 1 {(2\pi)^{T/2}} \pv \int e^{ix \cdot t }\lp\f{-i\sqrt 2}{\sqrt\pi}\rp^T \f{e^{-\sum_{i=1}^T t_i^2/4}}{\prod_{i=1}^T t_i} \dd t
    .
\end{align*}

Let $\gamma(x; \Sigma) \defeq (\det 2\pi \Sigma)^{-T/2} e^{-\f 1 2 x^\trsp \inv\Sigma x}$ be the density of $\Gaus(0, \Sigma)$ for \emph{nonsingular} $\Sigma$.
Note that $\Fourier\{\gamma(x; \Sigma)\}(t) = (2\pi)^{-T/2} e^{-\f 1 2 t^\trsp \Sigma t}$.
We thus have
\begingroup
\allowdisplaybreaks
\begin{align*}
    &\EV\left[\prod_{i=1}^T \phi(x_i): x \sim \Gaus(0, \Sigma)\right]\\
    &=
        \left\la
        \Fourier\left\{\lp\f{-i\sqrt 2}{\sqrt\pi}\rp^T \pv \f{e^{-\sum_{i=1}^T t_i^2/4}}{\prod_{i=1}^T t_i}\right\},\ 
        \gamma(x; \Sigma)
        \right\ra\\
    &=
        \left\la
        \lp\f{-i\sqrt 2}{\sqrt\pi}\rp^T \pv \f{e^{-\sum_{i=1}^T t_i^2/4}}{\prod_{i=1}^T t_i},\ 
        \Fourier\{\gamma(x; \Sigma)\}
        \right\ra\\
    &=
        \left\la
        \lp \f{-i\sqrt 2}{\sqrt\pi}\rp^T \pv \f{e^{-\sum_{i=1}^T t_i^2/4}}{\prod_{i=1}^T t_i},\ 
        (2\pi)^{-T/2} e^{-\f 1 2 t^\trsp \Sigma t}
        \right\ra\\
    &=
        \left(\f{-i\sqrt 2}{\sqrt\pi}\right)^T 
        \pv\int
            \f{e^{-\sum_{i=1}^T t_i^2/4}}{\prod_{i=1}^T t_i}
            (2\pi)^{-T/2} e^{-\f 1 2 t^\trsp \Sigma t}
            \dd t\\
    &=
        \left(\f{-i\sqrt 2}{\sqrt\pi}\right)^T 
        \pv\int
            \f{1}{\prod_{i=1}^T t_i}
            (2\pi)^{-T/2} e^{-\f 1 2 t^\trsp (\Sigma + \f 1 2 I) t}
            \dd t\\
    &=
        \left\la
            \left(\f{-i\sqrt 2}{\sqrt\pi}\right)^T \pv \f{1}{\prod_{i=1}^T t_i},\ 
        (2\pi)^{-T/2} e^{-\f 1 2 t^\trsp (\Sigma + \f 1 2 I) t}
        \right\ra\\
    &=
        \left\la
            \left(\f{-i\sqrt 2}{\sqrt\pi}\right)^T \pv \f{1}{\prod_{i=1}^T t_i},\ 
        \Fourier\{\gamma(x; \Sigma + \f 1 2 I)\}
        \right\ra\\
    &=
        \left\la 
            \Fourier\left\{\left(\f{-i\sqrt 2}{\sqrt\pi}\right)^T \pv \f{1}{\prod_{i=1}^T t_i}\right\},\ 
        \gamma(x; \Sigma + \f 1 2 I)
        \right\ra\\
    &=
        \left\la 
        \left(\f{-i\sqrt 2}{\sqrt\pi}\right)^T \lp i \sqrt{\pi/2}\rp^T \prod_{i=1}^T \sgn(x_i),\ 
        \gamma(x; \Sigma + \f 1 2 I)
        \right\ra\\
    &=
        \EV\left[\prod_{i=1}^T \sgn(x_i): x \sim \Gaus(0, \Sigma + \f 1 2 I)\right]
\end{align*}
\endgroup
where we used $\Fourier\{\pv \inv t\}(x) = i \sqrt{\pi/2} \sgn(x)$.
Similar reasoning show that this formula also works when the mean is nonzero:
\begin{align*}
    \EV\left[\prod_{i=1}^T \phi(x_i): x \sim \Gaus(\mu, \Sigma)\right]
    =
        \EV\left[\prod_{i=1}^T \sgn(x_i): x \sim \Gaus(\mu, \Sigma + \f 1 2 I)\right].
\end{align*}
A standard continuity argument yields the same formula for singular $\Sigma$.
Some simple arithmetic reduces the $\sigma$ case to $\phi$.
\begin{align*}
    2^{-T}\EV\left[\prod_{i=1}^T (1+\phi(x_i)): x \sim \Gaus(\mu, \Sigma)\right]
    =
        \EV\left[\prod_{i=1}^T \ind(x_i \ge 0): x \sim \Gaus(\mu, \Sigma + \f 1 2 I)\right]
        .
\end{align*}

\end{proof}

\section{\texorpdfstring{\netsorplus}{Netsor+} Master Theorem}
\label{sec:netsorplusMasterTheorem}

In this section, we state the Master Theorem for \netsorplus.
Its proof can be found in \cref{sec:netsorplusMasterTheoremProof}.

We first need to extend the notion of \emph{controlled functions} (\cref{defn:controlled}) to functions with parameters, and additionally require a smoothness assumption.
\begin{defn}\label{defn:parameterControlled}
We say a parametrized function $\phi(-; -): \R^k \times \R^l \to \R$ is \emph{parameter-controlled} at $\mathring {\bigtheta} \in \R^l $ if
\begin{enumerate}
    \item $\phi(-; \mathring{\bigtheta})$ is controlled, and\label{item:parameterControlled1}
    \item there are some controlled $\bar \phi:\R^k \to \R$ and some function $f: \R^l \to \R^{\ge 0} \cup \{\infty\}$ that has $f(\mathring{\bigtheta}) = 0$ and that is continuous at $\mathring \bigtheta$, such that, for all $x^1, \ldots, x^k \in \R$ and $\bigtheta \in \R^l$,
\[
|\phi(x^1, \ldots, x^k; \bigtheta) - \phi(x^1, \ldots, x^k; \mathring{\bigtheta})|
\le
f(\bigtheta) \bar \phi(x^1, \ldots, x^k)
.
\]
\label{item:parameterControlled2}
\end{enumerate}
\end{defn}
Note that $f$ and $\bar\phi$ here can depend on $\mathring{\bigtheta}$.
\begin{exmp}\label{exmp:parameterControl}
Any function that is (pseudo-)Lipschitz\footnote{A pseudo-Lipschitz function $\phi: \R^r \to \R$ is one that satisfies
\[
|\phi(x) - \phi(y)| \le C\|x - y\|(\|x\|^p + \|y\|^q + 1)
\]
for some constants $C, p, q \ge 0$.
Roughly speaking, pseudo-Lipschitz functions are those that have polynomially bounded weak derivatives.
}
in $x^1, \ldots, x^k$ and $\bigtheta$ is parameter-controlled.
An example of a discontinuous function that is parameter-controlled is $\phi(x; \theta) = \mathrm{step}(\theta x)$.
Then for $\mathring \theta \ne 0$, 
\[
|\phi(x; \theta) - \phi(x; \mathring \theta)| \le
\f{|\mathring \theta - \theta|}{|\mathring \theta|}
,
\]
so we can set $f(\theta) = \f{|\mathring \theta - \theta|}{|\mathring \theta|}$ and $\bar \phi = 1$ in \cref{defn:parameterControlled}.
\end{exmp}
\begin{assm}[Rank Stability]\label{assm:asRankStab}
For any $W: \Atype(n, m)$ and any collection $\Ss \sbe \{(h: \Htype(m)) \mid \exists (g: \Gtype(n)), g := W h\}$,
let $H \in \R^{m \times |\Ss|}$ be the matrix whose columns are $h \in \Ss$.
If $\f 1 m H^\trsp H \in \R^{|\Ss| \times |\Ss|}$ converges almost surely to some $\mathring C$ as $n, m \to \infty$ with convergent ratio $n/m \to \alpha$, then almost surely $\rank H = \rank \mathring C$ for all large $n$ and $m$.
\end{assm}
Note that a common situation where rank stability holds is when all limit $\mathring C$ matrices are full rank.
By the lower semi-continuity of rank, $\rank H = \rank \mathring C$ must hold asymptotically.
\begin{thm}\label{thm:Netsor+MasterTheorem}
Fix any \netsorplus program satisfying \cref{assm:equalDimNoNonlin+} and \cref{assm:asRankStab}.
Suppose for each parametrized nonlinearity $\phi(-; \bigvtheta)$ in the program (appearing as part of \ref{linetype:nonlin+}), the parameters $\bigvtheta$ are instantiated with random variables that converge almost surely to some deterministic vector $\mathring{\bigvtheta}$ as $n \to \infty$, and assume $\phi$ is parameter-controlled at $\mathring{\bigvtheta}$.
If $g^1, \ldots, g^M$ are all of the G-vars (including all input G-vars), then for any $l$, for any random vector $\bigtheta \in \R^l$ that converges almost surely to a deterministic vector $\mathring{\bigtheta}$, as $n \to \infty$, and for any $\psi: \R^M\times \R^l \to \R$ parameter-controlled at $\mathring{\bigtheta}$,
\begin{align*}
    \f 1 n \sum_{\alpha=1}^n \psi(g^1_\alpha, \ldots, g^M_\alpha; \bigtheta) \asto \EV_{Z \sim \Gaus(\tmu, \tSigma)}\psi(Z; \mathring{\bigtheta}),
    \numberthis\label{eqn:netsor+momentconv}
\end{align*}
where $\asto$ means almost sure convergence,
$Z \in \R^M$, and $\tmu \in \R^M$ and $\tSigma \in \R^{M \times M}$ are given in \cref{eqn:extendedMuSigma}, calculated by replacing each parametrized $\phi(-; \bigvtheta)$ with parameterless nonlinearity $\phi(-; \mathring{\bigvtheta})$.
\end{thm}
The proof of this theorem can be found in \cref{sec:netsorplusMasterTheoremProof}.

We will be instantiating the coordinates of $\bigtheta$ typically with ``empirical moments''
\begin{align*}
    \f 1 n \sum_{\alpha=1}^n \psi(g^1_\alpha, \ldots, g^M_\alpha)
        \numberthis\label{eqn:empiricalMomentParams}
\end{align*}
for some controlled $\psi$, since such ``moments'' should converge to a deterministic value by \cref{thm:Netsor+MasterTheorem};
or even, recursively, \begin{align*}
    \f 1 n \sum_{\alpha=1}^n \psi(g^1_\alpha, \ldots, g^M_\alpha; \bigvtheta')
    \numberthis\label{eqn:recursiveparams}
\end{align*}
for some sequence of random vectors $\bigvtheta'$ that converge a.s.\ to $\mathring{\bigvtheta}'$ and for $\psi$ parameter-controlled at $\mathring{\bigvtheta}'$.
One can keep recursing by replacing $\bigtheta'$ with further empirical moments.

However, there is a slight complication: we are using \cref{thm:Netsor+MasterTheorem} both for the convergence of the  parameters in \ref{linetype:nonlin+} rules in the program, as well as for the convergence \cref{eqn:netsor+momentconv}, in what seems like could be circular logic.
It turns out not hard to straighten out this reasoning, but it requires a bit more notation and setup to state the result.
We do this in the next section \cref{sec:selfParam}, with main theorem \cref{thm:selfParamNetsorplusMasterTheorem} that will be our primary tool concerning \netsorplus programs in practice.

To finish up this section, we make several remarks on the assumptions made in \cref{thm:Netsor+MasterTheorem}.

\begin{remk}[Necessity of parameter-control]
Suppose $\psi(x; \theta) = \ind(\theta x \ne 0)$.
For $\theta \ne 0$, $\psi$ is 1 everywhere except $\psi(0; \theta) = 0$.
For $\theta = 0$, $\psi$ is identically 0.
Thus it's easily seen that $\psi$ is not parameter-controlled at $\theta = 0$.

Now, if $g: \Gtype(n)$ is sampled like $g_\alpha \sim \Gaus(0,1)$, then
\[
\f 1 n \sum_{\alpha=1}^n \psi(g_\alpha; \theta) \asto 1
\]
if $\theta = 1/n$ so that $\theta \to \mathring \theta = 0$, but
\[
\EV_{Z \sim \Gaus(\tmu, \tSigma)} \psi(Z; \mathring \theta)
= \EV 0 = 0.
\]
So our Master Theorem can't hold in this case.
\end{remk}

\begin{remk}[Necessity of Rank Stability Assumption \cref{assm:asRankStab}]
\label{remk:necessityRankStab}
Suppose we have two input G-vars $g^1, g^2: \Gtype(n)$ which are sampled independently as $g^1_\alpha, g^2_\alpha \sim \Gaus(0, 1)$.
Let $W: \Atype(n, n)$ be sampled as $W_{\alpha \beta} \sim \Gaus(0, 1/n)$.
Then we can define $h^2 := \theta g^2 : \Htype(n)$ where $\theta = \exp(-n)$ as a function of $n$, using \ref{linetype:nonlin+}, so that $h^2_\alpha \asto 0$.
Additionally, let $\bar g^1 := W g^1: \Gtype(n)$ and $\bar g^2 := W h^2: \Gtype(n)$.
Again, $\bar g^2_\alpha \asto 0$ but for any finite $n$, $\bar g^2$ is linearly independent from $\bar g^1$.
Thus rank stability does not hold here.

Now consider the (parameterless) nonlinearity $\psi(x, y)$ that is 1 except on the line $y = 0$, where it is 0.
Then
\[
\f 1 n \sum_{\alpha=1}^n \psi(\bar g^1_\alpha, \bar g^2_\alpha) \asto 1
\]
but
\[
\EV_{Z \sim \Gaus(\tmu, \tSigma)} \psi(Z^{\bar g^1}, Z^{\bar g^2})
= \EV 0 = 0.
\]

\end{remk}

\begin{remk}[Rank Stability Already Holds for \netsor Programs]
\label{remk:rankStabilityNetsor}
It turns out that, as long as we only have parameterless nonlinearities, we get rank stability \cref{assm:asRankStab} for free.
This is formulated explicitly in \cref{lemma:rankStability}.
It is as a result of our proof of \cref{thm:netsorMasterTheorem} that interleaves an inductive proof of this rank stability (more generally, the inductive hypothesis \ref{IH:coreSet}) with an inductive proof of the ``empirical moment'' convergence (the inductive hypothesis \ref{IH:MomConv}).
\end{remk}

\subsection{Self-Parametrized \texorpdfstring{\netsorplus}{Netsor+} Programs and Their Master Theorem}
\label{sec:selfParam}

As stated below \cref{thm:Netsor+MasterTheorem}, there could be potentially circular logic when allowing \ref{linetype:nonlin+} rules to take parameters depending on previously defined variables in the program, such as in the form of \cref{eqn:empiricalMomentParams}.
In this section, we untangle this potentially circular logic into a sound reasoning 
\begin{itemize}
\item
by introducing a scalar type into \netsorplus programs to explicitly extract the \ref{linetype:nonlin+} parameters into their own variables (\cref{defn:selfParam}).
These scalar variables can recursively depend on previously defined scalar variables, making the ``recursive parameters'' discussed in \cref{eqn:recursiveparams} much more succinctly and clearly expressed.
\item
and by proving a Master Theorem for such \netsorplus programs (\cref{thm:selfParamNetsorplusMasterTheorem}).
This theorem will be the primary way through which we analyze \netsorplus programs in practice.
\end{itemize}

\newcommand{\Ctype}{\mathsf{C}}
\begin{defn}\label{defn:selfParam}
\footnote{We keep the definition here informal in terms of programming language convention to be accessible to the general machine learning audience. For those with PL background, see \cref{sec:formalspec}.}
A \emph{self-parametrized \netsorplus program} is a \netsor program where we have an additional scalar type, called $\Ctype$, which should intuitively be thought of as random variables that tend to a deterministic limit (i.e. a \emph{$\Ctype$onstant}) almost surely.
Colloquially, we will call variables of type $\Ctype$ ``C-vars.''
C-vars can be used as parameters of nonlinearities in \ref{linetype:nonlin+} rules, hence the \emph{``self-parametrized''} in the name.

For completeness, we specify a self-parametrized \netsorplus program as follows:
\begin{description}
    \item[Input]
        A set of input C-vars, in addition to the G- and A-vars allowed in \cref{defn:netsor}.
    \item[Body]
        New variables can be introduced and assigned via the following rules
    \begin{description}
        \item[\texttt{MatMul}] Same as in \cref{defn:netsor}.
        \item[\texttt{LinComb}] Same as in \cref{defn:netsor}.
        \item[\texttt{Nonlin$^+$}\label{linetype:nonlin+scalar}] If $x^1, \ldots, x^k: \Gtype(n)$ are G-vars with the same dimension $n$, $\theta_1, \ldots, \theta_l: \Ctype$ are C-vars, and $\phi(-; -): \R^k \times \R^l \to \R$ is a parametrized function, then we may create an H-var
        \[\phi(x^1, \ldots, x^k; \theta_1, \ldots, \theta_l): \Htype(n)\]
        where $\phi(- ; \theta_1, \ldots, \theta_l)$ acts coordinatewise.
        \item[\texttt{Moment}\label{linetype:moment}]
        If $x^1, \ldots, x^k: \Gtype(n)$ are G-vars with the same dimension $n$, $\theta_1, \ldots, \theta_l: \Ctype$ are C-vars, and $\phi(-; -): \R^k \times \R^l \to \R$ is a parametrized function, then we may create a C-var
        \[\f 1 n \sum_{\alpha=1}^n\phi(x^1_\alpha, \ldots, x^k_\alpha; \theta_1, \ldots, \theta_l): \Ctype.\]
    \end{description}
    \item[Output]
        Same as in \cref{defn:netsor}.
\end{description}
\end{defn}

The self-parametrized \netsorplus programs we are concerned will have all of its C-vars convergent to a deterministic constant.
We thus need this to be true for the input C-vars at the very least.
We encapsulate this requirement below.

\begin{assm}\label{assm:netsorplusScalarLimit}
Fix a self-parametrized \netsorplus program satisfying \cref{assm:equalDimNoNonlin+}.
Assume each input C-var $\theta$ is sampled in a way such that $\theta \asto \mathring \theta$ as $n \to \infty$ for some deterministic scalar $\mathring \theta \in \R$.
\end{assm}

Now, we shall define $\tmu$ and $\tSigma$ for self-parametrized \netsorplus programs just as in \netsorplus programs.
The only complication here is that we also need to keep track of the limit values of the C-vars in order to do so.
See \cref{defn:netsorplusMuSigma} below.

\begin{defn}\label{defn:netsorplusMuSigma}
Fix a \netsorplus program with scalar variables satisfying \cref{assm:netsorplusScalarLimit}.
For the purpose of this definition, write $g^1, \ldots, g^M$ for the entirety of the G-vars in the program, including input G-vars.

\emph{New Notations}\quad
For each H-var $h$ introduced by \ref{linetype:nonlin+}, \emph{we introduce the notations $\varphi^h, \bigtheta^h, \vartheta^h_i, \ell^h$ as follows}:
denote the associated parametrized nonlinearity by $\varphi^h(-; -): \R^M \times \R^{\ell^h} \to \R$ (implicitly padded so that it has as many input slots as G-vars in the program) and its parameters by $\bigvtheta^h = (\vartheta_1^h, \ldots, \vartheta_{\ell^h}^h) \in \R^{\ell^h}$ with length $\ell^h$.
For each G-var $g^i$, we also set $\varphi^{g^i}(x^1, \ldots, x^M) = x^i$ and $\bigvtheta^{g^i} = () \in \R^0$ to be the empty vector (so that $\ell^{g^i} = 0$).

Likewise, for each C-var $\theta$ introduced by \ref{linetype:moment}, \emph{we introduce the notations $\varphi^\theta, \bigtheta^\theta, \vartheta^\theta_i, \ell^\theta$ as follows}: denote the associated parametrized nonlinearity by $\varphi^\theta(-;-): \R^M \times \R^{\ell^\theta} \to \R$ (implicitly padded so that it has as many input slots as G-vars in the program) and its parameters by $\bigvtheta^\theta = (\vartheta_1^\theta, \ldots, \vartheta_{\ell^\theta}^\theta) \in \R^{\ell^\theta}$ with length $\ell^\theta$.

\emph{Extending the $\mathring{(\phantom{\theta})}$ notation from \cref{assm:netsorplusScalarLimit} and the
Recursive Definition of $\tmu$ and $\tSigma$}\quad
Given $\muin$ and $\Sigmain$ as in \cref{assm:equalDimNoNonlin+}, we define $\mu$ and $\Sigma$ on G-vars, along with ``limit scalars'' $\mathring{\theta}$ for each C-var $\theta$ (extending $\mathring \theta$ given by \cref{assm:netsorplusScalarLimit} for input $\theta$), as follows:
For any pair of G-vars $g, g'$ (among $g^1, \ldots, g^M$), we define recursively
\begin{align*}
    \tmu(g)
        &\defeq
            \begin{cases}
            \muin(g)  &   \text{if $g$ is input}\\
            \sum_{i} a_i \tmu(y^i)    &   \text{if $g = \sum_{i} a_i y^i$, (\ref{linetype:lincomb})}\\
            0   &   \text{otherwise}
            \end{cases}
            \\
    \tSigma(g, g')
        &\defeq
            \begin{cases}
            \Sigmain(g, g')   &   \text{if $g, g'$ are inputs}\\
            \sum_{i} a_i \tSigma(y^i, g')    &   \text{if $g = \sum_{i} a_i y^i$, (\ref{linetype:lincomb})}\\
            \sum_{i} a_i \tSigma(g, y^i)    &   \text{if $g' = \sum_{i} a_i y^i$, (\ref{linetype:lincomb})}\\
            \sigma^2_W \EV_Z \varphi^h(Z; \mathring{\bigvtheta}^h) \varphi^{h'}(Z; \mathring{\bigvtheta}^{h'}) &   \text{if $g = Wh, g'=Wh'$, (\ref{linetype:MatMul})}\\
            0   &   \text{otherwise}
            \end{cases}
            \numberthis\label{eqn:selfParamExtendedMuSigma}
\end{align*}
(this is the same as \cref{eqn:extendedMuSigma} except the \ref{linetype:MatMul} case)
and for each C-var $\theta$ introduced by \ref{linetype:moment},
\begin{align*}
\mathring \theta \defeq
    \EV_Z \varphi^\theta(Z; \mathring {\bigtheta}^\theta).
    \numberthis\label{eqn:mathringtheta}
\end{align*}
In all of the equations above, $Z \sim \Gaus(\tmu, \tSigma)$ is a random Gaussian vector with an entry for each G-var in the program, and $\mathring \Theta^u$ denotes $(\mathring \vartheta^u_1, \ldots, \mathring \vartheta^u_{\ell^u})$, for H-var or C-var $u$.

Note that since $\varphi^h$, $\varphi^{h'}$, and $\varphi^\theta$ only depend on entries of $Z$ corresponding to G-vars previous to $h, h'$, or $\theta$, the expectations involving $Z$ only depend on entries of $\tmu$ and $\tSigma$ already defined, so there is no circular logic in this recursive definition of $\tmu$ and $\tSigma$.
\end{defn}

Note that the notation $\varphi^h$ will be overloaded in a semantically consistent way in the context of \netsoro programs; see \cref{defn:unwindedNonlin}.

We are finally ready to formulate the Master Theorem for self-parametrized \netsorplus programs, which basically is just \cref{thm:Netsor+MasterTheorem} but explicitly allowing parameters of the form \cref{eqn:empiricalMomentParams} (``empirical moments'') in \ref{linetype:nonlin+}.

\begin{thm}[Self-Parameterized \netsorplus Master Theorem]
\label{thm:selfParamNetsorplusMasterTheorem}
Fix any self-parametrized \netsorplus program satisfying \cref{assm:netsorplusScalarLimit} and \cref{assm:asRankStab}.
For H-var or C-var $u$, adopt the notation $\varphi^u$, $\bigvtheta^u$, $\ell^u$ from \cref{defn:netsorplusMuSigma} and also let $\tmu, \tSigma, \mathring{\theta}$ be as computed in \cref{defn:netsorplusMuSigma}.
Let $g^1, \ldots, g^M$ be all of the G-vars in the program (including all input G-vars).

Suppose for every H-var or C-var $u$, $\varphi^u(-; \bigvtheta^u)$ is parameter-controlled at $\mathring{\bigvtheta}^u$.
\begin{enumerate}
\item 
Then for any $l$, for any random vector $\bigtheta \in \R^l$ that converges almost surely to a deterministic vector $\mathring{\bigtheta}$, as $n \to \infty$, and for any $\psi(-; -): \R^M\times \R^l \to \R$ parameter-controlled at $\mathring{\bigtheta}$,
\begin{align*}
    \f 1 n \sum_{\alpha=1}^n \psi(g^1_\alpha, \ldots, g^M_\alpha; \bigtheta) \asto \EV_{Z \sim \Gaus(\tmu, \tSigma)}\psi(Z; \mathring{\bigtheta}),
\end{align*}
where $\asto$ means almost sure convergence.

\item
In addition, for each C-var $\theta$ in the program,
\begin{align*}
\theta \asto \mathring \theta.
\end{align*}

\end{enumerate}

\end{thm}

This theorem almost trivially follows from \cref{thm:Netsor+MasterTheorem}, since the parameter vectors $\bigvtheta^h$ intuitively should converge to deterministic limits $\mathring{\bigvtheta}^h$.
The only slight complication is that this convergence intuitvely follows from \cref{thm:Netsor+MasterTheorem} itself in what may be a circular logic, so we need to be slightly careful to unwind this logic into a valid inductive argument.
We do so below, assuming \cref{thm:Netsor+MasterTheorem} (which is proved in \cref{sec:netsorplusMasterTheoremProof}).

\begin{proof}
Notice that the 2nd claim about $\theta \asto \mathring \theta$ follows immediately from the 1st claim, so we will prove the 1st claim here.

Assume that the G-vars $g^1, \ldots, g^M$ are in order of appearance in the program, and that $g^1, \ldots, g^{m_0}$ (with $m_0 \le M$) are all of the input G-vars.
We perform simultaneous induction on two claims \ref{IH:MomentsNetsorplus}$(m)$ and \ref{IH:CVarLimits}$(m)$ in $m$, defined below
\begin{description}
\item[Moments\label{IH:MomentsNetsorplus}]\!\!\!$(m)$\ \ \ {\it
For any $l$, for any random vector $\bigtheta \in \R^l$ that converges almost surely to a deterministic vector $\mathring{\bigtheta}$ as $n \to \infty$, and for any $\psi(-;-): \R^m\times \R^l \to \R$ parameter-controlled at $\mathring{\bigtheta}$,
\begin{align*}
    \f 1 n \sum_{\alpha=1}^n \psi(g^1_\alpha, \ldots, g^m_\alpha; \bigtheta) \asto \EV_{Z \sim \Gaus(\tmu|_m, \tSigma|_m)}\psi(Z; \mathring{\bigtheta})
\end{align*}
where $\tmu|_m$ and $\tSigma|_m$ are the restriction of $\tmu$ and $\tSigma$ to $g^1, \ldots, g^m$.
}
\item[CVarLimits\label{IH:CVarLimits}]\!\!\!$(m)$\ \ \ {\it
For each C-var $\theta$ introduced before $g^m$,
\begin{align*}
\theta \asto \mathring \theta
\end{align*}
as $n \to \infty$, where $\mathring \theta$ is as computed in \cref{defn:netsorplusMuSigma}.
}
\end{description}

When $m = M$, we would have \cref{thm:selfParamNetsorplusMasterTheorem} by \ref{IH:MomentsNetsorplus}$(M)$.
\paragraph{Base case: $m = m_0$ (input G-vars only).}
\ref{IH:MomentsNetsorplus}$(m_0)$ trivially follows from \cref{thm:Netsor+MasterTheorem}.
\ref{IH:CVarLimits}$(m_0)$ follows from \cref{assm:netsorplusScalarLimit}.

Now suppose \ref{IH:MomentsNetsorplus}$(m)$ and \ref{IH:CVarLimits}$(m)$ are true; we aim to show \ref{IH:MomentsNetsorplus}$(m+1)$ and \ref{IH:CVarLimits}$(m+1)$.

\paragraph{Inductive case: \ref{IH:CVarLimits}$(m+1)$}
By \ref{IH:CVarLimits}$(m)$, it suffices to show $\theta \asto \mathring \theta$ for all $\theta$ introduced after $g^m$ but before $g^{m+1}$.
We do so by another induction (an \emph{inner induction}) in order of C-var appearance.

The \textbf{inner base case} is the first C-var $\theta$ introduced after $g^m$.
Its parameters $\Theta^\theta$ are among those introduced before $g^m$, so by induction hypothesis \ref{IH:CVarLimits}$(m)$,
\begin{align*}
\Theta^\theta \asto \mathring \Theta^\theta.
\end{align*}
By the assumption of \cref{thm:selfParamNetsorplusMasterTheorem} that $\varphi^\theta$ is parameter-controlled at $\mathring \Theta^\theta$, we have
\begin{align*}
\theta = \f 1 n \sum_{\alpha=1}^n \varphi^\theta(g^1_\alpha, \ldots, g^m_\alpha; \Theta^\theta) \asto \EV_{Z \sim \Gaus(\tmu|_m, \tSigma|_m)} \varphi^\theta(Z; \mathring \Theta^\theta) = \mathring \theta
\end{align*}
by induction hypothesis \ref{IH:MomentsNetsorplus}$(m)$ (where we have explicitly truncated the input slots of $\varphi^\theta$ to reflect its dependence only on $g^1, \ldots, g^m$).
The \textbf{inner inductive case}, for a later $\theta$, follows the same logic, once we assume the inner inductive hypothesis that each $\theta'$ introduced before $\theta$ has $\theta' \asto \mathring \theta'.$

\paragraph{Inductive case: \ref{IH:MomentsNetsorplus}$(m+1)$}
The claim is trivially true by \ref{IH:MomentsNetsorplus}$(m)$ if $g^{m+1}$ is introduced via \ref{linetype:lincomb}, so consider the case when $g^{m+1}$ is introduced via \ref{linetype:MatMul}
\begin{align*}
g^{m+1} := W h
\end{align*}
where $h: \Htype(n)$ is an H-var with associated nonlinearity $\varphi^h$ and parameters $\bigvtheta^h$ as defined in \cref{defn:netsorplusMuSigma}.
By the claim \ref{IH:CVarLimits}$(m+1)$ we proved above, $\bigtheta^h \asto \mathring \bigtheta^h$.
By the assumption of \cref{thm:selfParamNetsorplusMasterTheorem}, $\varphi^{h}$ is parameter-controlled at the parameter limit $\mathring \bigtheta^{h}$.
Thus, the subprogram up to and including the introduction of $g^{m+1}$ satisfies the assumptions of \cref{thm:Netsor+MasterTheorem}.
Consequently, \ref{IH:MomentsNetsorplus}$(m+1)$ is true by \cref{thm:Netsor+MasterTheorem}.

This completes the simultaneous induction of \ref{IH:MomentsNetsorplus} and \ref{IH:CVarLimits} and thus the proof of \cref{thm:selfParamNetsorplusMasterTheorem}.
\end{proof}

\subsection{Gaussian Process Behavior of \texorpdfstring{\netsorplus}{Netsor+} Programs}

We can generalize the Gaussian process behavior (\cref{cor:GPConv}) to cases involving \ref{linetype:nonlin+}:
\begin{cor}[Computing the GP Kernel for \netsorplus programs]\label{cor:GPConvplus}
Adopt the same assumptions and notations as in \cref{thm:Netsor+MasterTheorem}.
Suppose the program outputs $(v{}^\trsp x^1/\sqrt{n}, \ldots, v{}^\trsp x^k/\sqrt{n})$, where
\begin{itemize}
    \item $v: \Gtype(n), v_\alpha \sim \Gaus(0, \sigma_v^2),$ is an input G-var not used elsewhere in the program and is sampled independently from all other G-vars, and
    \item $x^i$ was introduced as $x^i := \phi^i(g^1, \ldots, g^M; \bigtheta^i)$ for parametrized nonlinearity $\phi^i$ and parameter vector $\bigtheta^i$ that converges a.s.\ to a deterministic vector $\mathring{\bigtheta}^i$ as $n \to \infty$.
    Assume $\phi^i$ is parameter-controlled at $\mathring{\bigtheta}^i$.
\end{itemize}
Then the output vector converges in distribution to $\Gaus(0, \KK)$ where
\begin{equation}
    \KK_{ij} = \sigma_v^2 \EV_{Z \sim \Gaus(\tmu, \tSigma)} \phi^i(Z; \mathring{\bigtheta}^i) \phi^j(Z; \mathring{\bigtheta}^j)
    \label{eqn:limitingCovarianceGPNetsor+}
\end{equation}
with $\tmu, \tSigma$ computed by replacing each parametrized $\phi(-; \bigvtheta)$ with the parameterless $\phi(-; \mathring{\bigvtheta})$ in \cref{eqn:extendedMuSigma}, as in \cref{thm:Netsor+MasterTheorem}.
\end{cor}

The proof is a straightforward application of \cref{thm:Netsoro+MasterTheorem} and \cref{prop:gaussianDistConvFromMomentConv}.
Likewise, for self-parametrized programs, we have a similar result:
\begin{cor}[Computing the GP Kernel for self-parametrized \netsorplus programs]\label{cor:GPConvPlusSelfParam}
Adopt the same assumptions and notations as in \cref{thm:selfParamNetsorplusMasterTheorem}.
Suppose the program outputs $(v{}^\trsp x^1/\sqrt{n}, \ldots, v{}^\trsp x^k/\sqrt{n})$, where
\begin{itemize}
    \item $v: \Gtype(n), v_\alpha \sim \Gaus(0, \sigma_v^2),$ is an input G-var not used elsewhere in the program and is sampled independently from all other G-vars, and
    \item $x^i$ was introduced as $x^i := \varphi^{x^i}(g^1, \ldots, g^M; \bigvtheta^{x^i})$ for self-parametrized nonlinearity $\varphi^{x^i}$ and parameter vector $\bigvtheta^{x^i}$ (composed of C-vars) as defined in \cref{defn:netsorplusMuSigma}.
    Let $\mathring{\bigvtheta}^{x^i}$ be the limit parameter as in \cref{defn:netsorplusMuSigma}.
    Note that $\varphi^{x^i}$ is parameter-controlled at $\mathring{\bigvtheta}^{x^i}$ by assumption of \cref{thm:selfParamNetsorplusMasterTheorem}.
\end{itemize}
Then the output vector converges in distribution to $\Gaus(0, \KK)$ where
\begin{equation}
    \KK_{ij} = \sigma_v^2 \EV_{Z \sim \Gaus(\tmu, \tSigma)} \varphi^{x^i}(Z; \mathring{\bigvtheta}^{x^i}) \varphi^{x^j}(Z; \mathring{\bigvtheta}^{x^j}),
    \quad \text{with $\tmu, \tSigma$ defined in \cref{eqn:selfParamExtendedMuSigma}.}
    \label{eqn:limitingCovarianceGPNetsorSelfParam+}
\end{equation}
\end{cor}

\section{Example GP Kernel Computation with \texorpdfstring{\netsorplus}{Netsor+}}
\label{sec:netsorplusKernelExample}

\subsection{Layernorm: Concrete Example \texorpdfstring{(\cref{tp:layernorm})}{}}
\label{sec:layernorm}
\newcommand{\mean}{\nu}
\newcommand{\var}{var}

{   
\makeatletter
\renewcommand{\ALG@name}{Self-parametrized \netsorplus program}
\makeatother
\begin{algorithm}[tb]
    \caption{Layernorm Network}
    \label{tp:layernorm}
    \begin{multicols}{2}
    \begin{algorithmic}
      \Require $W^1 x, W^1 x': \Gtype(n)$
      \Require $W^2: \Gtype(n)$
      \Require $v: \G(n)$
      \State {\it // Mean and variance of $W^1 x$}
      \State {\it // \ref{linetype:moment}}
      \State $\mean^1 := \f 1 n \sum_{\alpha=1}^n (W^1 x)_\alpha: \Ctype$
      \State $\var^1 := \f 1 n \sum_{\alpha=1}^n (W^1 x)_\alpha^2 - (\mean^1)^2: \Ctype$
      \State {\it // \ref{linetype:nonlin+}}
      \State $x^{1} := \relu\left(\f{(W^1 x) - \mean^1 \onev}{\sqrt{\var^1}}\right): \Htype(n)$
      \State $h^{2} := W^2 x^{1}: \Gtype(n)$
      \State {\it // Mean and variance of $h^2$}
      \State {\it // \ref{linetype:moment}}
      \State $\mean^2 := \f 1 n \sum_{\alpha=1}^n h^2_\alpha: \Ctype$
      \State $\var^2 := \f 1 n \sum_{\alpha=1}^n (h^2_\alpha)^2 - (\mean^2)^2: \Ctype$
      \State {\it // \ref{linetype:nonlin+}}
      \State $x^{2} := \relu\left(\f{h^2 - \mean^2 \onev}{\sqrt{\var^2}}\right): \Htype(n)$
      \State {\it // Same thing for $x'$}
      \State {\it // Mean and variance of $W^1 x'$}
      \State {\it // \ref{linetype:moment}}
      \State $\mean^1{}' := \f 1 n \sum_{\alpha=1}^n (W^1 x')_\alpha: \Ctype$
      \State $\var^1{}' := \f 1 n \sum_{\alpha=1}^n (W^1 x')_\alpha^2 - (\mean^1{}')^2: \Ctype$
      \State {\it // \ref{linetype:nonlin+}}
      \State $x^{1}{}' := \relu\left(\f{(W^1 x') - \mean^1{}' \onev}{\sqrt{\var^1{}'}}\right): \Htype(n)$
      \State $h^{2}{}' := W^2 x^{1}{}': \Gtype(n)$
      \State {\it // Mean and variance of $h^2{}'$}
      \State {\it // \ref{linetype:moment}}
      \State $\mean^2{}' := \f 1 n \sum_{\alpha=1}^n h^2_\alpha{}': \Ctype$
      \State $\var^2{}' := \f 1 n \sum_{\alpha=1}^n (h^2_\alpha{}')^2 - (\mean^2{}')^2: \Ctype$
      \State {\it // \ref{linetype:nonlin+}}
      \State $x^{2}{}' := \relu\left(\f{h^2{}' - \mean^2{}' \onev}{\sqrt{\var^2{}'}}\right): \Htype(n)$
      \Ensure $(v^\trsp x^{2}/\sqrt n, v^\trsp x^2{}' / \sqrt n)$
    \end{algorithmic}
    \end{multicols}
\end{algorithm}
}

Consider the example layernorm network in \cref{tp:layernorm}.
This is a self-parametrized \netsorplus program.

\paragraph{Setup}
Suppose the inputs $x, x' \ne 0$ are in $\R^m$.
The network has parameters $W^1 \in \R^{n \times m}, W^2 \in \R^{n \times n}$, and $v \in \R^n$.
Let us sample them as follows
\begin{align*}
W^1_{\alpha \beta} \sim \Gaus(0, \sigma_w^2/m),\quad
W^2_{\alpha \beta} \sim \Gaus(0, \sigma_w^2/n),\quad
v_\alpha \sim \Gaus(0, \sigma_v^2),
\end{align*}
for $\sigma_w, \sigma_v > 0$.
This corresponds to the \netsorplus sampling data $\muin = 0$ and $\Sigmain$ given as
\begin{align*}
\Sigmain(W^1 x, W^1 x) &= \sigma_w^2 \|x\|^2/m,&
\Sigmain(W^1 x, W^1 x') &= \sigma_w^2 x^\trsp x'/m,&
\Sigmain(W^1 x', W^1 x') &= \sigma_w^2 \|x'\|^2/m,
\end{align*}
$\Sigmain(v, v) = \sigma_v^2,$
and $\Sigmain(g, g') = 0$ for any other pairs of input G-vars $g, g'$.

\subsubsection{Computing \texorpdfstring{$\tmu$}{mu}, \texorpdfstring{$\tSigma$}{Sigma}, and Limit Parameters \texorpdfstring{$\mathring \theta$}{}}
Let's compute the values of $\tmu$, $\tSigma$, $\mathring \theta$ in order of the appearance of the variables, according to \cref{defn:netsorplusMuSigma}.
For each C-var or H-var $u$, we also show that $\varphi^u$ is parameter-controlled at $\mathring\vartheta^{u}$.

First, one can quickly notice that $\tmu(g) = 0$ for all G-vars $g$, so we shall focus on computing $\tSigma$ and $\mathring \theta$.

\paragraph{C-var $\mean^1$}
Here we have introduced $\mean^1$ via \ref{linetype:moment} by
\begin{align*}
\mean^1 := \f 1 n \sum_{\alpha=1}^n \varphi^{\mean^1}((W^1x)_\alpha),\quad\text{where
$\varphi^{\mean^1}(z) = z$,}
\end{align*}
and there are no parameters.
The function $\varphi^{\mean^1}$ is then obviously controlled and trivially parameter-controlled.
Finally, by \cref{eqn:mathringtheta}, we set
\begin{align*}
\mathring \mean^1 \defeq \EV_{z \sim \Gaus(0, \sigma^2_w\|x\|^2/m)} z = 0.
\end{align*}

\paragraph{C-var $\var^1$}
Note here
\begin{align*}
\var^1 := \f 1 n \sum_{\alpha=1}^n \varphi^{\var^1}((W^1x)_\alpha; \nu^1),\quad\text{where
$\varphi^{\var^1}(z; \theta) \defeq z^2 - \theta^2$.}
\end{align*}
(Here superscript-2 denotes square, not an index).
Since $\varphi^{\var^1}(-; -)$ is pseudo-Lipschitz in both its inputs and its parameter jointly, it is parameter-controlled at $\theta = \mathring \nu^1 = 0$ by \cref{exmp:parameterControl}.

In addition, $\mathring \var^1$ is computed by \cref{eqn:mathringtheta} as
\begin{align*}
\mathring \var^1 \defeq \EV_{z \sim \Gaus(0, \sigma^2_w\|x\|^2/m)} \varphi^{\var^1}(z; \mathring \mean^1)
    = \EV_{z \sim \Gaus(0, \sigma^2_w\|x\|^2/m)}
        z^2
    = 
        \f{\sigma_w^2}{m} \|x\|^2
        .
\end{align*}

\paragraph{H-var $x^1$}
The first H-var introduced in the program is $x^{1} := \relu\left(\f{(W^1 x) - \mean^1 \onev}{\sqrt{\var^1}}\right)$.
It can be written as a \ref{linetype:nonlin+} with
\begin{align*}
x^1 := \varphi^{x^1}(W^1 x; \mean^1, \var^1)
\end{align*}
where $\varphi^{x^1}(-; -): \R \times \R^2 \to \R$, and 
\begin{align*}
\varphi^{x^1}(z; \theta_1, \theta_2)
    &\defeq \relu\lp \f{z - \theta_1}{\sqrt{\theta_2}} \rp
        .
\end{align*}

Since $\sigma_w > 0$ and $x \ne 0$, we have $\mathring \var^1 \ne 0$, and we claim that $\varphi^{x^1}$ is parameter-controlled at $(\mathring \mean^1, \mathring \var^1) = \left(0, \f{\sigma_w^2}{m} \|x\|^2\right)$.

Indeed, $\varphi^{x^1}(-; \mathring \nu^1, \mathring \var^1)$ is obviously controlled, so that condition \ref{item:parameterControlled1} of \cref{defn:parameterControlled} is satisfied.
In addition, for any $z \in \R$,
\begin{align*}
\left|\varphi^{x^1}(z; \theta_1, \theta_2) - \varphi^{x^1}(z; \mathring \nu^1, \mathring\var^1)\right|
    &=
        \left|\relu\lp \f{z - \theta_1}{\sqrt{\theta_2}} \rp
            - \relu\lp \f{z - \mathring \nu^1}{\sqrt{\mathring\var^1}} \rp
        \right|
        \\
    &\le
        \left| \f{z - \theta_1}{\sqrt{\theta_2}}
            - \f{z - \mathring \nu^1}{\sqrt{\mathring\var^1}}
        \right|
        \\
    &=
        \left| z \lp \f 1 {\sqrt{\theta_2}} - \f {1}{\sqrt{\mathring\var^1}} \rp
        - \lp
            \f{\theta_1}{\sqrt{\theta_2}} 
            - \f{\mathring \nu^1}{\sqrt{\mathring\var^1}}
          \rp
        \right|
        \\
    &\le
        \left| z \lp \f 1 {\sqrt{\theta_2}} - \f {1}{\sqrt{\mathring\var^1}} \rp\right|
        + \left|
            \f{\theta_1}{\sqrt{\theta_2}} 
            - \f{\mathring \nu^1}{\sqrt{\mathring\var^1}}
        \right|
        \\
    &\le
        \sqrt{
                \lp \f 1 {\sqrt{\theta_2}} - \f {1}{\sqrt{\mathring\var^1}} \rp^2
                + \lp \f{\theta_1}{\sqrt{\theta_2}} 
                    - \f{\mathring \nu^1}{\sqrt{\mathring\var^1}}\rp^2
        }\sqrt{
            z^2 + 1
        }
\end{align*}
by Cauchy-Schwarz.
Note that $
\sqrt{
        \lp \f 1 {\sqrt{\theta_2}} - \f {1}{\sqrt{\mathring\var^1}} \rp^2
        + \lp \f{\theta_1}{\sqrt{\theta_2}} 
            - \f{\mathring \nu^1}{\sqrt{\mathring\var^1}}\rp^2
}$ equals 0 and is continuous at $(\theta_1, \theta_2) = (\mathring \nu^1, {\mathring\var^1})$ because ${\mathring\var^1} \ne 0$.
Then since $\sqrt{z^2 + 1}$ is controlled in $z$, $\varphi^{x^1}$ satisfies property \ref{item:parameterControlled2} of \cref{defn:parameterControlled}.
Altogther, we have shown that $\varphi^{x^1}$ is indeed parameter-controlled at $(\mathring \nu^1, {\mathring\var^1})$.

\paragraph{G-var $h^2$}
By the \ref{linetype:MatMul} case of \cref{eqn:selfParamExtendedMuSigma},
\begin{align*}
\Sigma(h^2, h^2) = \sigma_w^2 \EV_z \phi(z; \mathring \mean^1, \mathring \var^1)^2,\qquad
\Sigma(h^2, W^1 x) = \Sigma(h^2, W^1 x') = 0, 
\end{align*}
where $z \sim \Gaus(\tmu(W^1 x), \tSigma(W^1 x, W^1 x)) = \Gaus(0, \f{\sigma_w^2}{m} \|x\|^2)$, and
\begin{align*}
\phi(z; \theta_1, \theta_2) &\defeq \varphi^{x^1}(z; \theta_1, \theta_2) = \relu\lp \f{z - \theta_1}{\sqrt{\theta_2}} \rp.
\end{align*}
We can then simplify
\begin{align*}
\Sigma(h^2, h^2) &= \sigma_w^2 \EV_z \relu\lp \f{z}{\sqrt{\f{\sigma_w^2}{m} \|x\|^2}} \rp^2
    =
        \sigma_w^2 \EV_{\zeta\sim\Gaus(0, 1)} \relu(\zeta)^2
    =
        \f 1 2 \sigma_w^2
        .
\end{align*}

\paragraph{C-var $\mean^2$}
Similar to the case of $\mean^1$, we can express $\mean^2$ via \ref{linetype:moment} by
\begin{align*}
\mean^2 := \f 1 n \sum_{\alpha=1}^n \varphi^{\mean^2}(h^2_\alpha),\quad\text{where
$\varphi^{\mean^2}(z) = z$,}
\end{align*}
and there are no parameters.
The function $\varphi^{\mean^2}$ is then obviously controlled and trivially parameter-controlled.
Finally, by \cref{eqn:mathringtheta}, we set
\begin{align*}
\mathring \mean^2 \defeq \EV_{z \sim \Gaus(\tmu(h^2), \tSigma(h^2, h^2))} z = \EV_{z \sim \Gaus(0, \f 1 2 \sigma_w^2)} z = 0.
\end{align*}

\paragraph{C-var $\var^2$}
Similar to the case of $\var^1$, we can express $\var^2$ via \ref{linetype:moment} by
\begin{align*}
\var^2 := \f 1 n \sum_{\alpha=1}^n \varphi^{\var^2}(h^2_\alpha; \nu^2),\quad\text{where
$\varphi^{\var^2}(z; \theta) \defeq z^2 - \theta^2$.}
\end{align*}
(Here $z^2$ and $\theta^2$ are the squares of $z$ and $\theta$).
Since $\varphi^{\var^2}(-; -)$ is pseudo-Lipschitz in both its inputs and its parameter jointly, it is parameter-controlled at $\theta = \mathring \nu^2 = 0$ by \cref{exmp:parameterControl}.

In addition, $\mathring \var^2$ is computed by \cref{eqn:mathringtheta} as
\begin{align*}
\mathring \var^2
    \defeq
        \EV_{z \sim \Gaus(\tmu(h^2), \tSigma(h^2, h^2))} \varphi^{\var^2}(z; \mathring \mean^2)
    = \EV_{z \sim \Gaus(0, \f 1 2 \sigma_w^2)}
           z^2
    = 
        \f 1 2 \sigma_w^2
        .
\end{align*}

\paragraph{H-var $x^2$}
Similar to the case of $x^1$, we can express $x^2$ via \ref{linetype:nonlin+} by
\begin{align*}
x^2 := \varphi^{x^2}(h^2; \mean^2, \var^2)
\end{align*}
where $\varphi^{x^2}(-; -): \R \times \R^2 \to \R$, and 
\begin{align*}
\varphi^{x^2}(z; \theta_1, \theta_2)
    &\defeq \relu\lp \f{z - \theta_1}{\sqrt{\theta_2}} \rp
        .
\end{align*}

Since $\sigma_w > 0$, we also have $\mathring \var^2 > 0$.
Then by the same reasoning as in the case of $x^1$, $\varphi^{x^2}$ is parameter-controlled at $\mathring{\bigvtheta}^{x^2} = (\mathring \mean^2, \mathring \var^2)$.

\paragraph{C-vars $\mean^1{}', \var^1{}'$ and H-var $x^1{}'$}
These calculations proceed similarly to those for $\mean^1, \var^1$ and $x^1$.
We end up with
\begin{align*}
\mathring \mean^1{}' = 0,\quad
\mathring \var^1{}' = \f{\sigma_w^2}{m} \|x'\|^2,
\end{align*}
and, for each $u \in \{\mean^1{}', \var^1{}', x^1{}'\}$, the associated nonlinearity $\varphi^u$ is parameter-controlled at limit parameter $\bigtheta^{u}$.

\paragraph{G-var $h^2{}'$}
By the \ref{linetype:MatMul} case of \cref{eqn:selfParamExtendedMuSigma},
\begin{align*}
\Sigma(h^2, h^2{}') = \sigma_w^2 \EV_{z, z'} \phi(z; \mathring \mean^1, \mathring \var^1) \phi(z; \mathring \mean^1{}', \mathring \var^1{}'),\quad
\Sigma(h^2{}', h^2{}') = \sigma_w^2 \EV_{z'} \phi(z'; \mathring \mean^1{}', \mathring \var^1{}')^2
\end{align*}
and $\Sigma(h^2{}', g) = 0$ for all other G-var $g$ (by the ``otherwise'' case of \cref{eqn:selfParamExtendedMuSigma}),
where
\[(z, z') \sim \Gaus(\tmu|_{W^1 x, W^1 x'}, \tSigma|_{W^1 x, W^1 x'}) = \Gaus\lp 0,
    \f{\sigma_w^2}m
    \begin{pmatrix}
    \|x\|^2 & x^\trsp x'\\
    x^\trsp x' & \|x'\|^2
    \end{pmatrix}\rp\]
and
\begin{align*}
\phi(z; \theta_1, \theta_2) &\defeq \varphi^{x^1}(z; \theta_1, \theta_2) = \varphi^{x^1{}'}(z; \theta_1, \theta_2) = \relu\lp \f{z - \theta_1}{\sqrt{\theta_2}} \rp.
\end{align*}

We can simplify
\begin{align*}
\Sigma|_{h^2, h^2{}'}
    &=
        \sigma_w^2 \EV_{\tilde z,\tilde z'}
        \relu(\tilde z)\relu(\tilde z'),
        \quad
        (\tilde z, \tilde z') \sim \Gaus\lp 0,
            \begin{pmatrix}
            1 & \f{x^\trsp x'}{\|x\|\|x'\|}\\
            \f{x^\trsp x'}{\|x\|\|x'\|} & 1
            \end{pmatrix}\rp
        \\
    &=
        \sigma_w^2 \Vt{\relu}\begin{pmatrix}
            1 & \f{x^\trsp x'}{\|x\|\|x'\|}\\
            \f{x^\trsp x'}{\|x\|\|x'\|} & 1
            \end{pmatrix}
        ,
\end{align*}
where $\Vt\relu$ is as given in \cref{fact:Vrelu}.
In particular, with $c \defeq \f{x^\trsp x'}{\|x\|\|x'\|}$, this yields
\begin{align*}
\Sigma(h^2{}', h^2{}') = \Sigma(h^2{}, h^2{}) = \f 1 2 \sigma_w^2,\quad
\Sigma(h^2, h^2{}') =  \f {\sigma_w^2} {2\pi} (\sqrt{1-c^2} + (\pi - \arccos c) c)
.
\numberthis\label{eqn:SigmaRestricth2}
\end{align*}

\paragraph{C-vars $\mean^2{}', \var^2{}'$ and H-var $x^2{}'$}
These calculations proceed similarly to those for $\mean^2, \var^2$ and $x^2$.
We end up with
\begin{align*}
\mathring \mean^2{}' = 0,\quad
\mathring \var^2{}' = \f 1 2 \sigma_w^2,
\end{align*}
and, for each $u \in \{\mean^2{}', \var^2{}', x^2{}'\}$, the associated nonlinearity $\varphi^u$ is parameter-controlled at limit parameter $\bigtheta^{u}$.

\subsubsection{Computing the GP Kernel}
It is easy to see that the set of H-vars are all linearly independent almost surely.
Therefore we may apply \cref{cor:GPConvPlusSelfParam}.
By \cref{cor:GPConvPlusSelfParam}, $(v^\trsp x^{2}/\sqrt n, v^\trsp x^2{}' / \sqrt n)$ converges in distribution to $\Gaus(0, K)$ where
\begin{align*}
K =
\sigma_v^2 \EV_{z,z'}
\begin{pmatrix}
\phi(z; \mathring{\bigvtheta}^{x^2})^2
    & \phi(z; \mathring{\bigvtheta}^{x^2})\phi(z; \mathring{\bigvtheta}^{x^2{}'})
        \\
\phi(z; \mathring{\bigvtheta}^{x^2})\phi(z; \mathring{\bigvtheta}^{x^2{}'})
    & \phi(z'; \mathring{\bigvtheta}^{x^2{}'})^2
\end{pmatrix}
\end{align*}
where $\phi(z; \theta_1, \theta_2)
    \defeq \relu\lp \f{z - \theta_1}{\sqrt{\theta_2}} \rp$ and
 $(z, z') \sim \Gaus(\tmu|_{h^2, h^2{}'}, \tSigma|_{h^2, h^2{}'})$ with $\tmu|_{h^2, h^2{}'} = 0$ and $\tSigma|_{h^2, h^2{}'}$ given in \cref{eqn:SigmaRestricth2}.
Since $\mathring{\vartheta}^{x^2}_1 = \mathring \nu^2 = \mathring{\vartheta}^{x^2{}'}_1 = \mathring \nu^2{}' = 0$ and $\mathring{\vartheta}^{x^2}_2 = \mathring{\var}^2 = \mathring{\vartheta}^{x^2{}'}_2 = \mathring{\var}^2{}' = \f 1 2 \sigma_w^2$, we can simplify
\begin{align*}
K
    =
        {\sigma_v^2} \Vt\relu\lp
            (\sigma_w^2/2)^{-1}\Sigma|_{h^2, h^2{}'}\rp
    =
        \f{2\sigma_v^2}{\sigma_w^2} \Vt\relu\lp \Sigma|_{h^2, h^2{}'}\rp.
\end{align*}
\subsection{Layernorm: General Case}

As mentioned in \cref{sec:MoreExamples}, layernorm in general can be implemented with \ref{linetype:nonlin+}.

Suppose $y^1, \ldots, y^k : \Htype(n)$ are H-vars defined by $y^i := \phi^i(g^1, \ldots, g^M; \bigtheta^i)$ for (possibly self-)parametrized nonlinearities $\phi^i(-; - ): \R^m \to \R, i \in [k]$ and parameters $\bigtheta^i$ (possibly dependent on previous G-vars).
Suppose that each $\bigtheta^i$ converges almost surely to a deterministic vector $\mathring{\bigtheta}^i$,
and suppose each $\phi^i$ is parameter-controlled at $\mathring{\bigtheta}^i$.
Each of $y^i$ has mean
\begin{align*}
    \mean(y^i) \defeq
    \f 1 n \sum_{\alpha=1}^n y^i_\alpha
    = \f 1 n \sum_{\alpha=1}^n \phi^i(g^1_\alpha, \ldots, g^M_\alpha; \bigtheta^i)
\end{align*}
and variance
\begin{align*}
    \sigma^2(y^i) \defeq
    \f 1 n \sum_{\alpha=1}^n (y^i_\alpha)^2 - \nu(y^i)^2
    = \f 1 n \sum_{\alpha=1}^n \phi^i(g^1_\alpha, \ldots, g^M_\alpha; \bigtheta^i)^2 - \mean(y^i)^2
    .
\end{align*}
Under generic conditions (i.e.\ \cref{assm:asRankStab} and parameter-control), \cref{thm:Netsor+MasterTheorem} or \cref{thm:selfParamNetsorplusMasterTheorem} applies, so that
\[
\mean(y^i) \asto \mathring \mean(y^i) \defeq \EV_Z \phi^i\lp Z; \mathring{\bigtheta}^i\rp, \text{ and }
\sigma^2(y^i) \asto \mathring \sigma^2(y^i) \defeq \EV_Z \phi^i\lp Z; \mathring{\bigtheta}^i\rp^2 - \left[\EV_Z \phi^i\lp Z; \mathring{\bigtheta}^i \rp\right]^2
\]
where $Z \sim \Gaus(\tmu, \tSigma)$.
$\Layernorm(y^i)$ can then be expressed via a self-parametrized (\cref{defn:selfParam}) \ref{linetype:nonlin+} rule like so
\[
\Layernorm(y^i) = \psi(y^i; \mean(y^i), \sigma^2(y^i)),
\quad\text{where}\quad
\psi(z; a, b) \defeq (z - a) / \sqrt b.
\]

It's easy to check that $\psi(z; a,b)$ is parameter-controlled at $\mathring a, \mathring b$ as long as $\mathring b \ne 0$.
Assuming rank stability (\cref{assm:asRankStab}) is not violated by the new variables, \cref{thm:Netsor+MasterTheorem} holds, so that, intuitively, this application of \ref{linetype:nonlin+} can be replaced with a straightforward application of \ref{linetype:nonlin}:
\[
``\Layernorm(y^i) = \psi(y^i; \mean(y^i), \sigma^2(y^i))"
\to
``\Layernorm(y^i) = \psi(y^i; \mathring \mean(y^i), \mathring \sigma^2(y^i))"
.
\]
Therefore, if we define the kernel matrices
\begin{align*}
\Omega_{ij}
    &= \lim_{n \to \infty} y^i{}^\trsp y^j / n
    \\
\bar\Omega_{ij}
    &= \lim_{n \to \infty} \Layernorm(y^i)^\trsp \Layernorm(y^j) / n
    ,
\end{align*}
then
\begin{align*}
    \bar \Omega_{ij}
        &=
            \lim_{n \to \infty} \f 1 n \f{(y_i - \nu(y^i))^\trsp (y_j - \nu(y^j))}{\sqrt{\sigma^2(y^i)\sigma^2(y^j)}} \\
        &=
            \lim_{n \to \infty} \f{y_i^\trsp y_j/n - \nu(y^i) \nu(y^j)}{\sqrt{\sigma^2(y^i)\sigma^2(y^j)}}
            \\
        &=
            \lim_{n \to \infty} \f{y_i^\trsp y_j/n - \mathring \nu(y^i) \mathring \nu(y^j)}{\sqrt{\mathring \sigma^2(y^i)\mathring \sigma^2(y^j)}}
            \\
    \bar \Omega
        &=
            D^{-1/2}(\Omega - \mathring\nu \mathring\nu^\trsp)D^{-1/2},
\end{align*}
where $\mathring \nu$ is the column vector $(\mathring \nu(y^1), \ldots, \mathring \nu(y^k))^\trsp$ and $D = \Diag(\Omega - \mathring\nu \mathring\nu^\trsp)$.

In summary,
\begin{tcolorbox}[title=Computing Layernorm Kernel]
Suppose $y^1, \ldots, y^k : \Htype(n)$ are H-vars defined by $y^i := \phi^i(g^1, \ldots, g^M; \bigtheta^i)$ for (possibly self-)parametrized nonlinearities $\phi^i(-; -): \R^m \to \R, i \in [k]$ and parameters $\bigtheta^i$.
Assume that each $\bigtheta^i$ converges almost surely to a deterministic vector $\mathring{\bigtheta}^i$,
and that each $\phi^i$ is parameter-controlled at $\mathring{\bigtheta}^i$.
If we define the kernel matrices
\begin{align*}
\Omega_{ij}
    &= \lim_{n \to \infty} y^i{}^\trsp y^j / n
    \\
\bar\Omega_{ij}
    &= \lim_{n \to \infty} \Layernorm(y^i)^\trsp \Layernorm(y^j) / n
    ,
\end{align*}
then, assuming generic conditions (see main text above),
\begin{align*}
\bar \Omega = D^{-1/2}(\Omega - \mathring\nu \mathring\nu^\trsp)D^{-1/2},
\end{align*}
where $D = \Diag(\Omega - \mathring\nu \mathring\nu^\trsp)$ and $\mathring \nu$ is the vector given by $\mathring \nu_i = \EV \phi^i(Z; \mathring{\bigtheta}^i), Z \sim \Gaus(\tmu, \tSigma).$
\end{tcolorbox}

\subsection{Transformer \texorpdfstring{(\cref{tp:transformer})}{}}
\label{sec:transformerKernel}

\paragraph{The Transformer Variant, in Mathematical Terms}
We'll work with the following transformer model.
Let $x^0_1, \ldots, x^0_t$ be a sequence of inputs (the superscript will be layer index, and the subscript will be token index).
Then each layer $l$ of our transformer works like the following
\begin{align*}
k^l_i &= U^l x_i^{l-1} \in \R^n\\
h^l_i &= \Layernorm(k^l_i + \MaskedAttention_i(k^l_i, \{k^l_j\}_{j=1}^t, \{k^l_j\}_{j=1}^t))\\
    \numberthis\label{eqn:trsfmrAttnLayernorm}
x^l_i &= \Layernorm(W^{l2}\mathrm{relu}(W^{l1}h^l_i + b^{l1})+ b^{l2} + W^{l1} h^l_i)
\end{align*}
where $U^l$, $W^{l1}, W^{l2}$ are weights and $b^{l1}, b^{l2}$ are the biases, and
\begin{align*}
&\MaskedAttention_j(q, \{k^i\}_{i=1}^r, \{v^i\}_{i=1}^r)
= \sum_{i=1}^r a_i^j v^i,\\
&\quad \text{where}\quad
a_i^j = \SoftMax(q^\trsp k^1/n, \ldots, q^\trsp k^j/n, -\infty, \ldots, -\infty)_i
\numberthis\label{eqn:maskedattention}
\end{align*}
as described in \cref{sec:MoreExamples}.

Note that we make the following simplifications for ease of presentation, but all of them can be removed at the expense of more complex \netsor programs.
\begin{enumerate}
    \item We are forgoing positional embeddings
    \item The keys, values, and queries here are the same, compared to the standard version, where they are different linear projections of $x^{l-1}_i$
    \item There is only 1 head, compared to the standard multi-head attention
    \item The skip connection has base $W^{l2} h^l_i$ instead of just $h^l_i$
\end{enumerate}

\paragraph{Setup} assume for all $\alpha, \beta \in [n]$,
\begin{itemize}
    \item $W^{l1}_{\alpha\beta}, W^{l2}_{\alpha\beta} \sim \Gaus(0, \sigma_w^2/n)$ for all $l \ge 1$
    \item $U^l_{\alpha\beta} \sim \Gaus(0, \sigma_u^2/n)$ for all $l \ge 2$ and $U^1_{\alpha\beta} \sim \Gaus(0, \sigma_u^2/m)$
    \item $b^{l1}_\alpha, b^{l2}_\alpha \sim \Gaus(0, \sigma_b^2)$ for all $l$.
    \item $v_\alpha \sim \Gaus(0, \sigma_v^2)$
\end{itemize}

{   
\makeatletter
\renewcommand{\ALG@name}{Self-parametrized \netsorplus program}
\makeatother
\begin{algorithm}[tb]
    \caption{Transformer}
    \label{tp:transformer}
    \begin{algorithmic}[1]
      \Require $U^1 x^0_1, \ldots, U^1 x^0_t: \Gtype(n)$
      \Require $\forall l = 1, \ldots, L: W^{l1}, W^{l2}: \Atype(n, n)$
      \Require $\forall l = 2, \ldots, L: U^l: \Atype(n, n)$
      \Require $\forall l = 1, \ldots, L: b^{l1}, b^{l2}: \Gtype(n)$
      \Require $v: \G(n)$
      \For{$l = 1, \ldots, L$}
          \For{$i = 1, \ldots, t$}
          \State {\it // if $l=1$, apply \ref{linetype:lincomb}}
          \State {\it // if $l\ge 2$, apply \ref{linetype:MatMul}}
          \State $k^l_i := U^l x^{l-1}_i: \Gtype(n)$
          \EndFor
          \For{$i = 1, \ldots, t$}
            \For{$j = 1, \ldots, t$}
              \State {\it // \ref{linetype:moment}}
              \State $c_{ij} := k^l_i{}^\trsp k^l_j / n: \Ctype$
                \label{line:cij}
            \EndFor
          \State {\it // With $a^i_j$ being shorthand for}
          \State {\it // $\SoftMax(c_{i1}, \ldots, c_{ii}, -\infty, \ldots, -\infty)_j$}
          \State {\it // Mean, post attention}
          \State $\mean_i := \f 1 n \sum_{\alpha=1}^n (k^l_i + \sum_{j=1}^t a^i_j k^l_j)_\alpha: \Ctype$
            \label{line:meani}
          \State {\it // Variance, post attention}
          \State $\var_i = \f 1 n \sum_{\alpha=1}^n (k^l_i + \sum_{j=1}^t a^i_j k^l_j)_\alpha^2 - \nu_i^2: \Ctype$
            \label{line:vari}
          \State {\it // applying \ref{linetype:nonlin+} to express attention+layernorm}
          \State $h^l_i := (k^l_i + \sum_{j=1}^t a^i_j k^l_j - \mean_i \onev)/\sqrt{\var_i}: \Htype(n)$
            \label{line:hli}
          \EndFor
          \For{$i = 1, \ldots, t$}
          \State $y^{l1}_i := W^{l1} h^l_i: \Gtype(n)$
          \State $\hat y^{l1}_i := y^{l1}_i + b^{l1}: \Gtype(n)$
          \State $\hat x^{l1}_i := \relu(\hat  y^{l1}_i): \Htype(n)$
          \State $y^{l2}_i := W^{l2} \hat x^{l1}_i: \Gtype(n)$
          \State $\hat y^{l2}_i := y^{l2}_i + b^{l2}: \Gtype(n)$
          \State {\it // Layernorm mean and variance}
          \State $\mean_i' := \f 1 n \sum_{\alpha=1}^n (\hat y^{l2}_i)_\alpha + (y^{l1}_i)_\alpha: \Ctype$
            \label{line:ffmeani}
          \State $\var_i' := \f 1 n \sum_{\alpha=1}^n ((\hat y^{l2}_i)_\alpha + (y^{l1}_i)_\alpha)^2 - (\mean_i')^2: \Ctype$
            \label{line:ffvari}
          \State {\it // Layernorm}
          \State $x^l_i := (\hat y^{l2}_i + y^{l1}_i - \mean_i'\onev) / \sqrt{\var_i'}: \Htype(n)$
          \EndFor
      \EndFor
      \Ensure $(v^\trsp x^L_1/ \sqrt n, \ldots, v^\trsp x^L_t / \sqrt n)$
    \end{algorithmic}
\end{algorithm}
}

\subsubsection{Expressing the Composition of Attention, Skip Connection, and Layernorm via \texorpdfstring{\ref{linetype:nonlin+}}{Nonlin+} and \texorpdfstring{\ref{linetype:moment}}{Moment}}
\cref{tp:transformer} captures the computation of this transformer on an input sequence.
Let us explain how \cref{eqn:trsfmrAttnLayernorm} is expressed in \cref{tp:transformer}.
Throughout the below, we will use the easy observation that $\mu(g) = 0$ for all G-vars $g$.
For any layer $l$, we proceed as follows.

\paragraph{Attention Weights}
First, $c_{ij}$ in \cref{line:cij} represents a pre-SoftMax logit for the attention weights.
They are introduced via \ref{linetype:moment} by
\begin{align*}
c_{ij} := \f 1 n \sum_{\alpha=1}^n \varphi^{c_{ij}}((k^l_i)_\alpha, (k^l_j)_\alpha),\quad
\text{where $\varphi^{c_{ij}}(z_1, z_2) = z_1 z_2.$}
\end{align*}
This implies
\begin{align*}
\mathring c_{ij} = \EV_{Z \sim \Gaus(\tmu, \tSigma)} Z^{k^l_i} Z^{k^l_j} = \tSigma(k^l_i, k^l_j),
\numberthis\label{eqn:mathringcij}
\end{align*}
where we used $\tmu(g) = 0$ for all G-vars $g$.

\paragraph{Layernorm Mean and Variance}
Next, $\mean_i$ in \cref{line:meani} and $\var_i$ in \cref{line:vari} represent the mean and variance of the post-attention embedding of the $i$th token.
They are introduced via \ref{linetype:moment} by
\begin{align*}
\mean_i
    &:=
        \f 1 n \sum_{\alpha=1}^n \varphi^{\mean_i}((k^l_1)_\alpha, \ldots, (k^l_t)_\alpha; c_{i1}, \ldots, c_{ii})
        \\
\var_i
    &:=
        \f 1 n \sum_{\alpha=1}^n \varphi^{\var_i}((k^l_1)_\alpha, \ldots, (k^l_t)_\alpha; c_{i1}, \ldots, c_{ii}, \mean_i)
\end{align*}
where
\begin{align*}
\varphi^{\mean_i}(z_1, \ldots, z_t; \theta_1, \ldots, \theta_i)
    &\defeq
        z_i + \sum_{j=1}^t a_j z_j,
        \\
    &\text{where
        $(a_1, \ldots, a_t) = \SoftMax(\theta_1, \ldots, \theta_i, -\infty, \ldots, -\infty),$}
        \numberthis\label{eqn:aSoftMaxShortHand}
\end{align*}
and similarly,
\begin{align*}
\varphi^{\var_i}(z_1, \ldots, z_t; \theta_1, \ldots, \theta_i, \mean)
    &\defeq
        (z_i + \sum_{j=1}^t a_j z_j)^2 - \mean^2,
        \\
    &\text{where
        $(a_1, \ldots, a_t)$ are as in \cref{eqn:aSoftMaxShortHand}.}
\end{align*}
Note that both $\varphi^{\mean_i}$ and $\varphi^{\var_i}$ are pseudo-Lipschitz in both their inputs and parameters jointly, so that they are parameter-controlled by \cref{exmp:parameterControl}.

Their limit parameters can be computed as
\begin{align*}
\mathring\mean_i
    &=
        \tmu(k^l_i) + \sum_{j=1}^t \mathring a_j \tmu(k^l_j)
    = 0
        \\
    &\text{where
        $(\mathring a_1, \ldots, \mathring a_t) = \SoftMax(\mathring \theta_1, \ldots, \mathring\theta_i, -\infty, \ldots, -\infty),$}
        \numberthis\label{eqn:aSoftMaxShortHandLimit}
\end{align*}
since $\tmu = 0$ identically, and
\begin{align*}
\mathring \var_i
    &=
        \tSigma(k^l_i, k^l_i)
        + 2 \sum_{j} \mathring a_j \tSigma(k^l_i, k^l_j)
        +
        \sum_{j, j'} \mathring a_j \mathring a_{j'} \tSigma(k^l_j, k^l_{j'})
        \numberthis\label{eqn:varLimitTrsfmr}
\end{align*}
with $\mathring a_j$ same as in \cref{eqn:aSoftMaxShortHandLimit}.

\paragraph{Putting Them All Together}
Finally, $h^l_i$ in \cref{line:hli} represents the post-layernorm activations and is introduced via \ref{linetype:nonlin+} by
\begin{align*}
h^l_i
    &:=
        \varphi^{h^l_i}(k^l_1, \ldots, k^l_t; c_{i1}, \ldots, c_{ii}, \mean_i, \var_i)
\end{align*}
where
\begin{align*}
\varphi^{h^l_i}(z_1, \ldots, z_t; \theta_1, \ldots, \theta_i, \mean, \var)
    &:=
        (z_i + \sum_{j=1}^t a_j z_j - \mean) / \sqrt{\var}
        \numberthis\label{eqn:trsfmrVarphiH}
        \\
    &\text{where
        $(a_1, \ldots, a_t)$ are as in \cref{eqn:aSoftMaxShortHand}.}
\end{align*}
If $\mathring \var_i > 0$, then one can show that $\varphi^{h^l_i}$ is parameter-controlled at $(\mathring c_{i1}, \ldots, \mathring c_{ii}, \mathring \nu_i, \mathring \var_i)$ via the same reasoning as in \cref{sec:netsorplusKernelExample}.
When is $\mathring \var_i > 0$?
From \cref{eqn:varLimitTrsfmr}, because the $a_i$ are all nonnegative, $\mathring \var_i = 0$ implies that $\tSigma(k^l_i, k^l_i) = 0$.
This is impossible if all of the input tokens $x_i$ are nonzero and the weight variances satisfy $\sigma_w, \sigma_u > 0$, as one can easily see.

\subsubsection{Computing the GP Kernel}

By \cref{cor:GPConvPlusSelfParam}, the output vector converges in distribution to $\Gaus(0, K)$, where $K \in \R^{t \times t}$, and
\begin{align*}
K_{ij} = \sigma_v^2 \EV_{Z \sim \Gaus(\mu, \Sigma)} \varphi^{x^L_i}(Z; \mathring \Theta^{x^L_i}) \varphi^{x^L_j}(Z; \mathring \Theta^{x^L_j})
.
\end{align*}
Here $\Theta^{x^L_i} = \{ \mean'_i, \var'_i\}$ as given in \cref{line:ffmeani,line:ffvari}, and
\begin{align*}
\varphi^{x^L_i}(Z; \mean, \var)
    &=
        (Z^{\hat y^{L2}_i}  + Z^{y^{l1}_i}- \mean)/\sqrt{\var}
        .
\end{align*}
Simultaneously,
\begin{align*}
\mathring \mean'_i
    &=
        \EV_Z Z^{\hat y^{L2}_i} + Z^{y^{L1}_i},
        \quad
\mathring \var'_i
    =
        \EV_Z (Z^{\hat y^{L2}_i} + Z^{y^{L1}_i})^2 - (\mathring \mean'_i)^2
        .
\end{align*}
Thus, to compute $K$, it suffices to compute the restriction $\tSigma|_{y^{L1}_1, \ldots, y^{L1}_t, \hat y^{L2}_1, \ldots, \hat y^{L2}_t}$, from which $K$ can be computed by the equations above.

However, notice that by the ``otherwise'' case \cref{eqn:selfParamExtendedMuSigma}, $\Sigma({h^{L2}_i}, {h^{L1}_i}) = 0$ because $\hat h^{L2}_i$ and $h^{L1}_i$ are introduced by \ref{linetype:MatMul} with different A-vars, and consequently $\Sigma({\hat h^{L2}_i}, {h^{L1}_i}) = \Sigma({h^{L2}_i}, {h^{L1}_i}) + \Sigma(b^{l2}, {h^{L1}_i})= 0$.
Therefore, we only need to compute $\tSigma|_{y^{L1}_1, \ldots, y^{L1}_t}$ and $\tSigma|_{\hat y^{L2}_1, \ldots, \hat y^{L2}_t}$ separately.
Then $K$ is given by
\begin{align}
K = \sigma_v^2 D^{-1/2}(\tSigma|_{y^{L1}_1, \ldots, y^{L1}_t}
                    + \tSigma|_{\hat y^{L2}_1, \ldots, \hat y^{L2}_t})
                D^{-1/2},
    \label{eqn:GPkernelTransformer}
\end{align}
where $D$ is the diagonal matrix with diagonal equal to the diagonal of $\tSigma|_{y^{L1}_1, \ldots, y^{L1}_t}
                    + \tSigma|_{\hat y^{L2}_1, \ldots, \hat y^{L2}_t}$.

\subsubsection{Computing \texorpdfstring{$\tSigma$}{Sigma}}

Let
\[\tSigma^{\hat y^{l2}} \defeq \tSigma|_{\hat y^{l2}_1, \ldots, \hat y^{l2}_t} ,\quad
\tSigma^{y^{l1}} \defeq \tSigma|_{y^{l1}_1, \ldots, y^{l1}_t},\quad
\tSigma^{k^l} \defeq \tSigma|_{k^l_1, \ldots, k^l_t}\]
resp.\ be the restriction of $\tSigma$ to $\{\hat y^{l2}_i\}_i$, $\{y^{l1}_i\}_i$, and $\{\hat k^l_i\}_i$.
As explained above, the kernel of the Gaussian process underlying the output vector $(v^\trsp x^L_1/ \sqrt n, \ldots, v^\trsp x^L_t / \sqrt n)$ can be computed from $\tSigma^{\hat y^{L2}}$.

In this section, we shall describe equations tying together $\tSigma^{\hat y^{l2}}, \tSigma^{y^{l1}}, \tSigma^{k^l}$ that will allow us to compute $\tSigma^{\hat y^{L2}}$ recursively.

\paragraph{Computing $\tSigma^{y^{l1}}$ from $\tSigma^{k^l}$.}

The G-var $y^{l1}_i$ is introduced as $y^{l1}_i := W^{l1} h^l_i$.
Then given \cref{eqn:trsfmrVarphiH}, we have, for any $i, i' \in [t]$,
\begin{align}
    \tSigma(y^{l1}_i, y^{l1}_{i'}) = \f {\sigma_w^2} {\sqrt{\mathring \var_i \mathring \var_{i'}}}\lp
        \tSigma(k^l_i, k^l_{i'})
        + \sum_j \mathring a^i_j \tSigma(k^l_j, k^l_{i'})
        + \sum_{j'} \mathring a^{i'}_{j'} \tSigma(k^l_i, k^l_{j'})
        + \sum_{j, j'} \mathring a^i_j \mathring a^{i'}_{j'}\tSigma(k^l_j, k^l_{j'})\rp
        ,
    \label{eqn:transformerYl1}
\end{align}
where
\begin{align*}
(\mathring a^i_1, \ldots, \mathring a^i_t)
    &=
        \SoftMax(\mathring c_{i1}, \ldots, \mathring c_{ii}, -\infty, \ldots, -\infty)
        \\
    &=
        \SoftMax(\Sigma(k^l_i, k^l_1), \ldots, \Sigma(k^l_i, k^l_i), -\infty, \ldots, -\infty)
\end{align*}
by \cref{eqn:mathringcij}, 
and likewise for $i'$.
This reduces computing $\tSigma^{y^{l1}}$ to computing $\tSigma^{k^l}$.

\paragraph{Computing $\tSigma^{\hat y^{l2}}$ from $\tSigma^{y^{l1}}$.}

By some simple calculations in the vein of \cref{sec:MLPmulti}, we can also see
\begin{align}
\tSigma^{\hat y^{l2}}
    =
        \sigma_w ^2 \Vt\relu\lp \tSigma^{y^{l1}} + \sigma_b^2 \rp + \sigma_b^2
        .
    \label{eqn:transformerYhatl2}
\end{align}

\paragraph{Computing $\tSigma^{k^{l+1}}$ from $\tSigma^{\hat y^{l2}}$.}

Finally, following the same reasoning as in \cref{sec:layernorm}, we get
\begin{align*}
\mathring \mean_i' = 0,\quad
\mathring \var_i' = \tSigma(\hat y^{l2}_i, \hat y^{l2}_i) + \tSigma(y^{l1}_i, y^{l1}_i),
\end{align*}
$\varphi^{x^l_i}$ is parameter-controlled at $\mathring \bigtheta^{x^l_i}$ as long as $\mathring \var_i' > 0$, and
\begin{align}
\tSigma^{k^{l+1}}
    = 
        \sigma_u^2 D^{-1/2} (\tSigma^{\hat y^{l2}} + \tSigma^{y^{l1}}) D^{-1/2}
    \label{eqn:transformerKl}
\end{align}
where $D = \Diag(\tSigma^{\hat y^{l2}} + \tSigma^{y^{l1}})$.

Putting them all together, \cref{eqn:transformerYl1,eqn:transformerYhatl2,eqn:transformerKl} along with \cref{eqn:GPkernelTransformer} yield the complete set of equations to compute the GP kernel of a transformer.

\subsubsection{Vectorized Implementation: Single Sequence}
\cref{eqn:transformerYhatl2,eqn:transformerKl,eqn:GPkernelTransformer} are already in vectorized forms.
The following equation expresses \cref{eqn:transformerYl1} in a vectorized form as well:
\begin{align*}
\Sigma^{y^l} = \sigma_w^2 D^{-1/2}(I + \Delta) \Sigma^{k^{l}} (I + \Delta)^\trsp D^{-1/2}
\end{align*}
where 
\begin{itemize}
    \item $\Delta = \SoftMax(\mathrm{Mask}(\Sigma^{k^l}))$, with SoftMax applied to each row, and $\mathrm{Mask}(\Sigma^{k^l})$ is the same as $\Sigma^{k^l}$, except that its upper triangular portion (above the diagonal) is all set to $-\infty$, and
    \item $D$ is the diagonal matrix with diagonal equal to the diagonal of $(I + \Delta) \Sigma^{k^l} (I + \Delta)^\trsp$.
\end{itemize}

Here, $\Delta$ is the attention weights, masked so that a token's embedding cannot depend on those of future tokens.
The identity matrix $I$ appears due to the skip connection.
And the multiplication by $D^{-1/2}$ is as result of layernorm.

\subsubsection{Vectorized Implementation: Double Sequence}
\cref{tp:transformer} only expresses the computation of a transformer on a single sequence.
In general, the GP kernel will also have covariances between the embeddings of tokens of one sequence and those of tokens of another sequence.
One can derive the computation of these covariances just as we did above for a single sequence.
Below, we will just summarize the vectorized implementation for computing the joint kernel over multiple input sequences.
One should think of $\vecSigma^l$ below as the tensor of $\Sigma^{k^{l}}$ over every pair of sequences, and one should think of $\hat \vecSigma^l$ as the same for $\Sigma^{y^l}$.

\begin{tcolorbox}[title=Computing Transformer Kernel]
Suppose we have $p$ input sequences $\{(x_{1a}, \ldots, x_{ta})\}_{a=1}^p$, each with $t$ tokens.
Suppose each sequence is processed by a transformer as in \cref{tp:transformer}, and the transformer's parameters are sampled with nonzero variances as follows.
\begin{itemize}
    \item $W^{l1}_{\alpha\beta}, W^{l2}_{\alpha\beta} \sim \Gaus(0, \sigma_w^2/n)$ for all $l \ge 1$
    \item $U^l_{\alpha\beta} \sim \Gaus(0, \sigma_u^2/n)$ for all $l \ge 2$ and $U^1_{\alpha\beta} \sim \Gaus(0, \sigma_u^2/m)$
    \item $b^{l1}_\alpha, b^{l2}_\alpha \sim \Gaus(0, \sigma_b^2)$ for all $l$.
    \item $v_\alpha \sim \Gaus(0, \sigma_v^2)$
\end{itemize}

Then the transformer's outputs, one scalar for each input token, converge in distribution to a Gaussian $\Gaus(0, \KK)$ where $\KK \in \R^{pt \times pt}$ can be computed as follows:
\begin{enumerate}
    \item Initialize $\vecSigma^0 \in \R^{t \times p \times t \times p}$ by
     $\vecSigma^0_{iajb} \gets \sigma_u^2 x_{ia}^\trsp x_{jb} / m$ for all $a, b \in[p]$ and $i,j \in [t]$.
    \item For $l = 1, \ldots, L$, do
    \begin{enumerate}
        \item For $a = 1, \ldots, p$, do
        \begin{enumerate}
            \item $\Sigma^{l-1,a} \gets \vecSigma^{l-1}_{\bullet a\bullet a}$ be the $a$th ``diagonal block''
            \item $\Delta^{la} \gets \SoftMax(\mathrm{Mask}(\Sigma^{l-1, a}))$, where $\mathrm{Mask}$ replaces the upper triangular portion (above the diagonal) with $-\infty$, and $\SoftMax$ is applied row-wise.
        \end{enumerate}
        \item $\mathbf \Delta^{l} \gets $ block diagonal matrix with $\Delta^{l1}, \ldots, \Delta^{lp}$ as blocks.
        \item {\it // below, we treat each tensor as a $(pt \times pt)$ matrix.}
        \item $\hat \vecSigma^{l} \gets \mathbf (I + \mathbf \Delta^{l}) \vecSigma^{l-1} (I + \mathbf \Delta^{l})^\trsp$
        \item $\hat \vecSigma^{l} \gets \sigma_w^2 D^{-1/2} \hat \vecSigma^{l} D^{-1/2}$, where $D = \Diag(\vecSigma^{l})$
        \item $\vecSigma^{l} \gets
            \sigma_w^2 \Vt{\mathrm{ReLU}}(
                \hat \vecSigma^l + \sigma_b^2)
                    + \sigma_b^2$
        \item $\vecSigma^{l} \gets \sigma_u^2 D^{-1/2} (\vecSigma^{l} + \hat\vecSigma^{l})D^{-1/2}$, where $D = \Diag(\vecSigma^{l} + \hat \vecSigma^{l})$
    \end{enumerate}
    \item Return $\f{\sigma_v^2}{\sigma_u^2} \vecSigma^L$
\end{enumerate}
\end{tcolorbox}

See our repo \repo{} for an implementation of this algorithm.

\section{Different Versions of Tensor Programs}
\label{sec:VersionTensorPrograms}
\begin{defn}\label{defn:netsormin}
A \netsormin program is a \netsor program without the \ref{linetype:lincomb} rule.
\end{defn}

\begin{remk}\label{remk:netsormin}
Any \netsor program is semantically identical to a \netsormin program, by absorbing any usage of \ref{linetype:lincomb} into a downstream nonlinearity (e.g., if $g := g^1 + g^2$, and $h := \phi(g)$, write $h := \phi(g^1 + g^2)$ directly as an application of \ref{linetype:nonlin}), or if there is no downstream nonlinearity, treat it as an application of \ref{linetype:nonlin}.
Because \ref{linetype:lincomb} allows one to express certain gadgets such as skip connection and convolutions more easily, we chose to present \netsor as the canonical version of Tensor Program here.
See \cref{sec:formalspec} for a formal specification of \netsormin.
\end{remk}

By the remark above, the following \netsormin Master Theorem is equivalent to \cref{thm:netsorMasterTheorem}.

\begin{restatable}[\netsormin Master Theorem]{thm}{netsorminMasterTheorem}
\label{thm:netsorminMasterTheorem}
Fix any \netsormin program satisfying \cref{assm:equalDimNoNonlin+} and with all nonlinearities controlled.
If $g^1, \ldots, g^M$ are all of the G-vars (including all input G-vars), then for any controlled $\psi: \R^M \to \R$, as $n \to \infty$,
\begin{align*}
    \f 1 n \sum_{\alpha=1}^n \psi(g^1_\alpha, \ldots, g^M_\alpha) \asto 
    \EV_{Z \sim \Gaus(\tmu, \tSigma)}\psi(Z)
    =
    \EV_{Z \sim \Gaus(\tmu, \tSigma)}\psi(Z^{g^1}, \ldots, Z^{g^M}),
\end{align*}
where $\asto$ means almost sure convergence,
$Z = (Z^{g^1}, \ldots, Z^{g^M}) \in \R^M$, and $\tmu = \{\tmu(g^i)\}_{i=1}^M \in \R^M$ and $\tSigma = \{\tSigma(g^i, g^j)\}_{i,j=1}^M \in \R^{M \times M}$ are given in \cref{eqn:extendedMuSigma} (note that the cases involving \ref{linetype:lincomb} in \cref{eqn:extendedMuSigma} are now vacuous in this setting with \netsormin program).
See \cref{fig:mastertheoremIllustration} for an illustration.
\end{restatable}

To prove \cref{thm:netsorMasterTheorem}, we will in fact prove \cref{thm:netsorminMasterTheorem}; see \cref{sec:proofMasterTheorems}.

\begin{defn}\label{defn:netsoro}
A \netsoro program (pronounced ``Net-Sor-O'') is a \netsor program but where \ref{linetype:nonlin} rules allow nonlinearities $\phi$ to take H-vars.
\netsoro programs are thus a superset of \netsor programs.
Similarly, a \netsoroplus (pronounced ``Net-Sor-O-Plus'') program is a \netsorplus program but where \ref{linetype:nonlin+} rules allow nonlinearities $\phi$ to take H-vars.
\end{defn}

\begin{remk}\label{remk:netsoro}

Any \netsoro program is semantically identical to a \netsor program:
If $g := W h$ is any application of \ref{linetype:MatMul}, we can rewrite $h$ as a function of G-vars only by
unwinding its definition recursively (e.g., if $h := \phi(h^1, g)$ and $h^1 := \psi(g^1, g^2)$, then we can write directly $h := \phi(\psi(g^1, g^2), g)$ using a single application of \ref{linetype:nonlin} in G-vars).
Likewise, any \netsoroplus program can be rewritten as a \netsorplus program without losing any information.

\netsoro programs can be more concise than \netsor programs by reusing H-vars more efficiently; see \cref{tp:GRUNetsoro} for GRU expressed in \netsoro, and compare to \cref{tp:GRU}.
However, the Master Theorem is more complicated to state, and the task of unwinding the nonlinearity just shifts from the program to the scaling limit computation stage; see \cref{eqn:extendedMuSigma2} below.
This is why we did not present \netsoro as the canonical version of Tensor Programs.
\end{remk}

{
    
\makeatletter
\renewcommand{\ALG@name}{\netsoro program}
\makeatother
\begin{algorithm}[tb]
    \caption{GRU, with Gating Function $\sigma$ and Activation Function $\phi$}
    \label{tp:GRUNetsoro}
    \begin{algorithmic}
      \State {\it // Embeddings of input sequence}
      \Require $U_\zz x^1, \ldots, U_\zz x^t: \Gtype(n)$
      \Require $U_\rr x^1, \ldots, U_\rr x^t: \Gtype(n)$
      \Require $U_\hh x^1, \ldots, U_\hh x^t: \Gtype(n)$
      \State {\it // Parameters}
      \Require $W_\zz, W_\rr, W_\hh: \Atype(n, n)$
      \Require $b_\zz, b_\rr, b_\hh: \Gtype(n)$
      \State {\it // Initial GRU state}
      \Require $h^0: \Gtype(n)$
      \State {\it // Readout layer}
      \Require $v: \Gtype(n)$
      \State {\it // Time step 1}
      \State $h_\zz^{1} := W_\zz h^0: \Gtype(n)$
      \State $\tilde z^1 := h_\zz^{1} + U_\zz x^1 + b_\zz: \Gtype(n)$
      \State $h_\rr^{1} := W_\rr h^0: \Gtype(n)$
      \State $\tilde r^1 := h_\rr^{1} + U_\rr x^1 + b_\rr: \Gtype(n)$
      \State {\it // $\sigma$ is gating function, typically sigmoid; applying \ref{linetype:nonlin}}
      \State $\hat h^0 := h^0 \odot \sigma(\tilde r^1): \Htype(n)$
      \State $h_\hh^{1} := W_\hh \hat h^0: \Gtype(n)$
      \State $\tilde h^1 := h_\hh^{1} + U_\hh x^1 + b_\hh: \Gtype(n)$
      \State {\it // Apply \ref{linetype:nonlin}}
      \State {\it // $\phi$ is activation function, typically $\tanh$}
      \State $h^1 := (1 - \sigma(\tilde z^1)) \odot h^0 + \sigma(\tilde z^1) \odot \phi(\tilde h^1): \Htype(n)$
      \State {\it // Time step 2}
      \State $h_\zz^{2} := W_\zz h^1: \Gtype(n)$
      \State $\tilde z^2 := h_\zz^{2} + U_\zz x^2 + b_\zz: \Gtype(n)$
      \State $h_\rr^{2} := W_\rr h^1: \Gtype(n)$
      \State $\tilde r^2 := h_\rr^{2} + U_\rr x^2 + b_\rr: \Gtype(n)$
      \State {\it // No longer need to unwind $h^1$ as in \cref{tp:GRU}}
      \State $\hat h^1 = \sigma(\tilde r^1) \odot h^1: \Htype(n)$
      \State $h_\hh^{2} := W_\hh \hat h^1: \Gtype(n)$
      \State $\tilde h^2 := h_\hh^{2} + U_\hh x^2 + b_\hh: \Gtype(n)$
      \State {\it // No longer need to unwind $h^1$ as in \cref{tp:GRU}}
      \State $h^2 := (1 - \sigma(\tilde z^2)) \odot h^1 + \sigma(\tilde z^2) \odot \phi(\tilde h^2): \Htype(n)$
      \State {\it // Time step 3}
      \State $\vdots$
      \State {\it // Time step $t$}
      \State {\it // Define $\tilde z^t, \tilde r^t, \tilde h^t$ just like above}
      \State $\vdots$
      \State
      {\it // No longer need to unwind $h{t-1}$ as in \cref{tp:GRU}}
      \State  $h^t := (1 - \sigma(\tilde z^t)) \odot h^{t-1} + \sigma(\tilde z^t) \odot \phi(\tilde h^t): \Htype(n)$

      \Ensure $(v^\trsp h^1/\sqrt{n}, \ldots, v^\trsp h^t/\sqrt{n})$
    \end{algorithmic}
\end{algorithm}

}

\begin{defn}\label{defn:unwindedNonlin}
Fix a \netsoro program.
For any H-var $h$, let $\varphi^h$ be the unwinded nonlinearity expressing $h$ as a function of only G-vars, as described in \cref{remk:netsoro}, i.e.\ 
$h = \varphi^h(g^1, \ldots, g^M)$.
For example, if $h := \phi(h^1, g^3)$ and $h^1 := \psi(g^1, g^2)$, then $h = \phi(\psi(g^1, g^2), g^3)$ and $\varphi^h = \phi(\psi(-, -), -)$.

Similarly, in a \netsoroplus program, if $h$ is an H-var, let $\varphi^h$ be the unwinded nonlinearity (possibly with parameters) expressing $h$ as a function only G-vars, $h = \varphi^h(g^1, \ldots, g^M; \bigtheta)$.

For example, if $h := \phi(h^1, g^3; \theta_2)$ and $h^1 := \psi(g^1, g^2; \theta_1)$, then $h = \phi(\psi(g^1, g^2; \theta_1), g^3; \theta_2)$ and $\varphi^h(-, -, -; \theta_1, \theta_2) = \phi(\psi(-, -; \theta_1), -; \theta_2)$.
\end{defn}
Note that this $\varphi^h$ notation is consistent with the semantics of the same notation defined in \cref{defn:netsorplusMuSigma}, where there is nothing to unwind.

The extended mean and covariance $\tmu$ and $\tSigma$ can still be computed as before in a \netsoro program.
The only difference is that we are using the unwinded nonlinearities $\varphi^h$ instead.
\begin{align*}
    \tmu(g)
        &=
            \begin{cases}
            \muin(g)  &   \text{if $g$ is input}\\
            \sum_{i} a_i \tmu(y^i)    &   \text{if $g = \sum_{i} a_i y^i$, introduced by \ref{linetype:lincomb}}\\
            0   &   \text{otherwise}
            \end{cases},
            \\
    \tSigma(g, g')
        &=
            \begin{cases}
            \Sigmain(g, g')   &   \text{if $g, g'$ are inputs}\\
            \sum_{i} a_i \tSigma(y^i, g')    &   \text{if $g = \sum_{i} a_i y^i$, introduced by \ref{linetype:lincomb}}\\
            \sum_{i} a_i \tSigma(g, y^i)    &   \text{if $g' = \sum_{i} a_i y^i$, introduced by \ref{linetype:lincomb}}\\
            \sigma^2_W \EV_Z \varphi^h(Z) \varphi^{h'}(Z) &   \text{if $g = Wh, g'=Wh'$, introduced by \ref{linetype:MatMul} w/ same A-var $W$}\\
            0   &   \text{otherwise}
            \end{cases}
            \numberthis\label{eqn:extendedMuSigma2}
\end{align*}
where $\varphi^h$ and $\varphi^{h'}$ is as defined in \cref{defn:unwindedNonlin}
and
$Z \sim \Gaus(\tmu, \tSigma)$.

\begin{restatable}[\netsoro Master Theorem]{thm}{netsoroMasterTheorem}
\label{thm:netsoroMasterTheorem}
Fix any \netsoro program satisfying \cref{assm:equalDimNoNonlin+} and with all unwinded nonlinearities $\varphi^h$ controlled, for all H-vars $h$.
If $g^1, \ldots, g^M$ are all of the G-vars (including all input G-vars), then for any controlled $\psi: \R^M \to \R$, as $n \to \infty$,
\begin{align*}
    \f 1 n \sum_{\alpha=1}^n \psi(g^1_\alpha, \ldots, g^M_\alpha) \asto 
    \EV_{Z \sim \Gaus(\tmu, \tSigma)}\psi(Z)
    =
    \EV_{Z \sim \Gaus(\tmu, \tSigma)}\psi(Z^{g^1}, \ldots, Z^{g^M}),
\end{align*}
where $\asto$ means almost sure convergence,
$Z = (Z^{g^1}, \ldots, Z^{g^M}) \in \R^M$, and $\tmu = \{\tmu(g^i)\}_{i=1}^M \in \R^M$ and $\tSigma = \{\tSigma(g^i, g^j)\}_{i,j=1}^M \in \R^{M \times M}$ are given in \cref{eqn:extendedMuSigma2}.
See \cref{fig:mastertheoremIllustration} for an illustration.
\end{restatable}

\begin{thm}\label{thm:Netsoro+MasterTheorem}
Fix any \netsoroplus program satisfying \cref{assm:equalDimNoNonlin+} and \cref{assm:asRankStab}.
Suppose for each parametrized unwinded nonlinearity $\varphi^h(-; \bigvtheta)$, the parameters $\bigvtheta$ are instantiated with random variables that converge almost surely to some deterministic vector $\mathring{\bigvtheta}$ as $n \to \infty$, and assume $\varphi^h $ is parameter-controlled at $\mathring{\bigvtheta}$.
If $g^1, \ldots, g^M$ are all of the G-vars (including all input G-vars), then for any $l$, for any random vector $\bigtheta \in \R^l$ that converges almost surely to a deterministic vector $\mathring{\bigtheta}$, as $n \to \infty$, and for any $\psi: \R^M\times \R^l \to \R$ parameter-controlled at $\mathring{\bigtheta}$,
\begin{align*}
    \f 1 n \sum_{\alpha=1}^n \psi(g^1_\alpha, \ldots, g^M_\alpha; \bigtheta) \asto \EV_{Z \sim \Gaus(\tmu, \tSigma)}\psi(Z; \mathring{\bigtheta}),
\end{align*}
where $\asto$ means almost sure convergence,
$Z \in \R^M$, and $\tmu \in \R^M$ and $\tSigma \in \R^{M \times M}$ are given in \cref{eqn:extendedMuSigma2}, calculated by replacing each parametrized unwinded nonlinearity $\varphi(-; \bigvtheta)$ with parameterless nonlinearity $\varphi(-; \mathring{\bigvtheta})$.
\end{thm}

\section{Programs with Variable Dimensions}
\label{sec:VariableDim}

\paragraph{Notation}
In this section, we let $\dim(x)$ denote the dimension of an H-var $x$.

Before this section, we have mostly assumed that all dimensions in a \netsor (or \netsorplus) program are equal.
This is not necessary, and was done only to more quickly present the main ideas of this work.
In general, we can allow the H-vars in a program to vary in dimension, subject to the obvious dimensionality constraints imposed by the different rules:
\begin{equation}
\begin{cases}
\text{If $y := \sum_{i=1}^k a_i x^i$ or $y := \phi(x^1, \ldots, x^k)$, then the $\dim(y) = \dim(x^i)$ for each $i$.}
\\
\text{If $y := W x$ and $y' := W x'$, then $\dim(x) = \dim(x')$ and $\dim(y) = \dim(y')$.}
\end{cases}
\label{eqn:dimconstraint}
\end{equation}

\begin{defn}
Given an equivalence relation $\simeq$ on the input G-vars of a program, we extend this to an equivalence relation on all H-vars of the program by
\begin{equation}
h \equiv h' \iff h \simeq h' \text{ OR $h$ and $h'$ are constrained to have the same dimension by (\ref{eqn:dimconstraint})}.
\end{equation}
We call any such equivalence class a \emph{Common Dimension Class}, or CDC.
\end{defn}
Intuitively, the dimensions of H-vars in each CDC are all the same, and this common dimension is allowed to vary between CDCs.
\begin{exmp}
In \cref{tp:MLP}, the CDCs are $\{W^1 x, b^1, h^1, x^1\}$ and $\{b^2, v, \tilde h^2, h^2, x^2\}$.
In \cref{tp:RNN}, all G-vars are in the same CDC, and given the body of the program, this is the only way to partition the H-vars into CDCs, because the reuse of $W$ across time step ties all H-var dimensions to be equal.
\end{exmp}

\begin{assm}\label{assm:samplingVariableDim}
Fix a \netsor program with some equivalence relation on the input G-vars, and thus with induced CDCs over its H-vars.
Assume the dimensions in each CDC are the same, but the dimensions of different CDCs can vary.
Suppose for each A-var $W: \Atype(m', m)$, we sample $W_{\alpha \beta} \sim \Gaus(\sigma_W^2/m)$ for some $\sigma_W^2 > 0$.
Suppose further for each CDC $\cdc$ with dimension $n$, for each $\alpha \in [n]$, we sample, i.i.d., $\{x_\alpha: x \in \cdc \text{ and }x\text{ is input G-var}\} \sim \Gaus(\mu^\cdc, \Sigma^\cdc)$ for some mean $\mu^\cdc$ and covariance $\Sigma^\cdc$ over input G-vars in $\cdc$.
\end{assm}

Then the following result is an easy extension of \cref{thm:netsorMasterTheorem}.
\begin{restatable}[\netsor Master Theorem; Variable Dimensions]{thm}{netsorMasterTheoremVarDim}
\label{thm:netsorMasterTheoremVarDim}
Fix any \netsor program satisfying \cref{assm:samplingVariableDim} and with all nonlinearities controlled.
For any CDC $\cdc$, if $g^1, \ldots, g^M$ are all of the G-vars (including all input G-vars) in $\cdc$, then for any controlled $\psi: \R^M \to \R$, as all dimensions in the program tend to infinity (not just the dimension of $\cdc$) \footnote{Note that we do not require the dimensions of different CDCs to have a convergent, finite but nonzero, ratio},
\begin{align*}
    \f 1 n \sum_{\alpha=1}^n \psi(g^1_\alpha, \ldots, g^M_\alpha) \asto 
    \EV_{Z \sim \Gaus(\tmu^\cdc, \tSigma^\cdc)}\psi(Z)
    =
    \EV_{Z \sim \Gaus(\tmu^\cdc, \tSigma^\cdc)}\psi(Z^{g^1}, \ldots, Z^{g^M}),
\end{align*}
where $\asto$ means almost sure convergence,
$Z = (Z^{g^1}, \ldots, Z^{g^M}) \in \R^M$, and $\tmu^\cdc = \{\tmu^\cdc(g^i)\}_{i=1}^M \in \R^M$ and $\tSigma^\cdc = \{\tSigma^\cdc(g^i, g^j)\}_{i,j=1}^M \in \R^{M \times M}$ are given in \cref{eqn:extendedMuSigmaCDC}.
See \cref{fig:mastertheoremIllustration} for an illustration.
\end{restatable}

\begin{defn}
For any CDC $\cdc$ and G-vars $g, g'$ in $\cdc$, define recursively
\begin{align*}
    \tmu^\cdc(g)
        &=
            \begin{cases}
            \mu^\cdc(g)  &   \text{if $g$ is input}\\
            \sum_{i} a_i \tmu^\cdc(y^i)    &   \text{if $g = \sum_{i} a_i y^i$, introduced by \ref{linetype:lincomb}}\\
            0   &   \text{otherwise}
            \end{cases},
            \\
    \tSigma^\cdc(g, g')
        &=
            \begin{cases}
            \Sigma^\cdc(g, g')   &   \text{if $g, g'$ are inputs}\\
            \sum_{i} a_i \tSigma^\cdc(y^i, g')    &   \text{if $g = \sum_{i} a_i y^i$, introduced by \ref{linetype:lincomb}}\\
            \sum_{i} a_i \tSigma^\cdc(g, y^i)    &   \text{if $g' = \sum_{i} a_i y^i$, introduced by \ref{linetype:lincomb}}\\
            \sigma^2_W \EV_Z \varphi^h(Z) \varphi^{h'}(Z) &   \text{if $g = Wh, g'=Wh'$, introduced by \ref{linetype:MatMul} w/ same A-var $W$}\\
            0   &   \text{otherwise}
            \end{cases}
            \numberthis\label{eqn:extendedMuSigmaCDC}
\end{align*}
where $Z \sim \Gaus(\tmu^{\cdc'}, \tSigma^{\cdc'})$ with $\cdc'$ denoting the CDC of $h$ and $h'$.
\end{defn}

Essentially the same proof of \cref{thm:netsorMasterTheorem} goes through for \cref{thm:netsorMasterTheoremVarDim}, by noting that this proof only requires the minimum of all dimensions to go to infinity.

\section{Theoretical Tools}

\label{sec:proofs}

In this section, we list a series of theoretical tools needed to prove the Master Theorems.

\subsection{Probability Facts}

\paragraph{Notations}
Given two random variables $X, Y$, and a $\sigma$-algebra $\Aa$, the notation $X \disteq_\Aa Y$ means that for any integrable function $\phi$ and for any random varible $Z$ measurable on $\Aa$, $\EV \phi(X) Z = \EV \phi(Y)Z$.
We say that $X$ is distributed as (or is equal in distribution to) $Y$ conditional on $\Aa$.
In case $\Aa$ is the trivial $\sigma$-algebra, we just write $X \disteq Y$.
The expression $X \distto Y$ (resp. $X \asto Y$) means $X$ converges to $Y$ in distribution (resp. almost surely).

\begin{lemma}\label{lemma:momentBoundASConvergence}
Let $\{X_n\}_{n \ge 1}$ be a sequence of random variables with zero mean.
If for some $p \in \N$ and for all $n$, $\EV X_n^{2p} \le c n^{-1-\rho}$, for some $\rho > 0$, then $X_n \to 0$ almost surely.
\end{lemma}
\begin{proof}
By Markov's inequality, for any $\epsilon > 0$,
\begin{align*}
    \Pr(|X_n| > \epsilon)
        &=
            \Pr(X_n^{2p} > \epsilon^{2p})
        \le
            \EV X_n^{2p}/\epsilon^{2p}
        \le c n^{-1-\rho}/\epsilon^{2p}
        \\
    \sum_n \Pr(|X_n| > \epsilon)
        &\le
            \sum_n c n^{-1-\rho}/\epsilon^{2p}
        <    
            \infty.
\end{align*}
By Borel-Cantelli Lemma, almost surely, $|X_n| \le \epsilon$ for all large $n$.
Then, if we pick a sequence $\{\epsilon_k > 0\}_k$ converging to 0, we have that, almost surely, for each $k$, $|X_n| \le \epsilon_k$ for large enough $n$ --- i.e. almost surely, $X_n \to 0$.
\end{proof}

The following is a standard fact about multivariate Gaussian conditioning
\begin{prop}\label{prop:GaussianCondition}
Suppose $\R^{n_1 + n_2} \ni x \sim \Gaus(\mu, K)$, where we partition $x = (x_1, x_2) \in \R^{n_1} \times \R^{n_2}, \mu = (\mu_1, \mu_2) \in \R^{n_1} \times \R^{n_2}$, and $K = \begin{pmatrix} K_{11} & K_{12}\\ K_{21} & K_{22}\end{pmatrix}$.
Then
$x_1 \disteq_{x_2} \Gaus(\mu|_{x_2}, K|_{x_2})$
where
\begin{align*}
    \mu|_{x_2}
        &=
            \mu_1 - K_{12} K_{22}^+ (x_2 - \mu_2)\\
    K|_{x_2}
        &=
            K_{11} - K_{12} K_{22}^+ K_{21}.
\end{align*}

\end{prop}

\begin{lemma}[Stein's lemma]\label{lemma:stein}
For jointly Gaussian random variables $Z_1, Z_2$ with zero mean, and any function $\phi: \R \to \R$ where $\EV \phi'(Z_1)$ and $\EV Z_1 \phi(Z_2)$ exists, we have
\[\EV Z_1 \phi(Z_2) = \Cov(Z_1, Z_2) \EV \phi'(Z_2).\]
\end{lemma}

\begin{restatable}[Convergence of output vector to Gaussian given convergent 2nd moments]{prop}{gaussianDistConvFromMomentConv}
\label{prop:gaussianDistConvFromMomentConv}
Consider a sequence (in $t \in \N$) of collections of random vectors $\{x^{ab} \in \R^{n_a}\}_{b=1}^{r_a}, a = 1, \ldots, m,$ where $n_a$ and $x^{ab}$ can depend on $t$ but $m$ and $r_a$ are fixed.
Suppose as $t\to \infty$, $\f 1 {n_a} x^{ab}{}^\trsp x^{a b'} \distto \Sigma^\infty_{ab,ab'}$ for some deterministic PSD matrix $\Sigma^\infty = \{\Sigma^\infty_{ab,a'b'}\}_{a,b,a',b'}$.
If $v^a \sim \Gaus(0, \sigma_a^2 I)$ is sampled independently for each $a$, and independently from $\{x^{ab}\}_{a,b}$, then
\[
\{v^a{}^\trsp x^{ab} / \sqrt{n_a} \}_{a,b} \distto \Gaus(0, \Sigma)
\]
where the covariance $\Sigma = \{\Sigma_{ab, a'b'}\}_{a,b,a',b'}$ has
\[
\Sigma_{ab,a'b'}
    =
    \begin{cases}
    \sigma^2_a \Sigma^\infty_{ab,ab'}     &   \text{if $a = a'$}\\
    0                   &   \text{else.}
    \end{cases}
\]
\end{restatable}

\begin{proof}
WLOG, we assume $\sigma_a=1$ for all $a =1, \ldots, m$.
Let $f: \R^{\sum_a r_a} \to \R$ be a bounded continuous function.
We need to show that
\begin{align*}
    \EV f(\{v^a{}^\trsp x^{ab}\}_{a, b}) \to \EV_{Z \sim \Gaus(0, \Sigma)} f(Z)
\end{align*}
Note that if we define the PSD matrix $\hat \Sigma$ by $\hat \Sigma_{ab, ab'} = \f 1 {n_a} x^{ab}{}^\trsp x^{ab'}$ and $\hat \Sigma_{ab,a'b'} = 0$ if $a\ne a'$, then
\begin{align*}
        \EV f(\{v^a{}^\trsp x^{ab}\}_{a, b})
    =
        \EV_{\hat \Sigma} \EV_{Z \sim \Gaus(0, \hat \Sigma)}f(Z)
\end{align*}
where the distribution over $\hat \Sigma$ is induced by the distribution over $\{x^{ab}\}$.
The function $\tilde f(\hat\Sigma) := \EV_{Z \sim \Gaus(0, \hat \Sigma)}f(Z)$ is bounded because $f$ is bounded.
Thus, as $\hat \Sigma \distto \Sigma^\infty$ by assumption, we have
\begin{align*}
        \EV f(\{v^a{}^\trsp x^{ab}\}_{a, b})
    =
        \EV_{\hat \Sigma} \tilde f(\hat \Sigma)
    \to
        \tilde f(\Sigma^\infty)
    =
        \EV_{Z \sim \Gaus(0, \Sigma^\infty)}f(Z)
\end{align*}
as $t\to\infty$.
\end{proof}

\subsection{Review of Moore-Penrose Pseudoinverse}

We recall Moore-Penrose pseudoinverse and some properties of it.
\begin{defn}\label{defn:pseuodoinverse}
For $A \in \R^{n \times m}$, a pseudoinverse of $A$ is defined as a matrix $A^+ \in \R^{m \times n}$ that satisfies all of the following criteria
\begin{itemize}
    \item $A A^+ A = A$
    \item $A^+ A A^+ = A^+$
    \item $(AA^+)^\trsp = AA^+$
    \item $(A^+ A)^\trsp = A^+ A$
\end{itemize}
\end{defn}

The following facts are standard
\begin{itemize}
    \item if $A$ has real entries, then so does $A^+$.
    \item The pseudoinverse always exists and is unique.
    \item When $A$ is invertible, $A^+ = \inv A$.
    \item $(A^\trsp)^+ = (A^+)^\trsp$, which we denote as $A^{+\trsp}$.
    \item $A^+ = (A^\trsp A)^+ A^\trsp = A^\trsp (A A^\trsp)^+$.
    \item $AA^+$ is the orthogonal projector to the column space of $A$;
        $I - A^+ A$ is the orthogonal project to the null space of $A$.
    \item if $A$ has singular value decomposition $A = U\Lambda V$ where $U$ and $V$ are orthogonal and $\Lambda$ has the singular values on its diagonal, then $A^+ = V^\trsp \Lambda^+ U^\trsp$ where $\Lambda^+$ inverts all nonzero entries of $\Lambda$.
    \item For any collection of vectors $\{v_i\}_{i=1}^n$ in a Hilbert space, $w \mapsto \sum_{i,j=1}^n v_i (\Sigma^+)_{ij} \la v_j, w \ra $, where $\Sigma_{ij} = \la v_i, v_j \ra$, is the projection operator to the linear span of $\{v_i\}_{i=1}^n$.
\end{itemize}

\subsection{Gaussian Conditioning Trick}

The Gaussian conditioning trick was introduced by \citet{bolthausen_iterative_2012} for solving the TAP equation in statistical physics.
Later, this idea was used in \citet{bayati_dynamics_2011} to study the Approximate Message Passing algorithm in compressed sensing.

We present a slightly more general versions of lemmas from \citet{bayati_dynamics_2011} that deal with singular matrices.
\begin{lemma}\label{lemma:condTrickVec}
Let $z \in \R^n$ be a random vector with i.i.d. $\Gaus(0, \sigma^2)$ entries and let $D \in \R^{m\times n}$ be a linear operator.
Then for any constant vector $b \in \R^n$ the distribution of $z$ conditioned on $Dz = b$ satisfies:
\begin{align*}
    z
        &\disteq_{Dz = b}
            D^+ b + \Pi \tilde z
\end{align*}
where $D^+$ is the (Moore-Penrose) pseudoinverse, $\Pi$ is the orthogonal projection onto subspace $\{z: Dz = 0\}$, and $\tilde z$ is a random vector of i.i.d. $\Gaus(0, \sigma^2)$.
\end{lemma}
\begin{proof}
When $D = [I_{m \times m} | 0_{m \times {n-m}}]$, this claim is immediate.
By rotational symmetry, this shows that, for any vector space $\mathcal V$ and vector $v$ orthogonal to it, conditioning $z$ on $\mathcal V + v$ yields a Gaussian centered on $v$ with covariance determined by $\Pi_{\mathcal V} z$.
Then the lemma in the general case is implied by noting that $\{z: Dz = b\}$ can be decomposed as $\{z: Dz = 0 \} + D^+ b$.
\end{proof}

\begin{lemma}\label{lemma:condTrick}
Let $A \in \R^{n \times m}$ be a matrix with random Gaussian entries, $A_{ij} \sim \Gaus(0, \sigma^2)$.
Consider fixed matrices $Q \in \R^{m \times q}, Y \in \R^{n \times q}, P \in \R^{n \times p}, X \in \R^{m \times p}$.
Suppose there exists a solution in $A$ to the equations $Y = AQ$ and $X = A^\trsp P$.
Then the distribution of $A$ conditioned on $Y = AQ$ and $X = A^\trsp P$ is
\begin{align*}
    A &\disteq_{Y=AQ, X=A^\trsp P} E + \Pi_P^\perp \tilde A \Pi_Q^\perp
\end{align*}
where
\begin{align*}
    E
        &=
            Y Q^+
            + P^{+\trsp} X^\trsp
            - P^{+\trsp} P^\trsp
                YQ^+,
\end{align*}
$\tilde A$ is an iid copy of $A$,
and $\Pi_P^\perp = I - \Pi_P$ and $\Pi_Q^\perp = I - \Pi_Q$ in which $\Pi_P = PP^+$ and $\Pi_Q = QQ^+$ are the orthogonal projection to the space spanned by the column spaces of $P$ and $Q$ respectively.
\end{lemma}
\begin{proof}
We apply \cref{lemma:condTrickVec} to $D: A \mapsto (AQ, P^\trsp A)$.
The pseudoinverse of $D$ applied to $(Y, X^\trsp)$ can be formulated as the unique solution of
\begin{align*}
    \argmin_A \left\{ \|A\|^2_F : AQ = Y, P^\trsp A = X^\trsp \right\}
\end{align*}
where $\|-\|_F$ denotes Frobenius norm.
We check that $E$ is a 1) a solution to $AQ = Y, P^\trsp A = X^\trsp$ and 2) the minimal norm solution.

We have $EQ = 
        Y Q^+Q
            + P^{+\trsp} X^\trsp Q
            - P^{+\trsp} P^\trsp
                YQ^+Q$.
Note that $YQ^+Q = Y$ because $Y=AQ \implies YQ^+ Q = AQQ^+Q = AQ = Y$.
So $EQ = Y + P^{+T} (X^\trsp Q - P^\trsp Y)$.
But $X^\trsp Q = P^\trsp A Q = P^\trsp Y$, so $EQ = Y$ as desired.
A similar, but easier reasoning, gives $P^\trsp E = X^\trsp$.
This verifies that $E$ is a solution.

To check that $E$ is minimal norm, we show that it satisfies the stationarity of the Lagrangian
\begin{align*}
    L(A, \Theta, \Gamma)
        &=
            \|A\|^2_F + \la \Theta, Y - AQ \ra + \la \Gamma, X - A^\trsp P\ra.
\end{align*}
So $\pdf{L}{A} = 0 \implies 2A = \Theta Q^\trsp + P \Gamma^\trsp$ for some choices of $\Theta \in \R^{n \times q}$ and $\Gamma \in \R^{m \times p}$.
For $\Theta = 2 Y (Q^\trsp Q)^+$ and $\Gamma^\trsp = 2(P^\trsp P)^+ [ X^\trsp - P^\trsp Y Q^\trsp]$, we can check that
\begin{align*}
    \Theta Q^\trsp + P \Gamma^\trsp
        &=
            2 Y (Q^\trsp Q)^+ Q^\trsp + 2P(P^\trsp P)^+ [ X^\trsp - P^\trsp Y Q^+] \\
        &=
            2 Y Q^+ + 2 P^{+\trsp} X^\trsp - 2P^{+\trsp} P^\trsp Y Q^+\\
        &=
            2E
\end{align*}
as desired.
\end{proof}

\subsection{\texorpdfstring{$\alpha$}{Alpha}-Controlled Functions}

We generalize \cref{defn:controlled} slightly as follows.
\begin{defn}[$\alpha$-controlled]\label{defn:alphaControlled}
For $\alpha > 0$, a function $\phi: \R^k \to \R$ is said to be \textit{$\alpha$-controlled} if for some $C, c > 0$, $|\phi(x)| \le e^{C \sum_{i=1}^k |x_i|^{\alpha} + c}$ for all $x \in \R^k$.
\end{defn}

We present a few helper lemmas to facilitate our reasoning with $\alpha$-controlled functions.
The next lemma is easy to show using the equivalence of norms in finite dimensional Euclidean space.
\begin{lemma}
Let $\phi: \R^k \to \R$.
The following are equivalent
\begin{enumerate}
    \item $\phi$ is $\alpha$-controlled
    \item For some $p \ge 1$ and some $g(x) = o_{\|x\|_p \to \infty}(\|x\|_p^\alpha)$, $C, c > 0$, $|\phi(x)| \le e^{C \|x\|^{\alpha}_p + g(x)}$
    \item For all $p \ge 1$, there is some $C, c > 0$, $|\phi(x)| \le e^{C \|x\|^{\alpha}_p + c}$
\end{enumerate}
\end{lemma}

\newcommand{\alpExp}{\mathsf{C}}
\begin{lemma}\label{lemma:alpExp}
Let $\alpExp_\alpha^k: \R^{\ge 0} \to \R, c \mapsto \EV_{z \sim \Gaus(0, I_k)}e^{c \|z\|^\alpha_2} $.
Then
\begin{enumerate}
    \item $\alpExp_\alpha^k < \infty$ iff $\alpha < 2$
    \item for $\alpha \ge 1$,
        \begin{align*}
            \EV_{z\sim \Gaus(\mu, \Sigma)}e^{C \|z\|^\alpha_2} \le
                e^{C\|\mu\|^\alpha_2 } \alpExp_\alpha^k(C  \alpha \|\Sigma\|_2^{\alpha/2})
        \end{align*}
    where $\|\Sigma\|_2$ denotes the spectral norm of $\Sigma$.
    \item for any $\alpha$-controlled $\phi: \R^k \to \R$ with $\alpha \ge 1$, there is $C > 0$ such that for all $\mu \in \R^k$ and $k\times k$ PSD matrix $\Sigma$,
        \begin{align*}
            \EV_{z \sim \Gaus(\mu, \Sigma)} |\phi(z)|
                &\le
                    C e^{C\|\mu\|^\alpha_2 } \alpExp_\alpha^k(C \alpha \|\Sigma\|_2^{\alpha/2})
        \end{align*}
        where $\|\Sigma\|_2$ denotes the spectral norm of $\Sigma$.
\end{enumerate}
Note that the RHS is a montonic function in $\|\mu\|_2$ and $\|\Sigma\|_2$, in the sense that if $\|\mu\|_2$ and $\|\Sigma\|_2$ don't decrease, then the RHS will not decrease either.
\end{lemma}
\begin{proof}
The first claim is obvious and the third follows from the second easily.
For the second,
\begin{align*}
    \EV_{z\sim \Gaus(\mu, \Sigma)}e^{C \|z\|^\alpha_2}
        &\le
            \EV_{z\sim \Gaus(0, I)}e^{C \|\sqrt \Sigma z + \mu\|^\alpha_2}\\
        &\le
            \EV_{z\sim \Gaus(0, I)}e^{C \alpha \lp
                \|\sqrt \Sigma z\|^\alpha_2 + \|\mu\|^\alpha_2 
            \rp}\\
        &\le
            e^{C\|\mu\|^\alpha_2 }
            \EV_{z\sim \Gaus(0, I)}e^{C  \alpha \|\Sigma\|_2^{\alpha/2}
                \|z\|^\alpha_2}\\
        &=
            e^{C\|\mu\|^\alpha_2 } \alpExp_\alpha^k(C  \alpha \|\Sigma\|_2^{\alpha/2}).
\end{align*}
\end{proof}

\section{Proof of \texorpdfstring{\netsor\!}{Netsor} Master Theorem}
\label{sec:proofMasterTheorems}

In this section, we prove \cref{thm:netsorminMasterTheorem}, i.e.\ the Master Theorem for programs without \ref{linetype:lincomb}.
By \cref{remk:netsormin}, this would also show \cref{thm:netsorMasterTheorem}.

\renewcommand{\nu}{\eta}
\newcommand{\rankext}[1]{{\color{green}#1}}

\newcommand{\coreset}{{\mathcal{M}}}
\newcommand{\basespace}{\mathcal{U}}
\renewcommand{\MM}{m}

\paragraph{A Bit of Notation and Terminology}
Note that, for each $n$, the randomness of our program specified by \cref{thm:netsorMasterTheorem} comes from the sampling of the input variables.
Let $\basespace$ be the product space obtained from multiplying together the corresponding probability space for each $n$.
Each sample from this product probability space thus correspond to a sequence $\{S(n)\}_n$ of instantiatiations of input variables.
Below, when we say ``almost surely'' (often abbreviated ``a.s.''), we mean ``almost surely over the probability of $\basespace{}$.''
We will also often make statements of the form 
\begin{equation*}
\text{\emph{almost surely (or, a.s.), for all large $n$, \quad $\mathcal A(n)$ is true}}
\end{equation*}
where $\mathcal A(n)$ is a claim parametrized by $n$.
This means that for all but a $\basespace{}$-probability-zero set of sequences $\{S(n)\}_n$ of input variable instantiations, $\mathcal A(n)$ is true for large enough $n$.
Note that the order of the qualifiers is very important here.

\paragraph{We induct, but on what?}
A natural way of going about proving \cref{thm:netsorMasterTheorem} is by inducting on the number of variables in a program.
It turns out this is not enough to prove our claim in its full generality (see below), and it would be more fruitful to perform a simultaneous induction on our claim (\ref{IH:MomConv}) along with another statement, parametrized by $\MM$,
\begin{description}
\item[Moments\label{IH:MomConv}]\!\!\!$(\MM)$\ \ \ 
    For any controlled $\psi: \R^\MM \to \R$, as $n \to \infty$,
\begin{align*}
    \f 1 n \sum_{\alpha=1}^n \psi(g^1_\alpha, \ldots, g^\MM_\alpha) \asto \EV_{Z \sim \Gaus(\tmu, \tSigma)}\psi(Z).
\end{align*}
\item[CoreSet\label{IH:coreSet}]\!\!\!$(\MM)$\ \ \ 
    There exists a ``core set'' $\coreset \sbe [\MM]$ such that, 
 \begin{description}
    \item[Basis\label{prop:basis}]\!\!\!$(\MM)$\ \ \ 
    almost surely, for large enough $n$, for every $i \in [\MM]$, there exist \emph{unique} constants (not depending on $n$) $\{a_j\}_{j \in \coreset}$ such that $g^i = \sum_{j \in \coreset} a_j g^j$.
    Note the uniqueness implies that $\{g^i\}_{i \in \coreset}$ is linearly independent.
    \item[NullAvoid\label{prop:nullAvoid}]\!\!\!$(\MM)$\ \ \ 
    for every triangular array of Lesbegue measure zero sets $\{A_{n\alpha} \in \R^{\coreset} \}_{n \in \N, \alpha \in [n]}$, almost surely for all large enough $n$, for all $\alpha \in [n]$, we have
    \[\{g^i_\alpha\}_{i \in \coreset} \not \in A_{n \alpha}.\]
    In other words, the values $\{g^i_\alpha\}_{\alpha \in \coreset}$ of the core set ``avoid'' Lebesgue measure zero sets asymptotically.
    Intuitively, this says that the distribution of these values are not singular.
    (Note the LHS depends on $n$ although we are suppressing it notationally)
\end{description}
\end{description}

Let us explain in brief why we need to consider \ref{IH:coreSet} satisfying \ref{prop:basis} and \ref{prop:nullAvoid}.
\begin{itemize}
\item
    \ref{prop:basis} reduces the consideration of \ref{IH:MomConv} to only the core set G-vars, since every other G-var is asymptotically a linear combination of them.
\item
    When we apply the Gaussian conditioning technique \cref{prop:GaussianCondition}, we need to reason about the pseudo-inverse $\Lambda^+$ of some submatrix $\Lambda$ of a covariance matrix.
    Each entry of $\Lambda$ is of the form $\f 1 n \sum_{\alpha=1}^n \phi_i(g^1_\alpha, \ldots, g^{\MM-1}_\alpha) \phi_j(g^1_\alpha, \ldots, g^{\MM-1}_\alpha)$ for a collection of controlled scalar functions $\{\phi_i\}_i$.
    This $\Lambda$ will be a random variable which converges a.s.\ to a determinstic limit $\mathring \Lambda$  as $n \to \infty$.
    It should be generically true that $\Lambda^+ \asto \mathring \Lambda^+$ as well, which is essential to make the Gaussian conditioning argument go through.
    But in general, this is guaranteed only if $\Lambda$'s rank doesn't drop suddenly in the $n \to \infty$ limit.
    We thus need to guard against the possibility that $g^1, \ldots, g^\MM$, in the limit, suddenly concentrate on a small set on which $\{\phi_i(g^1, \ldots, g^\MM)\}_i$ are linearly dependent.
    This is where \ref{prop:nullAvoid} comes in.
    It tells us that $g^1, \ldots, g^\MM$ will avoid any such small set asymptotically, so that indeed the rank of $\Lambda$ will not drop in the limit.
\end{itemize}

\paragraph{Proof organization}

We will show that \ref{IH:MomConv} and \ref{IH:coreSet} are true for input variables, as the base case, and
\begin{equation*}
\text{\ref{IH:MomConv}}(\MM-1) \text{ and } \text{\ref{IH:coreSet}}(\MM-1) \implies \text{\ref{IH:MomConv}}(\MM) \text{ and } \text{\ref{IH:coreSet}}(\MM)
\end{equation*}
as the inductive step.
By induction, we obtain \ref{IH:MomConv}$(M)$, which is \cref{thm:netsorMasterTheorem}.

The base cases are easy and we will dispatch with them immediately after this in \cref{sec:basecases}, but the inductive step is much more complicated, and we will need to set up notation in \cref{sec:inductiveSetup}.
During this setup, we prove some basic limit theorems using the induction hypothesis.
However, the full generality of these claims requires some consequences of \ref{IH:coreSet}, which we call ``rank stability'' and ``zero stability'' (related to \cref{assm:asRankStab}).
These notions are introduced and proved in \cref{sec:rankStabilityZeroStability}.

We would then finally be able to handle the inductive steps at this point.
We first prove
\begin{equation*}
\text{\ref{IH:MomConv}$(\MM-1)$ and \ref{IH:coreSet}$(\MM-1)$} \implies \text{\ref{IH:coreSet}}(\MM)
\end{equation*}
in \cref{sec:inductiveCoreSet} because it is easier.
Then we prove
\begin{equation*}
\text{\ref{IH:MomConv}$(\MM-1)$ and \ref{IH:coreSet}$(\MM-1)$} \implies \text{\ref{IH:MomConv}}(\MM)
\end{equation*}
in \cref{sec:inductiveMoments}.

\subsection{Base Cases: \ref{IH:MomConv} and \ref{IH:coreSet} for Input Variables}
\label{sec:basecases}

\paragraph{Base case: \ref{IH:MomConv}(input vars)}
Suppose the input variables are $x^1, \ldots, x^k: \Gtype(n)$ (so that $\muin \in \R^k, \Sigmain \in \R^{k \times k}$).
We need to show that for any controlled function $\psi: \R^k \to \R$,
\begin{align*}
    \f 1 n \sum_{\alpha=1}^n \psi(x^1_\alpha, \ldots, x^k_\alpha) \asto \EV_{Z \sim \Gaus(\tmu, \tSigma)}\psi(Z),
\end{align*}
where $\psi$ on the RHS ignores all coordinates corresponding to non-input G-vars.
Since $\tmu$ and $\tSigma$ restricted to input variables are just $\muin$ and $\Sigmain$ (see \cref{eqn:extendedMuSigma}), the RHS expectation is just
\begin{align*}
    \EV_{Z \sim \Gaus(\tmu, \tSigma)}\psi(Z) = \EV_{Z^{\mathrm{in}} \sim \Gaus(\muin, \Sigmain)} \psi(Z^{\mathrm{in}})
\end{align*}
and the almost sure convergence we desire is just a result of the law of large numbers.

\paragraph{Base Case: \ref{IH:coreSet}(input vars)}
Let $x^1, \ldots, x^k$ be the input G-vars as above.
Pick the core set $\coreset$ to be any subset of $[k]$ such that $\rank \Sigmain|_\coreset = \rank \Sigmain$.
Then it's straightforward to verify \ref{prop:basis} and \ref{prop:nullAvoid}.

\subsection{Inductive Case: Setup}
\label{sec:inductiveSetup}
We now assume \ref{IH:MomConv}$(\MM-1)$ and \ref{IH:coreSet}$(\MM-1)$ and want to reason about $g^\MM$ to show \ref{IH:MomConv}$(\MM)$ and \ref{IH:coreSet}$(\MM)$.
Suppose
\begin{align*}
    g^{\MM} := A h \quad \text{where} \quad A: \Atype(n,n) \text{ and $h: \Htype(n)$ was introduced by $h := \phi(g^{1}, \ldots, g^{\MM - 1})$}
\end{align*}
(WLOG padding coordinates if necessary; if $h = g^i$ is a G-var, then pretend $\phi$ just projects to the $i$th coordinate).
For brevity, we will just write $g = g^{\MM}$.
Consider all previous instances where $A$ is used: 
\[\hat g^{i} := A \hat h^{i}, i = 1, \ldots, r.\]
Define
\begin{equation}
\hat G \defeq [\hat g^1| \ldots| \hat g^{r}] \in \R^{n \times r}, \hat H \defeq [\hat h^1| \ldots| \hat h^r]
.
\end{equation}
We will also use $\hat G$ to denote the \emph{set} of G-vars $\{\hat g^1, \ldots, \hat g^r\}$ when we later write expressions like $\tSigma(\hat G, \hat G)$.
Let $\Bb$ be the $\sigma$-algebra spanned by all previous G-vars $g^1, \ldots, g^{\MM-1}$ (and hence also all previous H-vars).
Conditioning on $\Bb$, $A$ is constrained by $\hat G = A\hat H$, and we have by \cref{lemma:condTrick},
\begin{align*}
    g \disteq_{\Bb} (\hat G \hat H^+ + \tilde A \Pi_{\hat H}^\perp) h
\end{align*}
where $\tilde A$ is an independent copy of $A$ and $\Pi_{\hat H} = \hat H \hat H^+ = \hat H(\hat H^\trsp \hat H)^+ \hat H^\trsp$ is the projection to the column space of $\hat H$.

If we define
\begin{align}
    \omega
        &\defeq
            \hat G \hat H^+ h,
            \quad
    \sigma \defeq
        \sigma_A
        \sqrt{\|\Pi_{\hat H}^\perp h\|^2/n}
    \label{eqn:meanvardef}
\end{align}
then
\begin{align}
    g \disteq_\Bb \omega + \sigma y,\ \text{with $y \sim \Gaus(0, I_n)$}
    \label{eqn:gConditionedOnB}
\end{align}
For brevity, we will define the following matrices and vectors of fixed dimension
\begin{equation}
\begin{aligned}
    \hat \Lambda
        &\defeq
            \hat H^\trsp \hat H/n \in \R^{r \times r}
            ,
            &
    \hat \nu
        &\defeq
            \hat H^\trsp h/n \in \R^{r}
            .
\end{aligned}
    \label{eqn:momentMatrices}
\end{equation}

Suppose $\hat h^i$ was introduced by $\hat h^i := \hat \phi^i(g^1, \ldots, g^M)$, where $\hat \phi^i$ depends at most on $g^1, \ldots, g^{\MM - 1}$.
By induction hypothesis \ref{IH:MomConv}$(\MM-1)$, $\hat \Lambda$ and $\hat \nu$ all converge a.s.\ to corresponding limit values $\mathring{\hat \Lambda}$ and $\mathring{\hat \nu}$, since their entries are moments of $Z^1, \ldots, Z^{\MM - 1}$:
\begin{align*}
    \hat \Lambda_{ij}
        &\asto
            \mathring{\hat \Lambda}_{ij}
        \defeq
            \EV \hat\phi^i(Z) \hat \phi^j(Z)
        =
            (\sigma_A)^{-2} \tSigma(\hat g^i, \hat g^j)
            \\
    \hat \nu_i
        &\asto
            \mathring{\hat \nu}_i
        \defeq
            \EV \hat \phi^i(Z) \phi(Z)
        =
            (\sigma_A)^{-2} \tSigma(\hat g^i, g)
        .
\end{align*}

It turns out that, as a consequence of \cref{lemma:rankStability} below, a.s.\ for all large enough $n$, $\rank \hat \Lambda = \rank \mathring {\hat \Lambda}$.
Therefore, as pseudoinverse is continuous on matrices of fixed rank,
we get the following proposition
\begin{prop}\label{prop:pseudoinverseLambda}
$\hat \Lambda^+ \asto \mathring{\hat \Lambda}^+$.
\end{prop}

Using this proposition, we compute the limits of the conditional mean $\omega$ and variance $\sigma^2$.
\begin{lemma}\label{lemma:sigmaConverges}
$\sigma^2 \asto \mathring \sigma^2 \defeq \tSigma(g, g) - \tSigma(g, \hat G) \tSigma(\hat G, \hat G)^+ \tSigma(\hat G, g)$
\end{lemma}
\begin{proof}
Note that
\begin{align*}
    \sigma^2 = \f {\sigma_A^2} n (h^\trsp h - h^\trsp \Pi_{\hat H} h)
        = \f {\sigma_A^2} n (h^\trsp h - h^\trsp {\hat H} (\hat H^\trsp \hat H)^+ \hat H^\trsp h)
        = \f {\sigma_A^2} n (h^\trsp h - \hat \nu^\trsp \hat \Lambda^+ \hat \nu).
\end{align*}
Because $\phi$ is polynomially-bounded, so is $\phi(z)^2$ as well.
By induction hypothesis,
\begin{align*}
    \f 1 n h^\trsp h = \f 1 n \sum_{\alpha = 1}^n \phi(g^1_\alpha, \ldots, g^{\MM - 1}_\alpha)^2
    \asto\EV_{Z \sim \Gaus(\tmu, \tSigma)}
        \phi(Z)^2 = 
    \sigma_A^{-2} \tSigma(g, g).
\end{align*}
Likewise, $\hat \nu \asto \mathring{\hat \nu}$ and $\hat \Lambda \asto \mathring{\hat \Lambda}$.
By \cref{prop:pseudoinverseLambda}, $\hat \Lambda^+ \asto \mathring{\hat \Lambda}^+$.
Combining all of these limits together yields the desired claim.
\end{proof}

\begin{lemma}\label{lemma:omegaExpansion}
Let $v \defeq \hat \Lambda^+ \hat \nu$, so that $v \asto \mathring v \defeq \mathring{\hat \Lambda}^+ \mathring{\hat \nu}.$
Then for some vector $\hat \varepsilon \in \R^r$ that go to 0 a.s.\ with $n$, $\omega = Eh = \hat G(\mathring v + \hat \varepsilon)$
\end{lemma}

\begin{proof}
Using \cref{eqn:momentMatrices}, we can re-express $\omega$ as
$
\omega
    =
        \hat G \hat \Lambda^+ \hat \nu
        .
$
By \cref{prop:pseudoinverseLambda}, $\hat \Lambda^+ \asto \mathring{\hat \Lambda}^+$, so that setting $\hat \varepsilon \defeq v - \mathring v $, we get $\hat \varepsilon \asto 0$.
Thus,
$
\omega = \hat G(\mathring v + \hat \varepsilon)
$
as desired.
\end{proof}

\subsection{Rank Stability and Zero Stability}
\label{sec:rankStabilityZeroStability}
In this section, we prove the following consequence of \ref{IH:coreSet}$(\MM-1)$ and \ref{IH:MomConv}$(\MM-1)$.

\begin{lemma}[Rank Stability]\label{lemma:rankStability}
For any collection of controlled functions $\{\psi_j: \R^{\MM-1} \to \R\}_{j=1}^l$, let $K \in \R^{l \times l}$ be the random matrix (depending on $n$) defined by
\begin{equation*}
K_{ij} = \f 1 n \sum_{\alpha=1}^n \psi_i(g^1_\alpha, \ldots, g^{\MM-1}_\alpha) \psi_j(g^1_\alpha, \ldots, g^{\MM-1}_\alpha).
\end{equation*}
By \ref{IH:MomConv}$(\MM-1)$,
\[K \asto \mathring K\]
for some matrix $\mathring K \in \R^{l \times l}$.
\begin{enumerate}
\item
    Then, almost surely, for large enough $n$,
    \begin{equation*}
    \ker K = \ker \mathring K, \quad \im K = \im \mathring K, \quad\text{and}\quad \rank K = \rank \mathring K.
    \end{equation*}
    Here $\ker$ denotes null space and $\im$ denotes image space.
\item
    Suppose $I \sbe [l]$ is any subset such that $\mathring K|_I$, the restriction of $\mathring K$ to rows and columns corresponding to $I$, satisfies
    \[|I| = \rank \mathring K|_I = \rank \mathring K.\]
    There are unique coefficients $\{F_{ij}\}_{i \in [l], j \in I}$ that expresses each row of $\mathring K$ as linear combinations of rows corresponding to $I$:
    \[
    \forall i \in [l],\quad 
    \mathring K_i = \sum_{j \in I} F_{ij} \mathring K_j.
    \]
    Then, a.s.\ for all large $n$, for all $\alpha \in [n]$,
    \begin{equation*}
    \psi_i(g^1_\alpha, \ldots, g^{\MM-1}_\alpha)
    = \sum_{j \in I} F_{ij} \psi_j(g^1_\alpha, \ldots, g^{\MM-1}_\alpha).
    \end{equation*}
\end{enumerate}

\end{lemma}

This will be primarily a corollary of the following \cref{lemma:zerofunStability}.

\begin{lemma}[Zero Stability]\label{lemma:zerofunStability}
If $\psi: \R^{\MM-1} \to \R^{\ge 0}$ is a nonnegative function such that
\begin{align*}
\f 1 n \sum_{\alpha = 1}^n \psi(g^1_\alpha, \ldots, g^{\MM-1}_\alpha) \asto 0
\end{align*}
then, almost surely, for large enough $n$,
\[\psi(g^1_\alpha, \ldots, g^{\MM-1}_\alpha) = 0\]
for all $\alpha \in [n]$.
\end{lemma}

We give the proof of \cref{lemma:rankStability} now, assuming \cref{lemma:zerofunStability}.
\begin{proof}
Let $v \in \R^l$ be in the null space of $\mathring K$, i.e. $v^\trsp \mathring K v = 0$.
Then we also have $v^\trsp K v \asto v^\trsp \mathring K v = 0$.
But
\begin{align*}
v^\trsp K v
    &=
        \f 1 n \sum_{\alpha=1}^n \Psi(g^1_\alpha, \ldots, g^{\MM-1}_\alpha)
        ,
        \quad
        \text{where}
        \quad
        \Psi(g^1_\alpha, \ldots, g^{\MM-1}_\alpha) \defeq
        \lp
            \sum_{i=1} v_i \psi_i(g^1_\alpha, \ldots, g^{\MM-1}_\alpha)
        \rp^2
\end{align*}
and $\Psi$ is a nonnegative function.
By \cref{lemma:zerofunStability}, we have that: almost surely, for large enough $n$, 
\[\Psi(g^1_\alpha, \ldots, g^{\MM-1}_\alpha) = 0 \quad \text{for all $\alpha \in [n]$}
\quad
\implies v^\trsp K v = 0\]
\textit{Claim 1.}\ \ 
If we apply this argument to a basis $\{v^1, \ldots, v^t\}$ of $\ker \mathring K$, then we get, 
\[\text{a.s.\ for all large $n$,}\quad 
\ker \mathring K \sbe \ker K,\]
so that
\[\text{a.s.\ for all large $n$,}\quad 
\rank \mathring K \ge \rank K.\]
Because the rank function is lower semicontinuous (i.e.\ the rank can drop suddenly, but cannot increase suddenly), and $K \asto \mathring K$, we also have
\[\text{a.s.\ for all large $n$,}\quad 
\rank \mathring K \le \rank K.\]
Combined with the above, this gives the desired result on rank.
The equality of null space then follows from the equality of rank, and the equality of image space follows immediately, as the image space is the orthogonal complement of the null space.

\textit{Claim 2.}\ \ 
If we apply the above argument to each $v^i$ defined by inner product as
\[\forall x \in \R^l,\quad x^\trsp v^i \defeq x_i - \sum_{j \in I} F_{ij} x_j,
\]
(note that only for $i \not \in I$ is $v^i$ nonzero),
then we have, a.s.\ for large $n$, $v^i{}^\trsp K v^i = 0$, or
\begin{equation*}
    \psi_i(g^1_\alpha, \ldots, g^{\MM-1}_\alpha)
    = \sum_{j \in I} F_{ij} \psi_j(g^1_\alpha, \ldots, g^{\MM-1}_\alpha).
\end{equation*}
\end{proof}

In the rest of this section, we prove \cref{lemma:zerofunStability}.
It helps to first show that the linear relations given in \ref{prop:basis} carries over to the $n\to\infty$ limit.
\begin{prop}\label{prop:limitbasis}
Let $\tSigma|_\coreset$ be the submatrix of $\tSigma$ with rows and columns corresponding to $\{g^i: i \in \coreset\}$.
Then $\rank \tSigma = \rank \tSigma|_\coreset = |\coreset|$.
Furthermore, if $Z = (Z^1, \ldots, Z^{\MM-1}) \sim \Gaus(\tmu|_{\MM-1}, \tSigma|_{\MM-1})$, where $\tmu|_{\MM-1}, \tSigma|_{\MM-1}$ are the restrictions of $\tmu, \tSigma$ to $g^1, \ldots, g^{\MM-1}$, then
\[Z^i \disteq \sum_{j \in \coreset} a_j Z^j\]
where $\{a_j\}_{j \in \coreset}$ are the coefficients corresponding to $g^i$ given in \ref{prop:basis}.
\end{prop}
\begin{proof}
By \ref{prop:basis} property, each $g^i, i \in \coreset$, has a set of unique constants $\{a_j\}_{j \in \coreset}$ (independent of $n$) such that, almost surely, for large enough $n$,
\[g^i = \sum_{j \in \coreset} a_j g^j.\]
Let $\psi(x^1, \ldots, x^{\MM-1}) \defeq (x^i - \sum_{j \in \coreset} a_j x^j)^2$.
Then by \ref{prop:basis}$(\MM-1)$ and \ref{IH:MomConv}$(\MM-1)$,
\begin{align*}
\f 1 n \sum_{\alpha=1}^n \psi(g^1_\alpha, \ldots, g^{\MM-1}_\alpha) \asto \EV_{Z \sim \Gaus(\tmu|_{\MM-1}, \tSigma|_{\MM-1})} \psi(Z) = 0.
\end{align*}
where $\tmu|_{\MM-1}, \tSigma|_{\MM-1}$ are the restrictions of $\tmu, \tSigma$ to $g^1, \ldots, g^{\MM-1}$.
This implies that for $Z = (Z^1, \ldots, Z^{\MM-1}) \sim \Gaus(\tmu|_{\MM-1}, \tSigma|_{\MM-1})$,
\[Z^i \disteq \sum_{j \in \coreset} a_j Z^j.\]
Repeating this argument for all $i \in [{\MM-1}]$ implies that $\{Z^j\}_{j \in \coreset}$ is a ``spanning set'' of $Z^1, \ldots, Z^{\MM-1}$.
Furthermore, by the uniqueness of the coefficients, we also have that $\{Z^j\}_{j \in \coreset}$ is linearly independent as well.
This then implies the rank consequence we want.
\end{proof}

Now we show \cref{lemma:zerofunStability}.

\begin{proof}[Proof of \cref{lemma:zerofunStability}]
By \ref{IH:MomConv}$(\MM-1)$,
\begin{align*}
\f 1 n \sum_{\alpha = 1}^n \psi(g^1_\alpha, \ldots, g^{\MM-1}_\alpha) \to \EV_{Z \sim \Gaus(\tmu|_{\MM-1}, \tSigma|_{\MM-1})} \psi(Z).
\end{align*}
By \cref{prop:limitbasis}, if $Z \sim \Gaus(\tmu|_{\MM-1}, \tSigma|_{\MM-1})$ and $Z|_\coreset$ is the part of $Z$ corresponding to $\coreset$, then 
\begin{description}
\item[$Z|_\coreset$ has density.]
    The law of $Z|_\coreset$ (namely $\Gaus(\tmu|_\coreset, \tSigma|_\coreset)$, where $\tmu|_\coreset, \tSigma|_\coreset$ are the restriction of $\tmu$ and $\tSigma$ to $\coreset$) is absolutely continuous against the Lebesgue measure of $\R^\coreset$ and vice versa, so that a set of Lebesgue measure zero is measure zero under $\Gaus(\tmu|_\coreset, \tSigma|_\coreset)$, and vice versa; and
\item[$Z|_\coreset$ is basis of $Z$.]
    \ref{prop:basis} yields a linear function $\lambda$ such that $\lambda(\{g^j_\alpha\}_{j \in \coreset}) = \{g^i_\alpha\}_{i=1}^{m-1}$ for all $\alpha$, almost surely asymptotically, and
    $\lambda(Z|_\coreset) \disteq Z$, so that
    \begin{equation*}
    \EV_{Z \sim \Gaus(\tmu|_{\MM-1}, \tSigma|_{\MM-1})} \psi(Z) = \EV_{Z' \sim \Gaus(\tmu|_\coreset, \tSigma|_\coreset)} \psi \circ \lambda(Z').
    \end{equation*}
    This expectation is 0 by our premise.
\end{description}
Because $\psi$, and thus $\psi \circ \lambda$, is a nonnegative function, the nullity of the expectation implies that, other than a set $U$ of $\Gaus(\tmu|_\coreset, \tSigma|_\coreset)$-measure zero, $\psi\circ \lambda$ is 0.
This set $U$ also has Lebesgue measure zero as $Z|_\coreset$ has density, by our reasoning above.

If in \ref{prop:nullAvoid}, we set $A_{n\alpha} = U$ for all $n$ and all $\alpha \in [n]$, then we get that: almost surely, for all large enough $n$, for all $\alpha \in [n]$,

\begin{equation*}
\{g^i_\alpha\}_{i \in \coreset} \not\in U
\iff
\psi\circ \lambda(\{g^i_\alpha\}_{i \in \coreset}) = 0
\iff
\psi(g^1_\alpha, \ldots, g^{\MM-1}_\alpha) = 0,
\end{equation*}
as desired.
\end{proof}

\subsection{Inductive Step: \ref{IH:coreSet}\texorpdfstring{$(\MM)$}{(m)}}
\label{sec:inductiveCoreSet}
In this section, we show
\begin{equation*}
\text{\ref{IH:MomConv}$(\MM-1)$ and \ref{IH:coreSet}$(\MM-1)$} \implies \text{\ref{IH:coreSet}}(\MM).
\end{equation*}
More explicitly, we need to think about whether to add $\MM$ to the core set $\coreset$ of $[\MM-1]$ in order to maintain the \ref{prop:basis} and \ref{prop:nullAvoid} properties.

We proceed by casework on whether $\mathring \sigma = 0$.

\newcommand{\Lsq}{\mathcal{L}}
\subsubsection{If \texorpdfstring{$\mathring \sigma = 0$}{sigma Converges to 0 a.s.}}

We will show that the core set properties are maintained if we don't add $m$ to the core set.

Consider the space $\Lsq \defeq L^2(\Gaus(\tmu|_\coreset, \tSigma|_\coreset))$ of square-integrable real functions against the measure $\Gaus(\tmu|_\coreset, \tSigma|_\coreset)$ defined on $\R^{\coreset}$.
Let $\la \phi, \psi \ra = \EV_{Y \sim \Gaus(\tmu|_\coreset, \tSigma|_\coreset)} \phi(Y) \psi(Y)$ be the inner product of this space.
Just like in a finite-dimensional inner product space, given a finite collection of functions $S = \{\psi^i\}_{i=1}^k$, the orthogonal projection operator $\Pi_S$ to the span of $S$ (inside $\Lsq$) is given by
\[\Pi_S \phi = \sum_{i=1}^k a_i \psi^i,
\]
for any $\phi \in \Lsq$, where
\begin{align*}
a &= \Lambda^+ b \in \R^k,\\
b_j &= \la \psi^j, \phi \ra, b \in \R^k,\\
\Lambda_{ij} &= \la \psi^i, \psi^j\ra, \Lambda \in \R^{k \times k}.
\end{align*}

Recall that $g = Ah$ where $h$ was introduced by $h := \phi(g^{1}, \ldots, g^{\MM - 1})$, for some controlled $\phi$, and likewise $\hat g^i = A \hat h^i$ where $\hat h^i = \hat \phi^i(g^1, \ldots, g^{\MM-1})$, for each $i \in [r]$.
By \ref{prop:basis}, we know that, a.s.\ for large enough $n$, each of $g^1, \ldots, g^{\MM-1}$ is a (unique, constant-in-$n$) linear combination of $\{g^j\}_{j \in \coreset}$.
Therefore, we can express
\[h = \underline\phi(\{g^j\}_{j \in \coreset}),\quad\text{and}\quad
\forall i \in [r], \hat h^i = \underline {\hat \phi^i}(\{g^j\}_{j \in \coreset})
\]
for some functions $\underline \phi, \underline {\hat \phi^i} \in \Lsq$.
For convenience, set $S \defeq \{\underline{\hat \phi^i}\}_i.$

One can see then,
as in the proof of \cref{lemma:sigmaConverges},
\[
\mathring\sigma^2
    =
        \sigma_A^2 (\EV \phi(Z)^2 - \mathring{\hat \nu}^\trsp \mathring{\hat \Lambda}^+ \mathring{\hat \nu})
        \\
    =
        \sigma_A^2 (\la \underline\phi, \underline\phi \ra - \la \underline\phi, \Pi_{S} \underline \phi \ra)
\]
by expanding the definition of $\mathring{\hat \nu}$ and $\mathring{\hat \Lambda}$.
Therefore, $\mathring \sigma = 0$ implies that 
\[\la \underline\phi, \underline\phi \ra = \la \underline\phi, \Pi_{S} \underline \phi \ra
\]
so that: after changing its values on a set $U$ of measure zero under $\Gaus(\tmu|_\coreset, \tSigma|_\coreset)$ (and thus also under Lebesgue measure by \cref{lemma:rankStability}), $\underline \phi$ is a linear combination of $\{\underline{\hat \phi^i}\}_{i=1}^r$, i.e.
\begin{equation*}
\forall \vec x \not \in U, \underline\phi(\vec x) = \sum_{i \in [r]} c_i \underline{\hat \phi^i}(\vec x)
\end{equation*}
for some coefficients $\{c_i\}_{i \in [r]}$.
By \ref{prop:nullAvoid} applied to $A_{n\alpha} = U$ for all $n$ and $\alpha \in [n]$, we also have that: a.s.\ for large enough $n$,
\[
\phi(g^1, \ldots, g^\alpha) = \underline \phi(\{g^j\}_{j \in \coreset})
= \sum_{i \in [r]} c_i \underline{\hat \phi^i}(\{g^j\}_{j \in \coreset}) 
= \sum_{i \in [r]} c_i \hat \phi^i(g^1, \ldots, g^\alpha)
,
\]
and therefore, under the same condition, (recall $A$ is the matrix giving rise to $g$ in $g:= A h$)
\[
g = A \phi(g^1, \ldots, g^\alpha)
= \sum_{i \in [r]} c_i A \hat \phi^i(g^1, \ldots, g^\alpha)
= \sum_{i \in [r]} c_i \hat g^i.
\]
This shows that, if we keep the core set as $\coreset$, then \ref{prop:basis} is still satisfied.
Since the core set is not changing, \ref{prop:nullAvoid} just follows from the induction hypothesis.

For usage later in the proof of \ref{IH:MomConv}$(\MM)$, we record our observation here as follows
\begin{lemma}\label{lemma:limitSigmaIsZero}
If $\mathring \sigma = 0$, then there are coefficients $\{c_i\}_{i=1}^r$ such that a.s.\ for large enough $n$,
\[
g = \sum_{i \in [r]} c_i \hat g^i.
\]
\end{lemma}

\subsubsection{If \texorpdfstring{$\mathring \sigma > 0$}{sigma Converges to Nonzero Value a.s.}}

It's clear that $g$ cannot be in the linear span of $\{\hat g^i\}_{i\in[r]}$ asymptotically, so we will add $g$ to the core set, and the \ref{prop:basis} property follows immediately.
In the below, we shall write $\coreset$ for the old core set, and $\coreset' \defeq \coreset \cup \{g\}$ for the new one.

It remains to show \ref{prop:nullAvoid} for $\coreset'$.
Because the conditional variance of $g^\MM_\alpha$ given $g^1, \ldots, g^{\MM-1}$ is $\sigma^2$, and because $\mathring \sigma > 0$, this assumption implies that, a.s.\ for all large enough $n$,
\begin{equation}
\text{$g^\MM_\alpha|g^1, \ldots, g^{\MM-1}$ has density for all $\alpha \in [n]$.}
\label{eqn:conditionalDistributionHasDensity}
\end{equation}
By ``has density'' here, we in particular mean that any Lesbegue measure zero set in $\R$ has zero probability under the conditional distribution of $g^\MM_\alpha$ given $g^1, \ldots, g^{\MM-1}$.

Now, to prove \ref{prop:nullAvoid} holds for $\coreset'$:
Let $\{A_{n\alpha} \sbe \R^{\coreset'}\}_{n\in\N, \alpha \in [n]}$ be a triangular array of Lesbegue measure zero sets.
For each $A_{n \alpha}$, define $B_{n\alpha} \defeq \{\vec x \in \R^{\coreset}: \lambda(A_{n\alpha}|_{\vec x}) \ne 0\}$, where $A_{n\alpha}|_{\vec x} = \{y \in \R: (\vec x, y) \in A_{n\alpha} \sbe \R^{\coreset} \times \R\}$ is the ``slice'' of $A_{n\alpha}$ at $\vec x$, and $\lambda$ is the 1-dimensional Lebesgue measure.
Because each $A_{n\alpha}$ has measure zero in $\R^{\coreset'}$, necessarily each $B_{n\alpha}$ also has measure zero in $\R^{\coreset}$.
Applying \ref{prop:nullAvoid} to the triangular array $\{B_{n\alpha} \sbe \R^{\coreset}\}_{n \in \N, \alpha \in [n]}$, we get that: a.s.\ for large enough $n$,
\[
\forall \alpha \in [n], \{g^i_\alpha\}_{i \in \coreset} \not \in B_{n\alpha}.\]
Therefore, by \cref{eqn:conditionalDistributionHasDensity}, a.s.\ for large enough $n$,
\[
\forall \alpha \in [n], \{g^i_\alpha\}_{i \in \coreset'} \not \in A_{n\alpha}.
\]
This finishes the proof of \ref{prop:nullAvoid} for $\coreset'$, and also \ref{IH:coreSet}$(\MM)$.

\begin{lemma}\label{lemma:maxbound}
Assume \ref{IH:MomConv}$(\MM-1)$.
Suppose $\psi: \R^{\MM-1} \to \R$ is controlled.
Then as $n \to \infty,$
\[\f 1 {n^p} \max_{\alpha \in [n]} |\psi(g^1_\alpha, \ldots, g^{\MM-1}_\alpha)| \asto 0\]
for any $p > 0$.
\end{lemma}
\begin{proof}
For any $q > 0$, we have the elementary bound
\begin{equation*}
\max_{\alpha \in [n]} |\psi(g^1_\alpha, \ldots, g^{\MM-1}_\alpha)|
\le
\sqrt[q]{\sum_{\alpha \in [n]}
|\psi(g^1_\alpha, \ldots, g^{\MM-1}_\alpha)|^q}.
\end{equation*}
Thus, for any $q > 0$,
\begin{align*}
\f 1 {n^p} \max_{\alpha \in [n]} |\psi(g^1_\alpha, \ldots, g^{\MM-1}_\alpha)|
&\le
    \f 1 {n^{p-1/q}} 
    \sqrt[q]{\f 1 n
        \sum_{\alpha \in [n]}
        |\psi(g^1_\alpha, \ldots, g^{\MM-1}_\alpha)|^q}.
\end{align*}
Because, by \ref{IH:MomConv}$(\MM-1)$, $\f 1 n
        \sum_{\alpha \in [n]}
        |\psi(g^1_\alpha, \ldots, g^{\MM-1}_\alpha)|^q \asto C$ for some constant $C$ as $n \to \infty$,
the RHS above converges a.s.\ to 0 as soon as we take $q > 1/p$, and therefore so does the LHS.

\end{proof}

\subsection{Inductive Step: \ref{IH:MomConv}\texorpdfstring{$(\MM)$}{(m)}}
\label{sec:inductiveMoments}
In this section, we show
\begin{equation*}
\text{\ref{IH:MomConv}$(\MM-1)$ and \ref{IH:coreSet}$(\MM-1)$} \implies \text{\ref{IH:MomConv}}(\MM).
\end{equation*}

More specifically,
we will show that for any controlled $\psi: \R^{\MM} \to \R$,
\begin{align*}
    \f 1 n \sum_{\alpha=1}^n \psi(g^1_\alpha, \ldots, g^{\MM}_\alpha) 
    \asto
    \EV_{Z \sim \Gaus(\tmu, \tSigma)} \psi(Z)
\end{align*}
where again on the RHS $\psi$ ignores all coordinates $Z^{\MM +1},\ldots, Z^M$ (corresponding to $g^{\MM +1}, \ldots, g^{M}$).

By \cref{lemma:limitSigmaIsZero}, if $\mathring \sigma = 0$, then almost surely, for large enough $n$, $g = g^\MM$ is just a (fixed) linear combination of $g^1, \ldots, g^{\MM-1}$, so \ref{IH:MomConv} is trivially true.
Therefore, in the below, we assume 
\begin{equation}
\mathring \sigma > 0.
\label{assm:mathringSigmaPositive}
\tag{$\star$}
\end{equation}
This assumption will be crucial for our arguments involving smoothness induced by Gaussian averaging.

\newcommand{\probA}{\mathsf{A}}
\newcommand{\probB}{\mathsf{B}}
\newcommand{\probC}{\mathsf{C}}
\newcommand{\probD}{\mathsf{D}}

\newcommand{\EVbr}[2][]{\EV_{#1}\left[#2\right]}
\newcommand{\EVcond}[2]{\EV\left[\left.#1\right| #2 \right]}
\newcommand{\extcom}[1]{{\color{blue}{#1}}}
\newcommand{\exttcom}[1]{{\color{olive}{#1}}}

To clarify notation in the following, we will write $\EVbr[X]{expression}$ to denote the expectation over only the randomization in $X$, and $\EVcond{expression}{\Bb}$ to denote the expectation taken over all randomness except those in $\Bb$.

\paragraph{Proof Plan}
Note that
\begin{align*}
    &\phantomeq
        \left|\f 1 n \sum_{\alpha = 1}^n \psi(g^1_\alpha, \ldots, g^{\MM}_\alpha)
        - \EV_{Z \sim \Gaus(\tmu, \tSigma)} \psi(Z)
        \right|
    \le
        \probA + \probB + \probC
        \numberthis\label{eqn:decompABC}
\end{align*}
where
\begin{align*}
    \probA
        &\defeq
            \left|\f 1 n \sum_{\alpha = 1}^n \psi(g^1_\alpha, \ldots, g^{\MM}_\alpha)
            - \EV_z\psi\left(g^1_\alpha, \ldots, g^{\MM-1}_\alpha, \omega_\alpha + \sigma z \right)\right|
            \\
    \probB
        &\defeq
            \left|
            \f 1 n \sum_{\alpha = 1}^n \EV_z\psi\left(g^1_\alpha, \ldots, g^{\MM-1}_\alpha, \omega_\alpha + \sigma z \right)
            -
            \EV_z
                {
                    \psi\lp
                        g^1_\alpha, \ldots, g^{\MM-1}_\alpha,
                        \sum_{i=1}^r \mathring v_i \hat g^i_\alpha + \mathring \sigma z
                        \rp
                }
            \right|
            \\
    \probC
        &\defeq
            \left|
            \f 1 n \sum_{\alpha = 1}^n
            \EV_z
                {
                    \psi\lp
                        g^1_\alpha, \ldots, g^{\MM-1}_\alpha,
                        \sum_{i=1}^r \mathring v_i \hat g^i_\alpha + \mathring \sigma z
                        \rp
                }
            -
            \EV_{Z \sim \Gaus(\tmu, \tSigma)} \psi(Z)
            \right|
\end{align*}
with $z  \sim \Gaus(0, 1)$.
Note that $\probB$ and $\probC$ are random variables in $\Bb$.
We will show that each of $\probA, \probB, \probC$ goes to 0 almost surely, which would finish the proof of \cref{thm:netsorMasterTheorem}.

Roughly speaking, $\probA \asto 0$ because of a law of large number, $\probB \asto 0$ because of the smoothness in $\EV_z \psi$ induced by Gaussian averaging, and $\probC \asto 0$ by induction hypothesis.
We start with the last item, since it's the easiest.

\subsubsection{\texorpdfstring{$\probC$}{C} Converges Almost Surely to 0}

In this section we show that $\probC \asto 0$ by a straightforward reduction to the inductive hypothesis.

Let $\hat Z^1, \ldots, \hat Z^r$ be the components of $Z \sim \Gaus(\tmu, \tSigma)$ corresponding to $\hat g^1, \ldots, \hat g^r$, and let $\hat Z$ be the column vector with these entries.
Note that, by \cref{prop:GaussianCondition}, $Z^{\MM}$ (corresponding to $g^{\MM}$), conditioned on $Z^1, \ldots, Z^{\MM-1}$, is distributed as a Gaussian with mean $
\tSigma(g, \hat G)\tSigma(\hat G, \hat G)^+  \hat Z
    = \mathring{\hat \nu}^\trsp \mathring{\hat \Lambda}^+  \hat Z
    = \mathring v^\trsp \hat Z$
and variance
$\tSigma(g, g) - \tSigma(g, \hat G) \tSigma(\hat G, \hat G)^+ \tSigma(\hat G, g)
    = \mathring \sigma$.
Thus
\begin{align*}
\EV_{Z} \psi(Z)
    &=
        \EV_{Z^1, \ldots, Z^{\MM - 1}} \EV[\psi(Z) | Z^1, \ldots, Z^{\MM - 1}]
        \\
    &=
        \EV_{Z^1, \ldots, Z^{\MM - 1}}
        \EV_{z \sim \Gaus(0, 1)}\psi(Z^1,\ldots, Z^{\MM - 1}, \mathring v^\trsp \hat Z + \mathring \sigma z)
        \\
    &=
        \EV_{Z^1, \ldots, Z^{\MM - 1}}
        \Psi(Z^1, \ldots, Z^{\MM - 1})
\end{align*}
where we have set $\Psi(Z^1, \ldots, Z^{\MM - 1}) \defeq \EV_{z \sim \Gaus(0, 1)}\psi(Z^1,\ldots, Z^{\MM - 1}, \mathring v^\trsp \hat Z + \mathring \sigma z)$.
$\Psi$ is a controlled function since $\psi$ is.
Applying the induction hypothesis to $\Psi$, we obtain
\begin{align*}
    &\phantomeq
        \f 1 n \sum_{\alpha = 1}^n
            \EV_z
                {
                    \psi\lp
                        g^1_\alpha, \ldots, g^{\MM-1}_\alpha,
                        \sum_{i=1}^r \mathring v_i \hat g^i_\alpha + \mathring \sigma z
                        \rp
                }
        \\
    &=
        \f 1 n \sum_{\alpha = 1}^n
                {
                    \Psi\lp
                        g^1_\alpha, \ldots, g^{\MM-1}_\alpha
                        \rp
                }
        \\
    &\asto
        \EV_{Z^1, \ldots, Z^{\MM-1}} \Psi(Z^1, \ldots, Z^{\MM-1})
        \\
    &\pushright{\text{by induction hypothesis}}
        \\
    &=
        \EV_{Z^1, \ldots, Z^{\MM - 1}}
        \EV_{z \sim \Gaus(0, 1)}\psi(Z^1,\ldots, Z^{\MM - 1}, \mathring{v}^\trsp \hat Z + \mathring \sigma z)
        \\
    &=
        \EV_Z \psi(Z)
\end{align*}
as desired.

\subsubsection{\texorpdfstring{$\probA$}{A} Converges Almost Surely to 0}
\label{sec:probA}

In this section we show $\probA \asto 0$ by a bounding moments of $\probA$ and then finishing with \cref{lemma:momentBoundASConvergence}.

\renewcommand{\rho}{\lambda}
For each $\alpha \in [n]$, let $\psi_\alpha (x) \defeq \psi(g^1_\alpha, \ldots, g^{\MM-1}_\alpha, \omega_\alpha + \sigma x)$, with $\omega$ and $\sigma$ defined in \cref{eqn:meanvardef}.
This is a random function depending on the randomness of $g^1_\alpha, \ldots, g^{\MM-1}_\alpha$, and it changes with $n$ as well.
Note by \cref{eqn:gConditionedOnB},
\[ \probA \disteq_\Bb \f 1 n \sum_{\alpha =1}^n
\psi_\alpha(\xi_\alpha) - \EV_{\xi'} \psi_\alpha(\xi_\alpha')\]
where $\xi, \xi' \sim \Gaus(0, I)$.

Now the $2k$-moment of $\probA$ for any integer $k\ge 1$ satisfies
\begin{align*}
    \EV[\probA^{2k} \mid \Bb]
        &=
            \f 1 {n^{2k}} 
            \EV \left[
                \sum_{\alpha =1}^n
                    \lp \psi_\alpha(\xi_\alpha) - \EV_{\xi'} \psi_\alpha(\xi_\alpha') \rp^{2k}
                + \cdots
                \ \bigg\vert\ \Bb \right]
\end{align*}
where the $\cdots$ include only terms that involve only powers of $\psi_\alpha(\xi_\alpha) - \EV_{\xi'} \psi_\alpha(\xi_\alpha')$ greater than 1 for each $\alpha$.
Indeed, other terms are killed by the conditional mean, since each $\psi_\alpha(\xi_\alpha) - \EV_{\xi'} \psi_\alpha(\xi_\alpha')$ has zero (conditional) mean and is independent from others when conditioned on $\Bb$.
We can push the conditional mean operator inside each product by conditional independence.
Then, applying power mean inequality and AM-GM to bound each mixed moment with linear combinations of the $2k$th powers, we get
\begin{align}
    \EV[\probA^{2k} \mid \Bb]
        &\le
            \f {D n^{2k-1}} {n^{2k}} \cdot
            \f 1 n \sum_{\alpha =1}^n \EV \left[
                    \lp \psi_\alpha(\xi_\alpha) - \EV_{\xi'} \psi_\alpha(\xi_\alpha') \rp^{2k}
                \ \bigg\vert\ \Bb \right]
        \defeq \f D n U
        \label{eqn:U}
\end{align}
where $D$ is some absolute constant.
Thus, to show $\probA \asto 0$, it suffices to bound $U$ and then apply \cref{lemma:momentBoundASConvergence}.
This is equivalent to bounding the uncentered moments $\EV_{z \sim \Gaus(0, 1)} |\psi_\alpha(x)|^q$ for $q=2k$.
Suppose $\psi$ is $\rho$-controlled and satisfies
\begin{align*}
|\psi(x)| \le e^{C \sum_i |x_i|^\rho + c}\quad\text{ for some $C, c> 0$ and $\rho < 2$}.
\numberthis\label{eqn:psiControlledDefn}
\end{align*}
We have
\begin{align*}
\EV_{z \sim \Gaus(0, 1)} |\psi_\alpha(z)|^q
    &\le
        \EVbr[z]{e^{Cq
                \left(
                    |\omega_\alpha + \sigma z|^\rho + \sum_{i=1}^{\MM -1} | g^i_\alpha|^\rho
                \right) + cq}
                }
        \\
    &\le
        \EVbr[z]{e^{Cq2^\rho
                \left(
                    |\omega_\alpha|^\rho + |\sigma z|^\rho + \sum_{i=1}^{\MM -1} | g^i_\alpha|^\rho
                \right) + cq}
                }
        \\
    &=
        e^{C q 2^\rho \left(
            |\omega_\alpha|^\rho
            + \sum_{i=1}^{\MM -1} | g^i_\alpha|^\rho
                \right) + cq}
        \EVbr[z]{e^{C q 2^\rho \sigma^\rho |z|^\rho}}
        \\
    &=
        e^{C q 2^\rho \left(
            |\omega_\alpha|^\rho
            + \sum_{i=1}^{\MM -1} | g^i_\alpha|^\rho
                \right) + cq}
        R
\end{align*}
where $R = \alpExp_\rho^1(C q 2^\rho \sigma^\rho) > 0$ is deterministic and $\alpExp_\rho^k$ is as defined in \cref{lemma:alpExp}.
Now,
\begin{align*}
    |\omega_\alpha|^\rho
        &=
            \left| \sum_{i=1}^r v_i \hat g^i_\alpha \right|^\rho
        \le
            r^\rho \sum_{i=1}^r |v_i|^\rho |\hat g^i_\alpha|^\rho.
\end{align*}
Additionally, almost surely, $|v_i| < |\mathring v_i| + 1$, for all $i \in [r]$ simultaneously, for large enough $n$ because $v_i \asto \mathring v_i$.
Let $L = C q 2^\rho r^\rho \max_{i=1}^r (|\mathring v_i| + 1)$ and $L' = cq$, where $C, c$ are as in \cref{eqn:psiControlledDefn}.
Then, almost surely, for large enough $n$, for all $z^1, \ldots, z^{\MM -1} \in \R$,
\begin{align*}
    e^{C q 2^\rho \left(
            \left|\sum_{i=1}^r v_i z^i \right|^\rho
            + \sum_{i=1}^{\MM -1} |z^i|^\rho
                \right) + cq}
    &\le
        e^{L' + L\sum_{i=1}^{\MM -1} |z^i|^\rho}
    \defeq
        \hat \psi(z^1, \ldots, z^{\MM -1}).
\end{align*}
Obviously $\hat\psi$ is $\rho$-controlled.
Then, again a.s. for large enough $n$, simultaneously for all $\alpha$,
\begin{align*}
    \EV_{z \sim \Gaus(0, 1)} |\psi_\alpha(z)|^q
        &\le
            R \hat \psi(g^1_\alpha, \ldots, g^{\MM -1}_\alpha),
            \quad \text{so that}
            \\
    \f 1 n \sum_{\alpha=1}^n 
        \EV_{z \sim \Gaus(0, 1)} |\psi_\alpha(z)|^q
        &\le
            R \f 1 n \sum_{\alpha=1}^n \hat \psi(g^1_\alpha, \ldots, g^{\MM -1}_\alpha)
        \asto
            R \EV_{Z} \hat\psi(Z)
\end{align*}
as $n \to \infty$, by induction hypothesis, where $Z \sim \Gaus(\tmu, \tSigma)$.
Consequently, almost surely, the $U$ in \cref{eqn:U} (as a function of $g^1, \ldots, g^{\MM-1}$) is uniformly bounded in $n$.
Applying \cref{lemma:momentBoundASConvergence} for large enough $q$ yields the result.

\subsubsection{\texorpdfstring{$\probB$}{B} Converges Almost Surely to 0}
\label{sec:probB}

In this section we show  $\probB \asto 0.$
The main insight here is integrating a function against Gaussian induces smoothness in the function.
We will assume that $\mathring \sigma > 0$, so that $\sigma > 0$ almost surely for large enough $n$.
This is because $\mathring \sigma=0$ implies that $g^\MM$ is in the linear span of $\{g^1, \ldots, g^{\MM-1}\}$ almost surely by \cref{lemma:rankStability}, and \ref{IH:MomConv}$(m)$ then holds trivially.

For each $\alpha \in [n]$, $w \in \R$, $\tau \ge 0$, let
\[\Psi_\alpha(w; \tau^2) \defeq
            \EV_{z\sim\Gaus(0, 1)}
            \psi\lp
                g^1_\alpha, \ldots, g^{\MM-1}_\alpha, w + \tau z
                \rp.
\]
(Here and in all that follows, $\tau^2$ is the square of $\tau$, and the $2$ is not an index).
This is a random function, with randomness induced by $g^1, \ldots, g^{\MM-1}$.

By \cref{lemma:stein}, $\Psi_\alpha$ is differentiable in $w$, and
\[
\pd_w \Psi_\alpha(w; \tau^2) = \tau^{-1} \EV_{z \sim \Gaus(0, 1)} z \psi(g^1_\alpha, \ldots, g^{\MM-1}_\alpha, w + \tau z).
\]

We can obtain the following smoothness condition on $\Psi_\alpha$.
\begin{lemma}\label{lemma:PsiAlphaSmoothness}
For any $w, \tau, \epsilon \in \R$ with $\epsilon, \tau > 0$,
\[|\Psi_\alpha(w; \tau^2) - \Psi_\alpha(w + \epsilon; \tau^2)|
\le
        |\epsilon| \inv \tau R(\tau)
        \hat \Psi(g^1_\alpha, \ldots, g^{\MM -1}_\alpha)
        e^{C 4^\rho (|w|^\rho + |\epsilon|^\rho)} 
        ,
\]
\end{lemma}
where $\hat \Psi(g^1_\alpha, \ldots, g^{\MM-1}_\alpha) \defeq 
            e^{C 2^\rho \sum_{i=1}^{\MM -1} |g^i_\alpha|^\rho + c}$ and $R(\tau) \defeq
            \EV_z |z| e^{C 2^\rho \tau^\rho |z|^\rho}$.
\begin{proof}
Clearly, with $z \sim \Gaus(0, 1)$,
\begin{align*}
    |\pd_w \Psi_\alpha(w; \tau^2)|
        &\le
            \inv \tau \EV_{z} |z \psi(g^1_\alpha, \ldots, g^{\MM - 1}_\alpha, w+ \tau z)|
            \\
        &\le
            \inv \tau \EV_z |z| e^{C\lp |w+\tau z|^\rho + \sum_{i=1}^{\MM -1} |g^i_\alpha|^\rho \rp + c}
            \\
        &\le
            \inv \tau \EV_z |z| e^{C2^\rho \lp |w|^\rho+\tau^\lambda |z|^\rho + \sum_{i=1}^{\MM -1} |g^i_\alpha|^\rho \rp + c}
            \\
        &=
            \inv \tau \hat \Psi(g^1_\alpha, \ldots, g^{\MM-1}_\alpha) R(\tau) e^{C 2^\rho |w|^\rho}
            .
\end{align*}
Then
\begin{align*}
    &\phantomeq
        |\Psi_\alpha(w; \tau^2) - \Psi_\alpha(w + \epsilon; \tau^2)|
        \\
    &\le
        \left| \int_{w}^{w+\epsilon} \dd \xi\ \pd_\xi \Psi_\alpha(\xi; \tau^2) \right|
        \\
    &\le
        \inv \tau R(\tau)
        \hat \Psi(g^1_\alpha, \ldots, g^{\MM -1}_\alpha) 
        \int_w^{w+\epsilon} \dd \xi\ 
            e^{C 2^\rho |\xi|^\rho}
        \\
    &=
        \inv \tau R(\tau)
        \hat \Psi(g^1_\alpha, \ldots, g^{\MM -1}_\alpha) 
        \int_0^{\epsilon} \dd \xi\ 
            e^{C 2^\rho |w+\xi|^\rho}
        \\
    &\le
        \inv \tau R(\tau)
        \hat \Psi(g^1_\alpha, \ldots, g^{\MM -1}_\alpha)
        \int_0^{\epsilon} \dd \xi\ 
            e^{C 4^\rho |w|^\rho} e^{C 4^\rho |\xi|^\rho}
        \\
    &=
        \inv \tau R(\tau)
        \hat \Psi(g^1_\alpha, \ldots, g^{\MM -1}_\alpha)
        e^{C 4^\rho |w|^\rho} |\epsilon| e^{C 4^\rho |\epsilon|^\rho}.
\end{align*}
\end{proof}

Therefore, with $z\sim \Gaus(0, 1)$,
\begin{align*}
    &\phantomeq
        \left|
            \EVbr[z]{
                \psi(g^1_\alpha, \ldots, g^{\MM-1}_\alpha, \sum_{i=1}^r v_i \hat g^i_\alpha + \sigma z)
            }
            -
            \EVbr[z]{
                \psi(g^1_\alpha, \ldots, g^{\MM-1}_\alpha, \sum_{i=1}^r \mathring v_i \hat g^i_\alpha + \sigma z)
            }
        \right|
        \\
    &=
        \left|
                \Psi_\alpha\left(\sum_{i=1}^r v_i \hat g^i_\alpha; \sigma^2\right)
            -
                \Psi_\alpha\left(\sum_{i=1}^r \mathring v_i \hat g^i_\alpha; \sigma^2\right)
        \right|
        \\
    &\le
        \inv \sigma R(\sigma)
        \hat \Psi(g^1_\alpha, \ldots, g^{\MM-1}_\alpha)
        e^{C 4^\rho \lp \left| \sum_{i=1}^r \mathring v_i \hat g^i_\alpha \right|^\rho + |\epsilon_\alpha|^\rho \rp} |\epsilon_\alpha|
        ,
\end{align*}
where $\epsilon_\alpha \defeq \sum_{i=1}^r (v_i - \mathring v_i) \hat g^i_\alpha$.
Because $\hat \Psi$ is $\lambda$-controlled,
\begin{align*}
    \f 1 n \sum_{\alpha=1}^n \hat \Psi(g^1_\alpha, \ldots, g^{\MM-1}_\alpha)
\end{align*}
converges almost surely to a deterministic limit.
At the same time, since $v_i \asto \mathring v_i$, we also have $\epsilon_\alpha \asto 0$ as $n \to \infty$,
so that
\begin{align*}
    \f 1 n \sum_{\alpha=1}^n
        \left|
            \EVbr[z]{
                \psi(g^1_\alpha, \ldots, g^{\MM-1}_\alpha, \sum_{i=1}^r v_i \hat g^i_\alpha + \sigma z)
            }
            -
            \EVbr[z]{
                \psi(g^1_\alpha, \ldots, g^{\MM-1}_\alpha, \sum_{i=1}^r \mathring v_i \hat g^i_\alpha + \sigma z)
            }
        \right|
    \asto 0
    .
\end{align*}
A similar argument shows that we can replace $\sigma$ with $\mathring \sigma$:
\begin{align*}
    \f 1 n \sum_{\alpha=1}^n
        \left|
            \EVbr[z]{
                \psi(g^1_\alpha, \ldots, g^{\MM-1}_\alpha, \sum_{i=1}^r \mathring v_i \hat g^i_\alpha + \sigma z)
            }
            -
            \EVbr[z]{
                \psi(g^1_\alpha, \ldots, g^{\MM-1}_\alpha, \sum_{i=1}^r \mathring v_i \hat g^i_\alpha + \mathring \sigma z)
            }
        \right|
    \asto 0
    .
\end{align*}
By triangular inequality, these limits show that $\probB \asto 0$ as desired.

\section{Proof of \texorpdfstring{\netsorplus}{Netsor+} Master Theorem}
\label{sec:netsorplusMasterTheoremProof}

In this section we describe how to augment the proof of \cref{thm:netsorminMasterTheorem} given in \cref{sec:proofMasterTheorems} to yield the proof of \cref{thm:Netsor+MasterTheorem}.
The key points to note here are 1) the presence of \ref{linetype:lincomb} rules in \netsorplus but not in \netsormin, 2) the rank stability assumption \cref{assm:asRankStab} used in \cref{thm:Netsor+MasterTheorem}, and 3) an additional term in \cref{eqn:decompABC} due to fluctuations in the parameter $\bigtheta$.

\subsection{\ref{linetype:lincomb}}

As remarked in \cref{remk:netsormin}, any usage of \ref{linetype:lincomb} in a \netsorplus program can be absorbed into downstream nonlinearities or be expressed as \ref{linetype:nonlin+} rule.
So WLOG, we can assume that the \netsorplus program has no applications of \ref{linetype:lincomb}.

\subsection{Rank Stability}

By \cref{remk:necessityRankStab}, we see that rank stability assumption is necessary for the \netsorplus Master Theorem.
Whereas in \cref{sec:proofMasterTheorems}, we had to intricately weave together an induction on rank stability (more generally, \ref{IH:coreSet}) and an induction on moment convergence (\ref{IH:MomConv}), here to show \cref{thm:Netsor+MasterTheorem}, we just need 1) to induct on \ref{IH:MomConv} and 2)
to invoke \cref{assm:asRankStab} whenever we need to use \cref{lemma:rankStability}, which is when we need to show that pseudo-inverse commutes with almost surely limit, such as in \cref{prop:pseudoinverseLambda}, and when we need to ensure either $\sigma$ is almost surely 0 or is almost surely positive, as in \cref{sec:probB}.

\subsection{Fluctuation of the Parameters}
\label{sec:D}
When we have parameters in nonlinearities, \cref{eqn:decompABC} needs to be modified to contain an additional term $\probD$:

\begin{align*}
    &\phantomeq
        \left|\f 1 n \sum_{\alpha = 1}^n \psi(g^1_\alpha, \ldots, g^{\MM}_\alpha; \bigtheta)
        - \EV_{Z \sim \Gaus(\tmu, \tSigma)} \psi(Z; \mathring{\bigtheta})
        \right|
    \le
        \probD + \probA + \probB + \probC
\end{align*}
where
\begin{align*}
\probD
    &\defeq
        \left|\f 1 n \sum_{\alpha = 1}^n \psi(g^1_\alpha, \ldots, g^{\MM}_\alpha; \bigtheta)
        - \psi(g^1_\alpha, \ldots, g^{\MM}_\alpha; \mathring{\bigtheta})
        \right|
\end{align*}
and $\probA, \probB, \probC$ are as in \cref{eqn:decompABC} but replacing $\psi(-)$ there with $\psi(-; \mathring{\bigtheta})$.
Because $\psi(-; -)$ is parameter-controlled at $\mathring{\bigtheta}$ by assumption, $\psi(-; \mathring{\bigtheta})$ is controlled, and $\probA, \probB, \probC \asto 0$ with the same arguments as before (except using rank stability assumption \cref{assm:asRankStab} where appropriate, instead of \ref{IH:coreSet}).

Now, by the other property of parameter-control, we have
\begin{align*}
\probD
    &\le
        \f 1 n \sum_{\alpha = 1}^n 
        \left|\psi(g^1_\alpha, \ldots, g^{\MM}_\alpha; \bigtheta)
        - \psi(g^1_\alpha, \ldots, g^{\MM}_\alpha; \mathring{\bigtheta})
        \right|
        \\
    &\le
        \f 1 n \sum_{\alpha = 1}^n 
        f(\bigtheta) \bar \psi(g^1_\alpha, \ldots, g^\MM_\alpha)
        \\
    &=
        f(\bigtheta)
        \f 1 n \sum_{\alpha = 1}^n 
        \bar \psi(g^1_\alpha, \ldots, g^\MM_\alpha)
\end{align*}
for some controlled $\bar \psi: \R^\MM \to \R$ and some $f: \R^l \to \R^{\ge 0} \cup \{\infty\}$ that is continuous at $\mathring{\bigtheta}$ and has $f(\mathring{\bigtheta}) = 0$ (where $\bar \psi$ and $f$ can both depend on $\mathring{\bigtheta}$).
Since $\bigtheta \asto \mathring{\bigtheta}$, we have $f(\bigtheta) \asto 0$.
In addition, by \ref{IH:MomConv}, $\f 1 n \sum_{\alpha = 1}^n 
\bar \psi(g^1_\alpha, \ldots, g^\MM_\alpha)$ converges a.s.\ as well to a finite constant.
Therefore,
\begin{align*}
\probD \asto 0
\end{align*}
as desired.

\subsection{Summary}

The proof of \cref{thm:Netsor+MasterTheorem}, WLOG for programs without \ref{linetype:lincomb}, would proceed as follows: We induct on \ref{IH:MomConv} with the same setup as \cref{sec:inductiveSetup}, except using \cref{assm:asRankStab} for \cref{prop:pseudoinverseLambda}.
Then we prove the inductive step for \ref{IH:MomConv} as in \cref{sec:inductiveMoments}.
We modify \cref{eqn:decompABC} to add a term $\probD$ as in \cref{sec:D}, which goes to 0 a.s.\ as argued there.
The same arguments for $\probA, \probB, \probC \asto 0$, exhibited in \cref{sec:inductiveMoments} still hold, except that in the proof of $\probB \asto 0$, we apply \cref{assm:asRankStab} (instead of \cref{lemma:rankStability}) to allow us to assume $\mathring \sigma > 0$ and $\sigma > 0$ almost surely.

\section{Formal Specification of Tensor Programs}
\label{sec:formalspec}

In the main text, we have adopted an informal approach to specifying the \netsor language and its siblings, in order to make the material accessible to a wide audience.
Here we give the formal specifications for \netsormin (\cref{grammar:netsormin,inf:netsormin,sem:netsormin}), \netsor (\cref{grammar:netsor,inf:netsor,sem:netsor}), and self-parametrized \netsorplus (\cref{grammar:spnetsorplus,inf:spnetsorplus,sem:spnetsorplus}).
For ease of presentation, we have represented matrix multiplication explicitly via an operation $\textbf{MatMul}$ (likewise for $\textbf{Moment}$ in self-parametrized \netsorplus), and have we used double colon :: instead of single colon : for type annotation.

\begin{figure}
\centering
\renewcommand{\synt}[1]{\textsl{#1}}
\renewcommand{\syntleft}{\slshape}
\renewcommand{\syntright}{}
\begin{tcolorbox}

\begin{grammar}
<program> ::= <stmt>*

<stmt> ::= 
\textbf{Input} <var> \textbf{::} <type>\\
| <var> \textbf{:=} <expr> \textbf{::} <type>

<expr> ::= \textbf{MatMul} \textbf{(}<var>, <var> \textbf{)}\\
| <fun>\textbf{(} <var>* \textbf{)}

<var> ::= $\la$ \textnormal{id} $\ra$

<fun> ::= $\la$ {function $\R^k \to \R$ for some $k \ge 0$} $\ra$

<type> ::= \textbf{G}\textbf{(}<nat>\textbf{)}
| \textbf{H}\textbf{(}<nat>\textbf{)}
| \textbf{A}\textbf{(}<nat>, <nat>\textbf{)}

<nat> ::= $\la$ any integer $\ge 1$ $\ra$

\end{grammar}
\end{tcolorbox}
\caption{\netsormin Grammar; see \cref{defn:netsormin}.}
\label{grammar:netsormin}
\end{figure}

\begin{figure}
\renewcommand{\sf}{\mathsf}
\begin{mathpar}
\infer{\sf{expr}: \sf{type}\\ \sf{var}\ \textbf{:=}\ \sf{expr}\ \textbf{::}\ \sf{type}}
{\sf{var} : \sf{type}}
\\
\infer{\sf a: \mathbf A(n_1, n_2) \\ \sf h: \mathbf H(n_2)}
{\textbf{MatMul}(\sf a, \sf h): \mathbf G(n_1)}
\quad
\infer{\sf g_1, \ldots, \sf g_k: \mathbf G(n) \\ \sf f: \R^k \to \R}
{\sf f(\sf g_1, \ldots, \sf g_k): \mathbf H(n)}
\\
\end{mathpar}
\caption{\netsormin Inference Rules}
\label{inf:netsormin}
\end{figure}

\begin{figure}
\renewcommand{\sf}{\mathsf}
\begin{mathpar}

\infer{\llbracket \sf a \rrbracket = W \in \R^{n_1 \times n_2}\\
\llbracket \sf h \rrbracket = v \in \R^{n_2}}
{\llbracket \textbf{MatMul}(\sf a, \sf h) \rrbracket = W v}
\quad
\infer{\forall i \in [k], \llbracket \sf g_i \rrbracket = v_i \in \R^n \\
\llbracket \sf f \rrbracket = f: \R^k \to \R}
{\llbracket \sf f(\sf g_1, \ldots, \sf g_k) \rrbracket = u \in \R^n \text{ with $u_\alpha = f(v_{1\alpha}, \ldots, v_{k\alpha})$}}
\end{mathpar}
\caption{\netsormin Semantics}
\label{sem:netsormin}
\end{figure}

\begin{figure}
\centering
\renewcommand{\synt}[1]{\textsl{#1}}
\renewcommand{\syntleft}{\slshape}
\renewcommand{\syntright}{}
\begin{tcolorbox}

\begin{grammar}
<program> ::= <stmt>*

<stmt> ::= 
\textbf{Input} <var> \textbf{::} <type>\\
| <var> \textbf{:=} <expr> \textbf{::} <type>

<expr> ::= \textbf{MatMul} \textbf{(}<var>, <var> \textbf{)}\\
| <fun>\textbf{(} <var>* \textbf{)}\\
| <var> (\textbf{+} <var>)$^+$

<var> ::= $\la$ \textnormal{id} $\ra$

<fun> ::= $\la$ {function $\R^k \to \R$ for some $k \ge 0$} $\ra$

<type> ::= \textbf{G}\textbf{(}<nat>\textbf{)}
| \textbf{H}\textbf{(}<nat>\textbf{)}
| \textbf{A}\textbf{(}<nat>, <nat>\textbf{)}

<nat> ::= $\la$ any integer $\ge 1$ $\ra$

\end{grammar}
\end{tcolorbox}
\caption{\netsor Grammar; see \cref{defn:netsor}. Compared to \netsormin grammar, the only new item is \ref{linetype:lincomb} in \textit{expr}.}
\label{grammar:netsor}
\end{figure}

\begin{figure}
\renewcommand{\sf}{\mathsf}
\begin{mathpar}
\infer{\sf{expr}: \sf{type}\\ \sf{var}\ \textbf{:=}\ \sf{expr}\ \textbf{::}\ \sf{type}}
{\sf{var} : \sf{type}}
\\
\infer{\sf a: \mathbf A(n_1, n_2) \\ \sf h: \mathbf H(n_2)}
{\textbf{MatMul}(\sf a, \sf h): \mathbf G(n_1)}
\quad
\infer{\sf g_1, \ldots, \sf g_k: \mathbf G(n) \\ \sf f: \R^k \to \R}
{\sf f(\sf g_1, \ldots, \sf g_k): \mathbf H(n)}
\quad
\infer{\sf g_1, \ldots, \sf g_k: \mathbf G(n)}
{\sf g_1 + \cdots + \sf g_k: \mathbf G(n)}
\\
\end{mathpar}
\caption{\netsor Inference Rules}
\label{inf:netsor}
\end{figure}

\begin{figure}
\renewcommand{\sf}{\mathsf}
\begin{mathpar}

\infer{\llbracket \sf a \rrbracket = W \in \R^{n_1 \times n_2}\\
\llbracket \sf h \rrbracket = v \in \R^{n_2}}
{\llbracket \textbf{MatMul}(\sf a, \sf h) \rrbracket = W v}
\quad
\infer{\forall i \in [k], \llbracket \sf g_i \rrbracket = v_i \in \R^n \\
\llbracket \sf f \rrbracket = f: \R^k \to \R}
{\llbracket \sf f(\sf g_1, \ldots, \sf g_k) \rrbracket = u \in \R^n \text{ with $u_\alpha = f(v_{1\alpha}, \ldots, v_{k\alpha})$}}
\\
\infer{\forall i \in [k], \llbracket \sf g_i \rrbracket = v_i \in \R^n}
{\llbracket \sf g_1 + \cdots + \sf g_k \rrbracket = v_1 + \cdots + v_k \in \R^n}
\end{mathpar}
\caption{\netsor Semantics}
\label{sem:netsor}
\end{figure}

\begin{figure}
\centering
\renewcommand{\synt}[1]{\textsl{#1}}
\renewcommand{\syntleft}{\slshape}
\renewcommand{\syntright}{}
\begin{tcolorbox}

\begin{grammar}
<program> ::= <stmt>*

<stmt> ::= 
\textbf{Input} <var> \textbf{::} <type>\\
| <var> \textbf{:=} <expr> \textbf{::} <type>

<expr> ::= \textbf{MatMul} \textbf{(}<var>, <var> \textbf{)}\\
| <fun>\textbf{(} <var>* \textbf{;} <var>*\textbf{)}\\
| <var> (\textbf{+} <var>)$^+$\\
| \textbf{Moment}\textbf{(}<fun>; <var>*; <var>*\textbf{)}

<var> ::= $\la$ \textnormal{id} $\ra$

<fun> ::= $\la$ {parametrized function $\R^k \times \R^l \to \R$ for some $k, l \ge 0$} $\ra$

<type> ::= 
\textbf{C}
| \textbf{G}\textbf{(}<nat>\textbf{)}
| \textbf{H}\textbf{(}<nat>\textbf{)}
| \textbf{A}\textbf{(}<nat>, <nat>\textbf{)}

<nat> ::= $\la$ any integer $\ge 1$ $\ra$

\end{grammar}
\end{tcolorbox}
\caption{Self-Parametrized \netsorplus Grammar; see \cref{defn:selfParam}. Compared to \netsor grammar, we have added a new type \textbf{C} and a new expression \textbf{Moment}.}
\label{grammar:spnetsorplus}
\end{figure}

\begin{figure}
\renewcommand{\sf}{\mathsf}
\begin{mathpar}
\infer{\sf{expr}: \sf{type}\\ \sf{var}\ \textbf{:=}\ \sf{expr}\ \textbf{::}\ \sf{type}}
{\sf{var} : \sf{type}}
\\
\infer{\sf a: \mathbf A(n_1, n_2) \\ \sf h: \mathbf H(n_2)}
{\textbf{MatMul}(\sf a, \sf h): \mathbf G(n_1)}
\quad
\infer{\sf g_1, \ldots, \sf g_k: \mathbf G(n)}
{\sf g_1 + \cdots + \sf g_k: \mathbf G(n)}
\\
\infer{\sf g_1, \ldots, \sf g_k: \mathbf G(n) \\ 
    \sf c_1, \ldots, \sf c_l: \mathbf C \\
    \sf f: \R^k \times \R^l \to \R}
{\sf f(\sf g_1, \ldots, \sf g_k; \sf c_1, \ldots, \sf c_l): \mathbf H(n)}
\\
\infer{\sf g_1, \ldots, \sf g_k: \mathbf G(n) \\ 
    \sf c_1, \ldots, \sf c_l: \mathbf C \\
    \sf f: \R^k \times \R^l \to \R}
{\textbf{Moment}(\sf f; \sf g_1, \ldots, \sf g_k; \sf c_1, \ldots, \sf c_l): \mathbf C}
\end{mathpar}
\caption{Self-Parametrized \netsorplus Inference Rules}
\label{inf:spnetsorplus}
\end{figure}

\begin{figure}
\renewcommand{\sf}{\mathsf}
\begin{mathpar}

\infer{\llbracket \sf a \rrbracket = W \in \R^{n_1 \times n_2}\\
\llbracket \sf h \rrbracket = v \in \R^{n_2}}
{\llbracket \textbf{MatMul}(\sf a, \sf h) \rrbracket = W v}
\quad
\infer{\forall i \in [k], \llbracket \sf g_i \rrbracket = v_i \in \R^n}
{\llbracket \sf g_1 + \cdots + \sf g_k \rrbracket = v_1 + \cdots + v_k \in \R^n}
\\

\infer{
    \forall i \in [k], \llbracket \sf g_i \rrbracket = v_i \in \R^n \\
    \forall j \in [l], \llbracket \sf c_j \rrbracket = c_j \in \R \\
    \llbracket \sf f \rrbracket = f: \R^k \times \R^l \to \R}
{\llbracket \sf f(\sf g_1, \ldots, \sf g_k; \sf c_1, \ldots, \sf c_l) \rrbracket = u \in \R^n \text{ with $u_\alpha = f(v_{1\alpha}, \ldots, v_{k\alpha}; c_1, \ldots, c_l)$}}
\\
\infer{
    \forall i \in [k], \llbracket \sf g_i \rrbracket = v_i \in \R^n \\
    \forall j \in [l], \llbracket \sf c_j \rrbracket = c_j \in \R \\
    \llbracket \sf f \rrbracket = f: \R^k \times \R^l \to \R}
{\llbracket \textbf{Moment}(\sf f; \sf g_1, \ldots, \sf g_k; \sf c_1, \ldots, \sf c_l) \rrbracket
    = \f 1 n \sum_{\alpha=1}^n f(v_{1\alpha}, \ldots, v_{k \alpha}; c_1, \ldots, c_l)}
\end{mathpar}
\caption{Self-Parametrized \netsorplus Semantics}
\label{sem:spnetsorplus}
\end{figure}

\end{document}